\documentclass[11pt]{article}
\usepackage[top=1in, bottom=1in, left=1in, right=1in]{geometry}

\usepackage{zhiyuz}

\title{\textbf{Provable Hierarchical Imitation Learning via EM}}
\author{
  Zhiyu Zhang \\
  Boston University\\
  \texttt{zhiyuz@bu.edu}\\
  \and
  Ioannis Ch. Paschalidis\\
  Boston University\\
  \texttt{yannisp@bu.edu}\\
}
\date{\vspace{-5ex}}

\begin{document}
\maketitle

\begin{abstract}
Due to recent empirical successes, the options framework for hierarchical reinforcement learning is gaining increasing popularity. Rather than learning from rewards, we consider learning an options-type hierarchical policy from expert demonstrations. Such a problem is referred to as \emph{hierarchical imitation learning}. Converting this problem to parameter inference in a latent variable model, we develop convergence guarantees for the EM approach proposed by \citep{daniel2016probabilistic}. The population level algorithm is analyzed as an intermediate step, which is nontrivial due to the samples being correlated. If the expert policy can be parameterized by a variant of the options framework, then, under regularity conditions, we prove that the proposed algorithm converges with high probability to a norm ball around the true parameter. To our knowledge, this is the first performance guarantee for an hierarchical imitation learning algorithm that only observes primitive state-action pairs. 
\end{abstract}

\section{Introduction}

Recent empirical studies \citep{kulkarni2016hierarchical,nachum2018data,tessler2017deep,vezhnevets2017feudal} have shown that the scalability of Reinforcement Learning (RL) algorithms can be improved by incorporating hierarchical structures. As an example, consider the \emph{options} framework \citep{sutton1999between} representing a two-level hierarchical policy: with a set of multi-step low level procedures (options), the high level policy selects an option, which, in turn, decides the primitive action applied at each time step until the option terminates. Learning such a hierarchical policy from environmental feedback effectively breaks the overall task into sub-tasks, each easier to solve. 

Researchers have investigated the hierarchical RL problem under various settings. Existing theoretical analyses \citep{brunskill2014pac,fruit2017exploration,fruit2017regret,mann2014scaling} typically assume that the options are given. As a result, only the high-level policy needs to be learned. Recent advances in deep hierarchical RL (e.g., \citep{bacon2017option}) focus on concurrently learning the full options framework, but still the initialization of the options is critical. A promising practical approach is to learn an initial hierarchical policy from expert demonstrations. Then, deep hierarchical RL algorithms can be applied for policy improvement. The former step is named as \emph{Hierarchical Imitation Learning} (HIL).

Due to its practicality, HIL has been extensively studied within the deep learning and robotics communities. However, existing works typically suffer from the following limitations. First, the considered HIL formulations often lack rigor and clarity. Second, existing works are mostly empirical, only testing on a few specific benchmarks. Without theoretical justification, it remains unclear whether the proposed methods can be generalized beyond their experimental settings. 

In this paper, we investigate HIL from a theoretical perspective. Our problem formulation is concise while retaining the essential difficulty of HIL: we need to learn a complete hierarchical policy from an \emph{unsegmented} sequence of state-action pairs. Under this setting, HIL becomes an inference problem in a latent variable model. Such a transformation was first proposed by \citep{daniel2016probabilistic}, where the Expectation-Maximization (EM) algorithm \citep{dempster1977maximum} was applied for policy learning. Empirical results for this algorithm and its gradient variants \citep{fox2017multi,krishnan2017ddco} demonstrate good performance, but the theoretical analysis remains open. By bridging this gap, we aim to solidify the foundation of HIL and provide some high level guidance for its practice. 

\subsection{Related work}

Due to its intrinsic difficulty, existing works on HIL typically consider its easier variants for practicality. If the expert options are observed, standard imitation learning algorithms can be applied to learn the high and low level policies separately \citep{le2018hierarchical}. If those are not available, a popular idea \citep{butterfield2010learning,manschitz2014learning,niekum2012learning,niekum2015learning} is to first divide the expert demonstration into segments using domain knowledge or heuristics, learn the individual option corresponding to each segment, and finally learn the high level policy. With additional supervision, these steps can be unified \citep{shiarlis2018taco}. In this regard, the EM approach \citep{daniel2016probabilistic,fox2017multi,krishnan2017ddco} is this particular idea pushed to an extreme: without any other forms of supervision, we simultaneously segment the demonstration and learn from it, by exploiting the latent variable structure. 

From the theoretical perspective, inference in parametric latent variable models is a long-standing problem in statistics. For many years the EM algorithm has been considered the standard approach, but performance guarantees \citep{mclachlan2007algorithm,wu1983convergence} were generally weak, only characterizing the convergence of parameter estimates to stationary points of the finite sample likelihood function. Under additional local assumptions, convergence to the Maximum Likelihood Estimate (MLE) can be further established. However, due to the randomness in sampling, the finite sample likelihood function is usually highly non-concave, leading to stringent requirements on initialization. Another weakness is that converging to the finite sample MLE does not directly characterize the distance to the maximizer of the population likelihood function which is the true parameter. 

Recent ideas on EM algorithms \citep{balakrishnan2017statistical,wang2015high,yang2017statistical,yi2015regularized} focus on the convergence to the true parameter directly, relying on an instrumental object named as the \emph{population EM algorithm}. It has the same two-stage iterative procedure as the standard EM algorithm, but its \emph{$Q$-function}, the maximization objective in the M-step, is defined as the infinite sample limit of the finite sample $Q$-function. Under regularity conditions, the population EM algorithm converges to the true parameter. The standard EM algorithm is then analyzed as its perturbed version, converging with high probability to a norm ball around the true parameter. The main advantage of this approach is that the true parameter usually has a large basin of attraction in the population EM algorithm. Therefore, the requirement on initialization is less stringent. See \citep[Figure~1]{yang2017statistical} for an illustration. 

The $Q$-function adopted in the population EM algorithm is named as the \emph{population $Q$-function}. To properly define such a quantity, the stochastic convergence of the finite sample $Q$-function needs to be constructed. When the samples are i.i.d., such as in Gaussian Mixture Models (GMMs) \citep{balakrishnan2017statistical,daskalakis2017ten,xu2016global}, the required convergence follows directly from the law of large numbers. However, this argument is less straightforward in time-series models such as Hidden Markov Models (HMMs) and the model considered in HIL. For HMMs, \citep{yang2017statistical} showed that the expectation of the $Q$-function converges, but both the stochastic convergence analysis and the analytical expression of the population $Q$-function are not provided. The missing techniques could be borrowed from a body of work \citep{cappe2006inference,de2017consistent,le2013online,van2008hidden} analyzing the asymptotic behavior of HMMs. Most notably, \citep{le2013online} provided a rigorous treatment of the population EM algorithm via sufficient statistics, assuming the HMM is parameterized by an exponential family. 

Finally, apart from the EM algorithm, a separate line of research \citep{anandkumar2014tensor,hsu2012spectral} applies spectral methods for tractable inference in latent variable models. However, such methods are mainly complementary to the EM algorithm since better performance can usually be obtained by initializing the EM algorithm with the solution of the spectral methods~\citep{kontorovich2013learning}. 

\subsection{Our contributions}

In this paper, we establish the first known performance guarantee for a HIL algorithm that only observes primitive state-action pairs. Specifically, we first fix and reformulate the original EM approach by \citep{daniel2016probabilistic} in a rigorous manner. The lack of mixing is identified as a technical difficulty in learning the standard options framework, and a novel \emph{options with failure} framework is proposed to circumvent this issue. 

Inspired by \citep{balakrishnan2017statistical} and \citep{yang2017statistical}, the population version of our algorithm is analyzed as an intermediate step. We prove that if the expert policy can be parameterized by the options with failure framework, then, under regularity conditions, the population version algorithm converges to the true parameter, and the finite sample version converges with high probability to a norm ball around the true parameter. Our analysis directly constructs the stochastic convergence of the finite sample $Q$-function, and an analytical expression of the resulting population $Q$-function is provided. Finally, we qualitatively validate our theoretical results using a numerical example. 

\section{Problem settings}

\paragraph{Notation.} Throughout this paper, we use uppercase letters (e.g., $S_t$) for random variables and lowercase letters (e.g., $s_t$) for values of random variables. Let $[t_1:t_2]$ be the set of integers $t$ such that $t_1\leq t\leq t_2$. When used in the subscript, the brackets are removed (e.g., $S_{t_1:t_2}=\set{S_t}_{t_1\leq t\leq t_2}$). 

\subsection{Definition of the hierarchical policy}

\begin{figure}[ht]
    \centering
    \includegraphics[width=350pt]{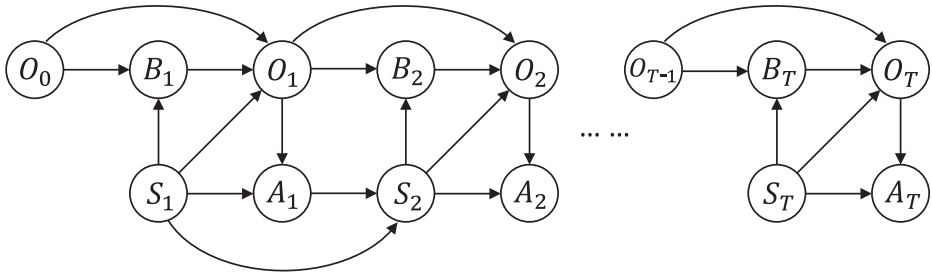}
    \caption{A graphical model for hierarchical reinforcement learning. }\label{figure:model}
\end{figure}

In this section, we first introduce the options framework for hierarchical reinforcement learning \citep{barto2003recent,sutton1999between}, captured by the probabilistic graphical model shown in Figure~\ref{figure:model}. The index $t$ represents the time; $(S_t,A_t,O_t,B_t)$ respectively represent the state, the action, the option and the termination indicator. For all $t$, $S_t$, $A_t$ and $O_t$ are defined on the finite state space $\cals$, the finite action space $\cala$ and the finite option space $\calo$; $B_t$ is a binary random variable. Define the parameter $\theta\defeq(\thetahi,\thetalo,\thetab)$ where $\thetahi\in\Thetahi$, $\thetalo\in\Thetalo$, and $\thetab\in\Thetab$. The parameter space $\calTheta\defeq\Thetahi\times\Thetalo\times\Thetab$ is a convex and compact subset of a Euclidean space. 

For any $(o_0,s_1)\in\calo\times \cals$, if we fix $(O_0,S_1)=(o_0,s_1)$ and consider a given $\theta$, the joint distribution on the rest of the graphical model is determined by the following components: an unknown environment transition probability $P$, a high level policy $\pihi$ parameterized by $\thetahi$, a low level policy $\pilo$ parameterized by $\thetalo$ and a termination policy $\pib$ parameterized by $\thetab$. Sampling a tuple $(s_{2:T},a_{1:T},o_{1:T},b_{1:T})$ from such a joint distribution, or equivalently, implementing the hierarchical decision process, follows the following procedure. Starting from the first time step, the decision making agent first determines whether or not to terminate the current option $o_0$. The decision is encoded in a termination indicator $b_1$ sampled from $\pi_{b}(\cdot|s_1,o_0;\theta_b)$. $b_1=1$ indicates that the option $o_0$ terminates and the next option $o_1$ is sampled from $\pi_{hi}(\cdot|s_1;\theta_{hi})$; $b_1=0$ indicates that the option $o_0$ continues and $o_1=o_0$. Next, the primitive action $a_1$ is sampled from $\pi_{lo}(\cdot|s_1,o_1;\theta_{lo})$, applying the low level policy associated with the option $o_1$. Using the environment, the next state $s_2$ is sampled from $P(\cdot|s_1,a_1)$. The rest of the samples $(s_{3:T},a_{2:T},o_{2:T},b_{2:T})$ are generated analogously.

The options framework corresponds to the above hierarchical policy structure and the policy triple $\{\pihi,\pilo,\pib\}$. However, due to a technicality identified at the end of this subsection, we consider a novel options with failure framework for the remainder of this paper, which adds an extra \emph{failure} mechanism to the graphical model in the case of $b_t=0$. Specifically, there exists a constant $0<\zeta<1$ such that when the termination indicator $b_t=0$, with probability $1-\zeta$ the next option $o_t$ is assigned to $o_{t-1}$, whereas with probability $\zeta$ the next option $o_t$ is sampled uniformly from the set of options $\mathcal O$. Notice that if $\zeta=0$, we recover the standard options framework. 

To simplify the notation, we define $\bar\pi_{hi}$ as the combination of $\pihi$ and the failure mechanism. For any $\theta_{hi}$, with any other input arguments,  
\begin{equation*}
  \bar\pi_{hi}(o_t|s_t,o_{t-1},b_t;\theta_{hi})\defeq
  \begin{cases}
    \pi_{hi}(o_t|s_t;\theta_{hi}), & \text{if $b_t=1$}, \\
    1-\zeta+\frac{\zeta}{|\mathcal O|}, & \text{if $b_t=0$, $o_t= o_{t-1}$},\\
    \frac{\zeta}{|\mathcal O|},& \text{if $b_t=0$, $o_t\neq o_{t-1}$.}
  \end{cases}
\end{equation*}

Formally, the options with failure framework is defined as the class of policy triples $\{\bar\pi_{hi},\pi_{lo}, \pi_b\}$ parameterized by $\zeta$ and $\theta$. With $(O_0,S_1)=(o_0,s_1)$ and a given $\theta$, let $\P_{\theta,o_0,s_1}$ be the joint distribution of $\{S_{2:T},A_{1:T},O_{1:T},B_{1:T}\}$. With any input arguments,
\begin{multline*}
\mathbb P_{\theta,o_0,s_1}(S_{2:T}=s_{2:T},A_{1:T}=a_{1:T},O_{1:T}=o_{1:T},B_{1:T}=b_{1:T})=\\
\left[\prod_{t=1}^T\pi_{b}(b_t|s_t,o_{t-1};\theta_b)\bar\pi_{hi}(o_t|s_t,o_{t-1},b_t;\theta_{hi})\pi_{lo}(a_t|s_t,o_t;\theta_{lo})\right]\left[\prod_{t=1}^{T-1}P(s_{t+1}|s_t,a_t)\right].
\end{multline*}

\paragraph{On the policy framework.} The options with failure framework is adopted to simplify the construction of the mixing condition (Lemma~\ref{lemma:mixing}). It is possible that our analysis could be extended to learn the standard options framework. In that case, instead of constructing the usual one step mixing condition, one could target the multi-step mixing condition similar to~\citep[Chap.~4.3]{cappe2006inference}. 

\subsection{The imitation learning problem}

Suppose an expert uses an options with failure policy with true parameters $\zeta$ and $\theta^*=(\theta^*_{hi},\theta^*_{lo},\theta^*_b)$; its initial condition $(o_0,s_1)$ is sampled from a distribution $\nu^*$. A finite length observation sequence $\{s_{1:T},a_{1:T}\}=\{s_t,a_t\}_{t=1}^T$ with $T\geq 2$ is observed from the expert. $\zeta$ and the parametric structure of the expert policy are known, but $\nu^*$ is unknown. Our objective is to estimate $\theta^*$ from $\{s_{1:T},a_{1:T}\}$.

\paragraph{On the practicality of our setting.} Two comments need to made here. First, it is common in practice to observe not one, but a set of independent observation sequences. In that case, the problem essentially becomes easier. Second, the cardinality of the option space and the parameterization of the expert policy are usually unknown. A popular solution is to assume an expressive parameterization (e.g., a neural network) in the algorithm and select $\textrm{card}(\calo)$ through cross-validation. Theoretical analysis of EM under this setting is challenging, even when samples are i.i.d. \citep{dwivedi2018singularity,dwivedi2018theoretical}. Therefore, we only consider the domain of \emph{correct-specification}. 

Throughout this paper, the following assumptions are imposed for simplicity.

\begin{assumption}[Non-degeneracy]\label{as:nondeg} With any other input arguments, the domain of $\pi_{hi}$, $\pi_{lo}$ and $\pi_b$ as functions of $\theta$ can be extended to an open set $\tilde\calTheta$ that contains $\calTheta$. Moreover,  for all $\theta\in\tilde\calTheta$, $\pi_{hi}$, $\pi_{lo}$ and $\pi_b$ parameterized by $\theta$ are strictly positive.
\end{assumption}

\begin{assumption}[Differentiability]\label{as:continuity}
With any other input arguments, $\pi_{hi}$, $\pi_{lo}$ and $\pi_b$ as functions of $\theta$ are continuously differentiable on $\tilde\calTheta$. 
\end{assumption}

Next, consider the stochastic process $\{O_{t-1},S_t\}_{t=1}^\infty$ induced by $\nu^*$ and the expert policy. Based on the graphical model, it is a Markov chain with finite state space $\calo\times\cals$. Let $\Pi_{\theta^*}$ be its set of stationary distributions, which is nonempty and convex. 

\begin{assumption}[Stationary initial distribution]\label{as:initial}
$\nu^*$ is an extreme point of $\Pi_{\theta^*}$. That is, $\nu^*\in\Pi_{\theta^*}$, and it cannot be written as the convex combination of two elements of $\Pi_{\theta^*}$. 
\end{assumption}

\paragraph{On the assumptions.} The first two assumptions are generally mild and therefore hold for many policy parameterizations. The third assumption is a bit more restrictive, but it is essential for our theoretical analysis. In Appendix~\ref{section:proof1}, we provide further justification of this assumption in a particular class of environments: $\forall s_t,s_{t+1}\in \cals$, there exists $a_t\in \cala$ such that $P(s_{t+1}|s_t,a_t)>0$. In such environments, $\Pi_{\theta^*}$ contains a unique element which is also the limiting distribution. If we start sampling the observation sequence late enough, Assumption~\ref{as:initial} is approximately satisfied. 

\section{A Baum-Welch type algorithm}\label{section:alg}

Adopting the EM approach, we present Algorithm~\ref{algorithm} for the estimation of $\theta^*$. It reformulates the algorithm by \citep{daniel2016probabilistic} in a rigorous manner, and an error in the latter is fixed: when defining the posterior distribution of latent variables, at any time $t<T$, the original algorithm neglects the dependency of future states $S_{t+1:T}$ on the current option $O_t$. A detailed discussion is provided in Appendix~\ref{subsection:problem}. 

\begin{algorithm*}[ht]
\caption{A Baum-Welch type algorithm for provable hierarchical imitation learning}
\label{algorithm}
\begin{algorithmic}[1]
\REQUIRE Observation sequence $\{s_{1:T},a_{1:T}\}$; a probability mass function $\mu(o_0|s_1)$ on $o_0\in\mathcal O$; $N\in\mathbb N_+$; $\theta^{(0)}\in\calTheta$. 
\FOR{$n=1,\ldots,N$}
\STATE Compute the forward message $\{\alpha^{\theta^{(n-1)}}_{\mu,t}\}_{t=1}^T$ and the backward message $\{\beta^{\theta^{(n-1)}}_{t|T}\}_{t=1}^T$ according to (\ref{eq:backwardT}), (\ref{eq:forward}), (\ref{eq:forward1}) and (\ref{eq:backward}). \label{line2}
\STATE Compute the smoothing distributions $\{\gamma^{\theta^{(n-1)}}_{\mu,t|T}\}_{t=1}^T$ and $\{\tilde \gamma^{\theta^{(n-1)}}_{\mu,t|T}\}_{t=2}^T$ according to (\ref{eq:smoothing}) and (\ref{eq:smoothing2}).\label{line3}
\STATE Update the parameter estimate $\theta^{(n)}\in\argmax_{\theta\in\calTheta} Q_{\mu,T}(\theta|\theta^{(n-1)})$ according to (\ref{eq:Q}). \label{line4}
\ENDFOR
\end{algorithmic}
\end{algorithm*}

Since our graphical model resembles an HMM, Algorithm~\ref{algorithm} is intuitively similar to the classical Baum-Welch algorithm \citep{baum1970maximization} for HMM parameter inference. Analogously, it iterates between forward-backward smoothing and parameter update. In each iteration, the algorithm first estimates certain marginal distributions of the latent variables $(O_{1:T},B_{1:T})$ conditioned on the observation sequence $\{s_{1:T},a_{1:T}\}$, assuming the current estimate of $\theta$ is correct. Such conditional distributions are named as \emph{smoothing distributions}, and they are used to compute the $Q$-function, which is a surrogate of the likelihood function. The next estimate of $\theta$ is assigned as one of the maximizing arguments of the $Q$-function.

From the structure of our graphical model, a prior distribution of $(O_0,S_1)$ is required to compute the smoothing distributions. Since the true prior distribution $\nu^*$ is unknown, $\hat\nu$, defined next, is used as its approximation: $\forall o_0\in\calo$, $\hat\nu(o_0,s_1)\defeq\mu(o_0|s_1)$; $\forall s'_1\neq s_1$, $\hat\nu(o_0,s'_1)\defeq 0$. Theorem~\ref{thm:existence} shows that the additional estimation error introduced by this approximation vanishes as $T\rightarrow\infty$, regardless of the choice of $\mu$. Let $\mathcal{M}$ be the set of $\mu$ allowed by Algorithm~\ref{algorithm}. 

\subsection{Latent variable estimation}

In the following, we define the forward message, the backward message and the smoothing distribution for all $\theta$, $\mu$ and all $t\in[1:T]$. All of these quantities are probability mass functions over $\calo\times\cals$, and normalizing constants $z_{\alpha,\mu,t}^{\theta}$, $z_{\beta,t}^{\theta}$ and $z_{\gamma,\mu}^{\theta}$ are adopted to enforce this. With any input arguments $o_t$ and $b_t$, the forward message is defined as
\begin{equation*}
\alpha^\theta_{\mu,t}(o_t,b_t)\defeq z_{\alpha,\mu,t}^{\theta}\mathbb E_{O_0\sim \mu(\cdot|s_1)}[\mathbb P_{\theta,O_0,s_1}(S_{2:t}=s_{2:t},A_{1:t}=a_{1:t},O_t=o_t,B_t=b_t)].
\end{equation*}
On the LHS, the dependency on $\{s_{1:T},a_{1:T}\}$ is omitted for a cleaner notation. By convention, $\alpha^\theta_{\mu,1}$ is equivalent to
\begin{equation*}
\alpha^\theta_{\mu,1}(o_1,b_1)=z_{\alpha,\mu,1}^{\theta}\mathbb E_{O_0\sim \mu(\cdot|s_1)}[\mathbb P_{\theta,O_0,s_1}(A_1=a_1,O_1=o_1,B_1=b_1)].
\end{equation*}
The backward message is defined as
\begin{equation*}
\beta^\theta_{t|T}(o_t,b_t)\defeq z_{\beta,t}^{\theta}\mathbb P_{\theta,o_0,s_1}(S_{t+1:T}=s_{t+1:T},A_{t+1:T}=a_{t+1:T}|S_t=s_t,A_t=a_t,O_t=o_t,B_t=b_t),
\end{equation*}
where the value of $o_0$ on the RHS is arbitrary. By convention, the boundary condition is
\begin{equation}\label{eq:backwardT}
\beta^\theta_{T|T}(o_T,b_T)=(2\left|\calo\right|)^{-1}.
\end{equation}
The smoothing distribution is defined as
\begin{equation*}
\gamma^{\theta}_{\mu,t|T}(o_t,b_t)\defeq z_{\gamma,\mu}^{\theta}\mathbb E_{O_0\sim \mu(\cdot|s_1)}[\mathbb P_{\theta,O_0,s_1}(S_{2:T}=s_{2:T},A_{1:T}=a_{1:T},O_t=o_t,B_t=b_t)].
\end{equation*}
It can be easily verified that the normalizing constant does not depend on $t$. 

Finally, for all $\theta$, $\mu$ and $t\in[2:T]$, with any input arguments $o_{t-1}$ and $b_t$, we define the two-step smoothing distribution as
\begin{equation*}
\tilde \gamma^{\theta}_{\mu,t|T}(o_{t-1},b_t)\defeq z_{\gamma,\mu}^{\theta}\mathbb E_{O_0\sim \mu(\cdot|s_1)}[\mathbb P_{\theta,O_0,s_1}(S_{2:T}=s_{2:T},A_{1:T}=a_{1:T},O_{t-1}=o_{t-1},B_t=b_t)],
\end{equation*}
where $z_{\gamma,\mu}^{\theta}$ is the same normalizing constant as the one for the smoothing distribution $\gamma^{\theta}_{\mu,t|T}$. 

The quantities above can be computed using the forward-backward recursion. For simplicity, we omit normalizing constants by using the proportional symbol $\propto$. The proof is deferred to Appendix~\ref{subsection:proofthm2}. 

\begin{theorem}[Forward-backward smoothing]\label{thm:fb}
For all $\theta\in\calTheta$ and $\mu\in\mathcal{M}$, with any input arguments on the LHS, 
\begin{enumerate}[leftmargin=*]
\item (Forward recursion) $\forall t\in[2:T]$, 
\begin{equation}\label{eq:forward}
\alpha^\theta_{\mu,t}(o_t,b_t)\propto\sum_{o_{t-1},b_{t-1}}\pi_b(b_t|s_t,o_{t-1};\theta_b)\bar\pi_{hi}(o_t|s_t,o_{t-1},b_t;\theta_{hi})\pi_{lo}(a_t|s_t,o_t;\theta_{lo})\alpha^\theta_{\mu,t-1}(o_{t-1},b_{t-1}).
\end{equation}
When $t=1$,
\begin{equation}\label{eq:forward1}
\alpha^\theta_{\mu,1}(o_1,b_1)\propto\mathbb E_{O_0\sim \mu(\cdot|s_1)}[\pi_b(b_1|s_1,O_0;\theta_b)\bar\pi_{hi}(o_1|s_1,O_0,b_1;\theta_{hi})\pi_{lo}(a_1|s_1,o_1;\theta_{lo})].
\end{equation}
\item (Backward recursion) $\forall t\in[1:T-1]$, 
\begin{multline}\label{eq:backward}
\beta^\theta_{t|T}(o_t,b_t)\propto\sum_{o_{t+1},b_{t+1}}\pi_b(b_{t+1}|s_{t+1},o_{t};\theta_b)\bar\pi_{hi}(o_{t+1}|s_{t+1},o_{t},b_{t+1};\theta_{hi})\\ \times\pi_{lo}(a_{t+1}|s_{t+1},o_{t+1};\theta_{lo})\beta^\theta_{t+1|T}(o_{t+1},b_{t+1}).
\end{multline}
\item (Smoothing) $\forall t\in[1:T]$, 
\begin{equation}\label{eq:smoothing}
\gamma^{\theta}_{\mu,t|T}(o_t,b_t)\propto\alpha^\theta_{\mu,t}(o_t,b_t)\beta^\theta_{t|T}(o_t,b_t).
\end{equation}
\item (Two-step smoothing) $\forall t\in[2:T]$, 
\begin{multline}\label{eq:smoothing2}
\tilde \gamma^{\theta}_{\mu,t|T}(o_{t-1},b_t)\propto\pi_b(b_t|s_t,o_{t-1};\theta_{b})\left[\sum_{o_t}\bar\pi_{hi}(o_{t}|s_{t},o_{t-1},b_{t};\theta_{hi})\pi_{lo}(a_{t}|s_{t},o_{t};\theta_{lo})\beta^\theta_{t|T}(o_t,b_t)\right]\\ \times\left[\sum_{b_{t-1}}\alpha^\theta_{\mu,t-1}(o_{t-1},b_{t-1})\right].
\end{multline}
\end{enumerate}
\end{theorem}

\subsection{Parameter update}

For all $\theta,\theta'\in\calTheta$ and $\mu\in\mathcal{M}$, the (finite sample) $Q$-function is defined as
\begin{multline}\label{eq:Q}
Q_{\mu,T}(\theta'|\theta)\defeq\frac{1}{T}\Bigg\{\sum_{t=2}^T\sum_{o_{t-1},b_t}\tilde \gamma^{\theta}_{\mu,t|T}(o_{t-1},b_t)\left[\log \pi_{b}(b_t|s_t,o_{t-1};\theta'_b)\right]+\sum_{t=1}^T\sum_{o_t,b_t}\gamma^{\theta}_{\mu,t|T}(o_t,b_t)\\
\times[\log \pi_{lo}(a_t|s_t,o_t;\theta'_{lo})]
+\sum_{t=1}^T\sum_{o_t}\gamma^{\theta}_{\mu,t|T}(o_t,b_t=1)[\log \pi_{hi}(o_t|s_t;\theta'_{hi})]\Bigg\}.
\end{multline}

The parameter update is performed as $\theta^{(n)}\in\argmax_{\theta\in\calTheta} Q_{\mu,T}(\theta|\theta^{(n-1)})$, which may not be unique. Since $\calTheta$ is compact and $Q_{\mu,T}(\theta'|\theta)$ is continuous with respect to $\theta'$, the maximization is well-posed. Note that our definition of $Q_{\mu,T}(\theta'|\theta)$ is an approximation of the standard definition of $Q$-function in the EM literature. See Appendix~\ref{subsection:q} for a detailed discussion. 

\subsection{Generalization to continuous spaces}

Although we require finite state and action space for our theoretical analysis, Algorithm~\ref{algorithm} can be readily generalized to continuous $\cals$ and $\cala$: we only need to replace $\pi_{lo}$ by a density function. However, generalization to continuous option space requires a substantially different algorithm. The forward-backward smoothing procedure in Theorem~\ref{thm:fb} involves integrals rather than sums, and Sequential Monte Carlo (SMC) techniques need to be applied. Fortunately, it is widely accepted that a finite option space is reasonable in the options framework, since the options need to be distinct and separate \citep{daniel2016hierarchical}.

\section{Performance guarantee}\label{section:guarantee}

Our analysis of Algorithm~\ref{algorithm} has the following structure. We first prove the stochastic convergence of the $Q$-function $Q_{\mu,T}(\theta'|\theta)$ to a population $Q$-function $\bar Q(\theta'|\theta)$, leading to a well-posed definition of the population version algorithm. This step is our major theoretical contribution. With additional assumptions, the \emph{first-order stability} condition is constructed, and techniques in \citep{balakrishnan2017statistical} can be applied to show the convergence of the population version algorithm. The remaining step is to analyze Algorithm~\ref{algorithm} as a perturbed form of its population version, which requires a high probability bound on the distance between their parameter updates. We can establish the strong consistency of the parameter update of Algorithm~\ref{algorithm} as an estimator of the parameter update of the population version algorithm. Therefore, the existence of such a high probability bound can be proved for large enough $T$. However, the analytical expression of this bound requires knowledge of the specific parameterization of $\{\bar\pi_{hi},\pi_{lo}, \pi_b\}$, which is not available in this general context of discussion. 

Concretely, we first analyze the asymptotic behavior of the $Q$-function $Q_{\mu,T}(\theta'|\theta)$ as $T\rightarrow\infty$. From Assumption~\ref{as:initial}, the observation sequence $\{s_{1:T},a_{1:T}\}$ is generated from a stationary Markov chain $\{X_t\}_{t=1}^\infty\defeq\{S_t,A_t,O_t,B_t\}_{t=1}^\infty$. Let $\mathcal{X}=\mathcal S\times \mathcal A\times\mathcal O\times \{0,1\}$ be its state space. Using Kolmogorov's extension theorem, we can extend this one-sided Markov chain to the index set $\Z$ and define a unique probability measure $\P_{\theta^*,\nu^*}$ over the sample space $\mathcal{X}^{\mathbb Z}$. Any observation sequence $\{s_{1:T},a_{1:T}\}$ can be regarded as a segment of an infinite length sample path $\omega\in\mathcal{X}^{\mathbb Z}$. Therefore, if the observation sequence is not specified, $Q_{\mu,T}(\theta'|\theta)$ is a random variable with underlying probability measure $\P_{\theta^*,\nu^*}$.

One caveat is that the definition of $Q_{\mu,T}(\theta'|\theta)$ from Section~\ref{section:alg} fails for some $\omega\in\mathcal{X}^{\mathbb Z}$. To fix this issue, define the set of \emph{proper} sample paths as
\begin{equation}\label{eq:omega}
\Omega=\left\{\omega\in\mathcal{X}^\Z;P(s_{t+1}|s_t,a_t)>0,\forall t\in\Z\right\}.
\end{equation}

Note that $\P_{\theta^*,\nu^*}(\Omega)=1$; therefore, working on $\Omega$ is probabilistically equivalent to working on $\mathcal{X}^\Z$. For all $\omega\in\Omega$, $Q_{\mu,T}(\theta'|\theta)$ follows the definition from Section~\ref{section:alg}; for other sample paths, $Q_{\mu,T}(\theta'|\theta)$ is defined arbitrarily. In this way, $Q_{\mu,T}(\theta'|\theta)$ becomes a well-defined random variable. Its stochastic convergence is characterized in the following theorem. 

\begin{theorem}[The stochastic convergence of the $Q$-function]\label{thm:existence} With Assumption~\ref{as:nondeg}, \ref{as:continuity} and \ref{as:initial}, there exists a real-valued function $\bar Q(\theta'|\theta)$ defined on the domain $\theta'\in\tilde\calTheta$ and $\theta\in\calTheta$ such that
\begin{enumerate}[leftmargin=*]
\item For all $\theta\in\calTheta$, $\bar Q(\theta'|\theta)$ is continuously differentiable with respect to $\theta'\in\tilde\calTheta$. Moreover, the set $\argmax_{\theta'\in\calTheta}\bar Q(\theta'|\theta)$ is nonempty.
\item As $T\rightarrow\infty$, 
\begin{equation*}
\sup_{\theta,\theta'\in\calTheta}\sup_{\mu\in\mathcal{M}}\left|Q_{\mu,T}(\theta'|\theta;\omega)-\bar Q(\theta'|\theta)\right|\rightarrow 0,~P_{\theta^*,\nu^*}\text{-a.s.}
\end{equation*}
\end{enumerate}
\end{theorem}

We name $\bar Q(\theta'|\theta)$ as the population $Q$-function. The analytical expressions of $\bar Q(\theta'|\theta)$ and $\nabla\bar Q(\theta'|\theta)$ are provided in Appendix~\ref{subsection:asymptoticqfun}, where the complete version of the above theorem (Theorem~\ref{thm:stronger}) is proved. In the following, we provide a high level sketch of the main idea. 

\begin{proof}[Proof Sketch]
The main difficulty of the proof is that, $Q_{\mu,T}(\theta'|\theta)$ defined in (\ref{eq:Q}) is (roughly) the average of $T$ terms, with each term dependent on the entire observation sequence; as $T\rightarrow\infty$, all the terms keep changing such that the law of large numbers cannot be applied directly. As a solution, we approximate $\gamma^{\theta}_{\mu,t|T}$ and $\tilde\gamma^{\theta}_{\mu,t|T}$ with smoothing distributions in an infinitely extended graphical model independent of $T$, resulting in an approximated $Q$-function (still depends on $T$). The techniques adopted in this step are analogous to \emph{Markovian decomposition} and \emph{uniform forgetting} in the HMM literature \citep{cappe2006inference,van2008hidden}. The limiting behavior of the approximated $Q$-function is the same as that of $Q_{\mu,T}(\theta'|\theta)$, since their difference vanishes as $T\rightarrow\infty$. For the approximated $Q$-function, we can apply the ergodic theorem since the smoothing distributions no longer depend on $T$. 
\end{proof}

The population version of Algorithm~\ref{algorithm} has parameter updates $\theta^{(n)}\in\argmax_{\theta\in\calTheta} \bar Q(\theta|\theta^{(n-1)})$. To characterize the local convergence of Algorithm~\ref{algorithm} and its population version, we impose the following assumptions for the remainder of Section~\ref{section:guarantee}. 

\begin{assumption}[Strong concavity]\label{as:concavity} There exists $\lambda>0$ such that for all $\theta_1,\theta_2\in\calTheta$, 
\begin{equation*}
\bar Q(\theta_1|\theta^*)-\bar Q(\theta_2|\theta^*)-\langle\nabla \bar Q(\theta_2|\theta^*),\theta_1-\theta_2\rangle\leq -\frac{\lambda}{2}\norm{\theta_1-\theta_2}_2^2. 
\end{equation*}
\end{assumption}

For any $r>0$, let $\calTheta_r\defeq \{\theta;\theta\in\calTheta, \norms{\theta-\theta^*}_2\leq r\}$. 

\begin{assumption}[Additional local assumptions]\label{as:local} There exists $r>0$ such that
\begin{enumerate}[leftmargin=*]
\item (Identifiability) For all $\theta\in\calTheta_r$, the set $\argmax_{\theta'\in\calTheta} \bar Q(\theta'|\theta)$ has a unique element $\bar M(\theta)$. Moreover, for all $\eps>0$, with the convention that $\sup_{\theta'\in\varnothing}\bar Q(\theta'|\theta)=-\infty$, we have
\begin{equation*}
\inf_{\theta\in\calTheta_r}\bigg[\bar Q(\bar M(\theta)|\theta)-\sup_{\theta'\in\calTheta;\norms{\theta'-\bar M(\theta)}_2\geq\eps}\bar Q(\theta'|\theta)\bigg]>0.
\end{equation*}
\item (Uniqueness of finite sample parameter updates) For all $\theta\in\calTheta_r$, $T\geq 2$ and $\mu\in\mathcal{M}$, $P_{\theta^*,\nu^*}$-almost surely, the set $\argmax_{\theta'\in\calTheta} Q_{\mu,T}(\theta'|\theta;\omega)$ has a unique element $M_{\mu,T}(\theta;\omega)$.
\end{enumerate}
\end{assumption}

\paragraph{On the additional assumptions.} In Assumption~\ref{as:concavity}, we require the strong concavity of $\bar Q(\cdot|\theta^*)$ over the entire parameter space since the maximization step in our algorithm is global. Such a requirement could be avoided: if the maximization step is replaced by a gradient update (Gradient EM), then $\bar Q(\cdot|\theta^*)$ only needs to be strongly concave in a small region around $\theta^*$. The price to pay is to assume knowledge on structural constants of $\bar Q(\cdot|\theta^*)$ (Lipschitz constant and strong concavity constant). See \citep{balakrishnan2017statistical} for an analysis of the gradient EM algorithm.

Nonetheless, we expect the following to hold in certain cases of tabular parameterization: for all $\theta\in\calTheta$, the function $\bar Q(\cdot|\theta)$ is strongly concave over $\calTheta$ (see the end of Appendix~\ref{subsection:asymptoticqfun}). From this condition, Assumption~\ref{as:concavity} and \ref{as:local}.1 directly follow. Assumption~\ref{as:local}.2 holds as well; in fact, it is a quite mild assumption due to the sample-based nature of $Q_{\mu,T}(\theta'|\theta;\omega)$. 

The next step is to characterize the convergence of the population version algorithm. 

\begin{theorem}[Convergence of the population version algorithm]\label{thm:population} With all the assumptions, 

\begin{enumerate}[leftmargin=*]
\item (First-order stability) There exists $\gamma>0$ such that for all $\theta\in\calTheta_r$, 
\begin{equation*}
\norm{\nabla\bar Q(\bar M(\theta)|\theta)-\nabla\bar Q(\bar M(\theta)|\theta^*)}_2\leq \gamma\norm{\theta-\theta^*}_2.
\end{equation*}
\item (Contraction) Let $\kappa=\gamma/\lambda$. For all $\theta\in\calTheta_r$, 
\begin{equation*}
\norm{\bar M(\theta)-\theta^*}_2\leq \kappa\norm{\theta-\theta^*}_2. 
\end{equation*}
If $\kappa<1$, the population version algorithm converges linearly to the true parameter $\theta^*$. 
\end{enumerate}
\end{theorem}

The proof is given in Appendix~\ref{subsection:proofpopulation}, where we also show an upper bound on $\gamma$. The idea mirrors that of \citep[Theorem~4]{balakrishnan2017statistical} with problem-specific modifications. Algorithm~\ref{algorithm} can be regarded as a perturbed form of this population version algorithm, with convergence characterized in the following theorem. 

\begin{theorem}[Performance guarantee for Algorithm~\ref{algorithm}]\label{thm:perturbed} With all the assumptions, if $\kappa<1$ we have

\begin{enumerate}[leftmargin=*]
\item For all $\Delta\in(0,(1-\kappa)r]$ and $q\in(0,1)$, there exists $\underline T(\Delta,q)\in\N_+$ such that the following statement is true. If the observation length $T\geq\underline T(\Delta,q)$, then with probability at least $1-q$, 
\begin{equation*}
\sup_{\theta\in\calTheta_r}\sup_{\mu\in\mathcal{M}}\norm{M_{\mu,T}(\theta;\omega)-\bar M(\theta)}_2\leq \Delta. 
\end{equation*}
\item If $T\geq \underline T(\Delta, q)$, Algorithm~\ref{algorithm} with any $\mu\in\mathcal{M}$ has the following performance guarantee. If $\theta^{(0)}\in\calTheta_r$, then with probability at least $1-q$, for all $n\in\N_+$,
\begin{equation*}
\norms{\theta^{(n)}-\theta^*}_2\leq \kappa^n\norms{\theta^{(0)}-\theta^*}_2+(1-\kappa)^{-1}\Delta.
\end{equation*}
\end{enumerate}
\end{theorem}

The proof is provided in Appendix~\ref{subsection:proofperturbed}. Essentially, we use Theorem~\ref{thm:existence} to show the uniform (in $\theta$ and $\mu$) strong consistency of $M_{\mu,T}(\theta;\omega)$ as an estimator of $\bar M(\theta)$, following the standard analysis of \emph{M-estimators}. A direct corollary of this argument is the high probability bound on the difference between $M_{\mu,T}(\theta;\omega)$ and $\bar M(\theta)$, as shown in the first part of the theorem. Combining this high probability bound with Theorem~\ref{thm:population} and \citep[Theorem~5]{balakrishnan2017statistical} yields the final performance guarantee. 

Theorem~\ref{thm:perturbed} has two practical implications. First, under regularity conditions, with large enough $T$, Algorithm~\ref{algorithm} can converge with arbitrarily high probability to an arbitrarily small norm ball around the true parameter. In other words, with enough samples, the EM approach can recover the true parameter of the expert policy arbitrarily well. Second, the estimation error (upper bound) decreases exponentially in the initial phase of the algorithm. In this regard, a practitioner can allocate his computational budget accordingly.

One limitation of our analysis is that the condition $\kappa<1$ is hard to verify for a practical parameterization of the expert policy. This is typical in the theory of EM algorithms: even in the case of i.i.d. samples, characterizing the contraction coefficient is intractable except for a few simple parametric models. Nonetheless, such a condition strengthens our intuition on when the EM approach to HIL works: $\bar Q(\theta'|\theta)$ should have a large curvature with respect to $\theta'$, and the function should not change much with respect to $\theta$ around $\theta^*$. In the next section, we present a numerical example to qualitatively demonstrate our result. 

\section{Numerical example}\label{section:experiment}

In this section, we qualitatively demonstrate our theoretical result through an example. Here, we value clarity over completeness, therefore large-scale experiments are deferred to future works. 

\begin{wrapfigure}{R}{0.5\textwidth}
\centering
\includegraphics[width=220pt]{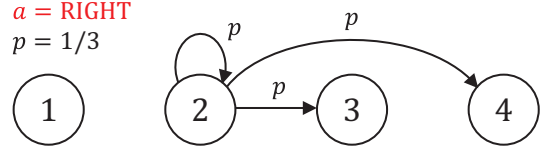}
\caption{The state transition structure.}\label{figure:MDP}
\end{wrapfigure}

Consider the Markov Decision Process (MDP) illustrated in Figure~\ref{figure:MDP}. There are four states, numbered from left to right as 1 to 4. At any state $s_t\in[1:4]$, there are two allowable actions: LEFT and RIGHT. If $a_t=\textrm{RIGHT}$, then the next state is sampled uniformly from the states on the right of state $s_t$ (including $s_t$ itself). Symmetrically, if $a_t=\textrm{LEFT}$, then the next state is sampled uniformly from the states on the left of state $s_t$ (including $s_t$). 

Suppose an expert applies the following options with failure policy with parameters $(\theta^*_{hi},\theta^*_{lo},\theta^*_b)=(0.6,0.7,0.8)$ and $\zeta=0.1$. The option space has two elements: LEFTEND and RIGHTEND. $\pi_{hi}(o_t=\textrm{LEFTEND}|s_t;\theta_{hi})$ equals $\theta_{hi}$ if $s_t=1,2$, and $1-\theta_{hi}$ if $s_t=3,4$. For all $s_t$, $\pi_{lo}(a_t=\textrm{LEFT}|s_t,o_t=\textrm{LEFTEND};\theta_{lo})=\pi_{lo}(a_t=\textrm{RIGHT}|s_t,o_t=\textrm{RIGHTEND};\theta_{lo})=\theta_{lo}$. $\pi_{b}(b_t=1|s_t,o_t=\textrm{LEFTEND};\theta_{b})$ equals $\theta_b$ if $s_t=1$, and $1-\theta_b$ otherwise. Symmetrically, $\pi_{b}(b_t=1|s_t,o_t=\textrm{RIGHTEND};\theta_{b})$ equals $\theta_b$ if $s_t=4$, and $1-\theta_b$ otherwise. Intuitively, the high level policy directs the agent to states 1 and 4, and the option terminates with high probability when the corresponding target state is reached. 

In our experiment, the parameter spaces $\calTheta_{hi}$, $\calTheta_{lo}$ and $\calTheta_{b}$ are all equal to the interval $[0.1,0.9]$. The initial parameter estimate $(\theta^{(0)}_{hi},\theta^{(0)}_{lo},\theta^{(0)}_b)=(0.5,0.6,0.7)$. For all $s_1$, $\mu(o_0=\textrm{RIGHTEND}|s_1)=1$. 

We investigate the behavior of $\norms{\theta^{(n)}-\theta^*}_2$ as a random variable dependent on $n$ and $T$. 50 sample paths of length $T$ are sampled from (approximately) the stationary Markov chain induced by the expert policy, with $T\in\{5000,8000,10000\}$. After running Algorithm~\ref{algorithm} with any sample path $\omega$ and any $T$, we obtain a sequence $\{\norms{\theta^{(n)}-\theta^*}_2;\omega,T\}_{n\in[0:N]}$. Let $err(n,T)$ be the average of $\norms{\theta^{(n)}-\theta^*}_2$ for fixed $n$ and $T$, over the 50 sample paths. The result is shown in Figure~\ref{figure:exp_1}.

\begin{figure}[ht]
    \centering
    \includegraphics[width=450pt]{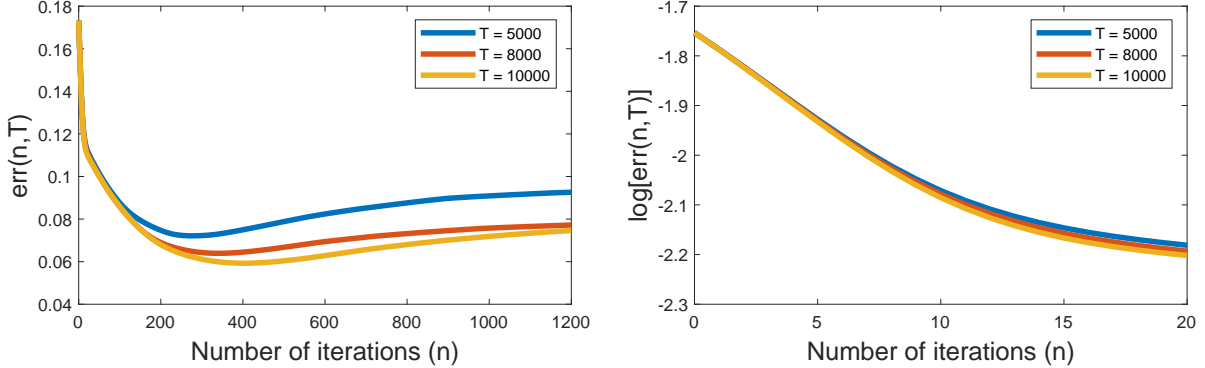}
    \caption[]{Plots of $err(n,T)$ and $\log err(n,T)$ with varying $n$ and $T$.}\label{figure:exp_1}
\end{figure}

Assumption~\ref{as:nondeg}, \ref{as:continuity}, \ref{as:initial} and \ref{as:local}.2 hold in this example, and we speculate that Assumption~\ref{as:concavity} and \ref{as:local}.1 hold as well. The condition $\kappa<1$ cannot be verified, but the empirical result exhibits patterns consistent with the performance guarantee, even though rigorously Theorem~\ref{thm:perturbed} is not applicable. First, $err(n,T)$ decreases exponentially in the early phase of the algorithm. Second, as $T$ increases, Algorithm~\ref{algorithm} achieves better performance. 

An observation is worth mentioning as a separate note: for $n>300$, $err(n,T)$ first slightly increases, then levels off. This is due to the parameter estimate on some sample paths converging to bad stationary points of the finite sample likelihood function, which suggests that early stopping could be helpful in practice. Omitted details and additional experiments are provided in Appendix~\ref{section:additional}, where we also investigate, for example, the effect of $\mu$ and random initialization on the performance of Algorithm~\ref{algorithm}. 

\section{Conclusions}

In this paper, we investigate the EM approach to HIL from a theoretical perspective. We prove that under regularity conditions, the proposed algorithm converges with high probability to a norm ball around the true parameter. To our knowledge, this is the first performance guarantee for an HIL algorithm that only observes primitive state-action pairs. Future works could further investigate the practical performance of this approach, especially its scalability in complicated environments. 

\subsection*{Acknowledgements}
We thank the anonymous reviewers for their constructive comments. Z.Z. thanks Tianrui Chen for helpful discussions. The research was partially supported by the NSF under grants DMS-1664644,
CNS-1645681, and IIS-1914792, by the ONR under grant N00014-19-1-2571, by the NIH
under grants R01 GM135930 and UL54 TR004130, and by the DOE under grant
DE-AR-0001282. 

\bibliography{Provable_HIL}

\newpage
\section*{Appendix}
\appendix

\paragraph{Organization.}
Appendix~\ref{section:proof1} presents discussions that motivate Assumption~\ref{as:initial}. In particular, we show that Assumption~\ref{as:initial} approximately holds in a particular class of environment. Appendix~\ref{section:techarg} provides details on Algorithm~\ref{algorithm}, including the comparison with the existing algorithm from \citep{daniel2016probabilistic}, the forward-backward implementation and the derivation of the $Q$-function from (\ref{eq:Q}). In Appendix~\ref{section:guaranteedetailed}, we prove our theoretical results from Section~\ref{section:guarantee}. Technical lemmas involved in the proofs are deferred to Appendix~\ref{section:technical}. Finally, Appendix~\ref{section:additional} presents details of our numerical example omitted from Section~\ref{section:experiment}. 

\section{Discussion on Assumption~\ref{as:initial}}\label{section:proof1}

In this section we justify Assumption~\ref{as:initial} in a particular class of environment. Consider the stochastic process $\{X_t;\theta\}_{t=1}^\infty=\{S_t,A_t,O_t,B_t;\theta\}_{t=1}^\infty$ generated by any $(o_0,s_1)$ and an options with failure hierarchical policy with parameter $\theta$. It is a Markov chain with its transition kernel parameterized by $\theta$, and its state space $\mathcal{X}=\mathcal S\times \mathcal A\times\mathcal O\times \{0,1\}$ is finite. Denote its one step transition kernel as $Q_\theta$ and its $t$ step transition kernel as $Q^t_\theta$. In the following, we show that $\{X_t;\theta\}_{t=1}^\infty$ is uniformly ergodic when the environment meets the reachability assumption: $\forall s_t,s_{t+1}\in \cals$, there exists $a_t\in \cala$ such that $P(s_{t+1}|s_t,a_t)>0$.

\begin{proposition}[Ergodicity]\label{lemma:ergo}
With Assumption~\ref{as:nondeg}, \ref{as:continuity} and the reachability assumption stated above, for all $\theta\in\calTheta$, a Markov chain with transition kernel $Q_\theta$ has a unique stationary distribution $\nu_\theta$. There exist constants $\alpha\in(0,1)$ and $C>0$ such that for all $\theta\in\calTheta$ and $t\in\mathbb N_+$,
\begin{equation*}
\sup_{\theta\in\calTheta}\max_{x\in\mathcal{X}}\norm{Q_\theta^t(x,\cdot)-\nu_\theta}_\tv\leq C\alpha^t.
\end{equation*}
\end{proposition}

\begin{proof}[Proof of Proposition~\ref{lemma:ergo}]
We start by analyzing the irreducibility of the Markov chain $\{X_t;\theta\}_{t=1}^\infty$ with any $\theta$. Denote the probability measure on the natural filtered space as $\mathbb P_X$. The dependency on $\theta$ is dropped for a cleaner notation, since the following proof holds for all $\theta\in\calTheta$. For any $x,\tilde x\in\mathcal{X}$, let $x=(s,a,o,b)$ and $\tilde x=(\tilde s,\tilde a,\tilde o,\tilde b)$. For any time $t$, 
\begin{multline*}
\mathbb P_X(X_{t+2}=\tilde x|X_t=x)=\sum_{\bar s\in\mathcal S,\bar a\in\mathcal A}\mathbb P_X(X_{t+2}=\tilde x|X_t=x,S_{t+1}=\bar s,A_{t+1}=\bar a)\\
\times\mathbb P_X(S_{t+1}=\bar s,A_{t+1}=\bar a|X_t=x).
\end{multline*}

From Assumption~\ref{as:nondeg}, there exists a state $\bar s$ such that $\forall \bar a\in\mathcal A$, $\mathbb P_X(S_{t+1}=\bar s,A_{t+1}=\bar a|X_t=x)>0$. Consider the first factor in the sum, 
\begin{multline*}
\mathbb P_X(X_{t+2}=\tilde x|X_t=x,S_{t+1}=\bar s,A_{t+1}=\bar a)=\mathbb P_X(S_{t+2}=\tilde s|S_{t+1}=\bar s,A_{t+1}=\bar a)\\
\times \mathbb P_X(B_{t+2}=\tilde b,O_{t+2}=\tilde o,A_{t+2}=\tilde a|X_t=x,S_{t+1}=\bar s,A_{t+1}=\bar a,S_{t+2}=\tilde s). 
\end{multline*}
From Assumption~\ref{as:nondeg}, the second term on the RHS is positive for all $\bar s\in\cals$ and $\bar a\in\cala$. From the reachability assumption, for any $\bar s$ there exists an action $\bar a$ such that $\mathbb P_X(S_{t+2}=\tilde s|S_{t+1}=\bar s,A_{t+1}=\bar a)>0$. As a result, for any $x,\tilde x\in\mathcal{X}$, $\mathbb P_X(X_{t+2}=\tilde x|X_t=x)>0$, and the considered Markov chain is irreducible. 

As shown above, for all $\theta\in\calTheta$, $\min_{x,\tilde x\in\mathcal X}Q^2_\theta(x,\tilde x)>0$ where $Q^2_\theta$ is the two step transition kernel of the Markov chain $\{X_t;\theta\}_{t=1}^\infty$. Due to Assumption~\ref{as:continuity}, $\min_{x,\tilde x\in\mathcal X}Q^2_\theta(x,\tilde x)$ is continuous with respect to $\theta$. Moreover, since $\calTheta$ is compact, if we let $\delta=\inf_{\theta\in\calTheta}\min_{x,\tilde x\in\mathcal X}Q^2_\theta(x,\tilde x)$ we have $\delta>0$. The classical Doeblin-type condition can be constructed as follows. For all $\theta\in\calTheta$ and $x,\tilde x\in\mathcal X$, with any probability measure $\nu$ over the finite sample space $\mathcal X$, 
\begin{equation}
Q^2_\theta(x,\tilde x)\geq \delta\nu(\tilde x).\label{eq:doeblin}
\end{equation}

A Markov chain convergence result is restated in the following lemma, tailored to our need. 
\begin{lemma}[\citep{cappe2006inference}, Theorem~4.3.16 restated]
With the Doeblin-type condition in (\ref{eq:doeblin}), the Markov chain $\{X_t;\theta\}_{t=1}^\infty$ with any $\theta\in\calTheta$ has a unique stationary distribution $\nu_\theta$. Moreover, for all $\theta\in\calTheta$, $x\in\mathcal X$ and $t\in\N_+$, 
\begin{equation*}
\norm{Q^t_\theta(x,\cdot)-\nu_\theta}_\tv\leq(1-\delta)^{\floor{t/2}}.
\end{equation*}
\end{lemma}

Letting $C=(1-\delta)^{-1}$ and $\alpha=(1-\delta)^{1/2}$, we have
\begin{equation*}
\sup_{\theta\in\calTheta}\max_{x_1\in\mathcal{X}}\norm{Q_\theta^t(x_1,\cdot)-\nu_\theta}_\tv\leq(1-\delta)^{\floor{t/2}}\leq C\alpha^t.\qedhere
\end{equation*}
\end{proof}

Proposition~\ref{lemma:ergo} shows that in $\{X_t;\theta\}_{t=1}^\infty$, the initial distribution (of $X_1$) is not very important since the distribution of $X_t$ converges to $\nu_\theta$ uniformly with respect to $X_1$ and $\theta$. As a result, $\{O_{t-1},S_t\}_{t=1}^\infty$ also converges to the unique limiting distribution, regardless of the initial distribution. When sampling the observation sequence from the expert, we can always start sampling late enough such that Assumption~\ref{as:initial} is approximately satisfied. Note that the proof of Proposition~\ref{lemma:ergo} does not use the failure mechanism imposed on the hierarchical policy, implying that the result also holds for the standard options framework.

\section{Details of the algorithm}\label{section:techarg}

\subsection{An error in the existing algorithm}\label{subsection:problem}

First, we point out a technicality when comparing Algorithm~\ref{algorithm} to the algorithm from \citep{daniel2016probabilistic}. The algorithm from \citep{daniel2016probabilistic} learns a hierarchical policy following the standard options framework, not the options with failure framework considered in Algorithm~\ref{algorithm}. To draw direct comparison, we need to let $\zeta=0$ in Algorithm~\ref{algorithm}. However, an error in the existing algorithm can be demonstrated without referring to $\zeta$. 

For simplicity, consider $O_0$ fixed to $o_0\in\calo$; let $2\leq t\leq T-1$. Then, according to the definitions in \citep{daniel2016probabilistic}, the (unnormalized) forward message is defined as
\begin{equation*}
\alpha^\theta_{t}(o_t,b_t)=\P_{\theta,o_0,s_1}(A_{1:t}=a_{1:t},O_t=o_t,B_t=b_t|S_{2:t}=s_{2:t}).
\end{equation*}
The (unnormalized) backward message is defined as
\begin{equation*}
\beta^\theta_{t|T}(o_t,b_t)=\P_{\theta,o_0,s_1}(A_{t+1:T}=a_{t+1:T}|S_{t+1:T}=s_{t+1:T},O_t=o_t,B_t=b_t).
\end{equation*}
The smoothing distribution is defined as
\begin{equation*}
\gamma^{\theta}_{t|T}(o_t,b_t)= \mathbb P_{\theta,o_0,s_1}(O_t=o_t,B_t=b_t|S_{2:T}=s_{2:T},A_{1:T}=a_{1:T}).
\end{equation*}

We use the proportional symbol $\propto$ to represent normalizing constants independent of $o_t$ and $b_t$. \citep{daniel2016probabilistic} claims that, for any $o_t$ and $b_t$, 
\begin{equation*}
\gamma^{\theta}_{t|T}(o_t,b_t)\propto \alpha^\theta_{t}(o_t,b_t)\beta^\theta_{t|T}(o_t,b_t).
\end{equation*}
However, applying Bayes' formula, it follows that
\begin{equation*}
\gamma^{\theta}_{t|T}(o_t,b_t)\propto \mathbb P_{\theta,o_0,s_1}(A_{1:T}=a_{1:T}|S_{2:T}=s_{2:T},O_t=o_t,B_t=b_t)\mathbb P_{\theta,o_0,s_1}(O_t=o_t,B_t=b_t|S_{2:T}=s_{2:T}).
\end{equation*}
Using the Markov property, 
\begin{multline*}
\mathbb P_{\theta,o_0,s_1}(A_{1:T}=a_{1:T}|S_{2:T}=s_{2:T},O_t=o_t,B_t=b_t)=\\\mathbb P_{\theta,o_0,s_1}(A_{1:t}=a_{1:t}|S_{2:T}=s_{2:T},O_t=o_t,B_t=b_t)\\\times\mathbb P_{\theta,o_0,s_1}(A_{t+1:T}=a_{t+1:T}|S_{2:T}=s_{2:T},O_t=o_t,B_t=b_t).
\end{multline*}
Therefore, 
\begin{equation*}
\gamma^{\theta}_{t|T}(o_t,b_t)\propto \mathbb P_{\theta,o_0,s_1}(A_{1:t}=a_{1:t},O_t=o_t,B_t=b_t|S_{2:T}=s_{2:T})\beta^\theta_{t|T}(o_t,b_t).
\end{equation*}
Applying Bayes' formula again, it follows that
\begin{align*}
&\mathbb P_{\theta,o_0,s_1}(A_{1:t}=a_{1:t},O_t=o_t,B_t=b_t|S_{2:T}=s_{2:T})\\
\propto~ &\mathbb P_{\theta,o_0,s_1}(A_{1:t}=a_{1:t},O_t=o_t,B_t=b_t|S_{2:t}=s_{2:t})\\
&\hspace{6em}\times\mathbb P_{\theta,o_0,s_1}(S_{t+1:T}=s_{t+1:T}|S_{2:t}=s_{2:t},A_{1:t}=a_{1:t},O_t=o_t,B_t=b_t)\\
=~&\alpha^\theta_{t}(o_t,b_t)\mathbb P_{\theta,o_0,s_1}(S_{t+1:T}=s_{t+1:T}|S_{t}=s_{t},A_{t}=a_{t},O_t=o_t,B_t=b_t).
\end{align*}

For the claim in \citep{daniel2016probabilistic} to be true, $\mathbb P_{\theta,o_0,s_1}(S_{t+1:T}=s_{t+1:T}|S_{t}=s_{t},A_{t}=a_{t},O_t=o_t,B_t=b_t)$ should not depend on $o_t$ and $b_t$. Clearly this requirement does not hold in most cases, since the likelihood of the future observation sequence should depend on the currently applied option.

\subsection{Proof of Theorem~\ref{thm:fb}}\label{subsection:proofthm2}

We drop the dependency on $\theta$, since the following proof holds for all $\theta\in\calTheta$. The proportional symbol $\propto$ is used to replace a multiplier term that depends on the context.

\vspace{1em}\noindent
1. (Forward recursion)

\noindent First consider any fixed $o_0$. For a cleaner notation, we use $p$ as an abbreviation of $\mathbb P_{\theta,o_0,s_1}$. Let $H_1$, $H_2$ be any two subsets of $\{S_t,A_t,O_t,B_t\}_{t=1}^T$, and let $h_1$, $h_2$ be the sets of values generated from $H_1$ and $H_2$, respectively, such that the uppercase symbols are replaced by the lowercase symbols. ($H_1$ and $H_2$ are two sets of random variables; $h_1$ and $h_2$ are two sets of values of random variables.) Then, for all $(o_0,s_1)$, $p$ is defined as
\begin{equation*}
p(h_1|h_2,o_0,s_1)\defeq\mathbb P_{\theta,o_0,s_1}(H_1=h_1|H_2=h_2).
\end{equation*}
If the RHS does not depend on $o_0$ and $s_1$, we can omit it on the LHS by using $p(h_1|h_2)$. $\forall t\in[2:T]$, 
\begin{align*}
&p(s_{2:t},a_{1:t},o_t,b_t|o_0,s_1)\\
=~&p(s_{2:t},a_{1:t-1},o_t,b_t|o_0,s_1)\pi_{lo}(a_t|s_t,o_t)\\
=~&\sum_{o_{t-1}}p(s_{2:t},a_{1:t-1},o_t,b_t,o_{t-1}|o_0,s_1)\pi_{lo}(a_t|s_t,o_t)\\
=~&\sum_{o_{t-1}}p(s_{2:t},a_{1:t-1},o_{t-1}|o_0,s_1)\pi_b(b_t|s_t,o_{t-1})\bar\pi_{hi}(o_t|s_t,o_{t-1},b_t)\pi_{lo}(a_t|s_t,o_t).
\end{align*}
Furthermore, 
\begin{align*}
p(s_{2:t},a_{1:t-1},o_{t-1}|o_0,s_1)&=p(s_{2:t-1},a_{1:t-1},o_{t-1}|o_0,s_1)P(s_t|s_{t-1},a_{t-1})\\
&\propto \sum_{b_{t-1}}p(s_{2:t-1},a_{1:t-1},o_{t-1},b_{t-1}|o_0,s_1),
\end{align*}
where $\propto$ replaces a multiplier that does not depend on $o_{t-1}$. Taking expectation with respect to $O_0$ gives the desirable forward recursion result. For the case of $t=1$, the proof is analogous.

\vspace{1em}\noindent
2. (Backward recursion) 

\noindent For any $o_0$, $\forall t\in[1:T-1]$, 
\begin{align*}
\beta^\theta_{t|T}(o_t,b_t)&\propto p(s_{t+1:T},a_{t+1:T}|s_t,a_t,o_t,b_t)\\
&= p(s_{t+2:T},a_{t+1:T}|s_{t+1},o_t)P(s_{t+1}|s_t,a_t)\\
&\propto\sum_{o_{t+1},b_{t+1}}p(s_{t+2:T},a_{t+1:T}|s_{t+1},o_t,o_{t+1},b_{t+1})p(o_{t+1},b_{t+1}|s_{t+1},o_{t}),
\end{align*}
where the multipliers replaced by $\propto$ are independent of $o_t$ and $b_t$. Moreover, 
\begin{align*}
&p(s_{t+2:T},a_{t+1:T}|s_{t+1},o_t,o_{t+1},b_{t+1})\\
=~&p(s_{t+2:T},a_{t+2:T}|s_{t+1},o_t,o_{t+1},b_{t+1},a_{t+1})p(a_{t+1}|s_{t+1},o_t,o_{t+1},b_{t+1})\\
=~&\beta^\theta_{t+1|T}(o_{t+1},b_{t+1})p(a_{t+1}|s_{t+1},o_t,o_{t+1},b_{t+1}). 
\end{align*}
Plugging in the structure of the policy gives the desirable result. 

\vspace{1em}\noindent
3. (Smoothing) 

\noindent Consider any fixed $o_0$. For any $t\in[2:T]$, 
\begin{align*}
p(s_{2:T},a_{1:T},o_t,b_t|o_0,s_1)&=p(s_{2:t},a_{1:t},o_t,b_t|o_0,s_1)p(s_{t+1:T},a_{t+1:T}|s_{1:t},a_{1:t},o_t,b_t,o_0)\\
&=p(s_{2:t},a_{1:t},o_t,b_t|o_0,s_1)p(s_{t+1:T},a_{t+1:T}|s_t,a_t,o_t,b_t).
\end{align*}
Taking expectation with respect to $O_0$ on both sides yields the desirable result. Notice that the second term on the RHS does not depend on $O_0$, therefore is not involved in the expectation. For the case of $t=1$ the proof is analogous. 

\vspace{1em}\noindent
4. (Two-step smoothing) 

\noindent For any $t\in[3:T]$, consider any fixed $o_0$, 
\begin{align*}
&p(s_{2:T},a_{1:T},o_{t-1},b_t|o_0,s_1)\\
=~&\sum_{b_{t-1}}p(s_{2:T},a_{1:T},o_{t-1},b_t,b_{t-1}|o_0,s_1)\\
=~&\sum_{b_{t-1}}p(s_{2:t-1},a_{1:t-1},o_{t-1},b_{t-1}|o_0,s_1)p(s_{t:T},a_{t:T},b_t|s_{1:t-1},a_{1:t-1},o_{t-1},b_{t-1},o_0)\\
=~&\sum_{b_{t-1}}p(s_{2:t-1},a_{1:t-1},o_{t-1},b_{t-1}|o_0,s_1)P(s_t|s_{t-1},a_{t-1})p(s_{t+1:T},a_{t:T},b_t|s_{t},o_{t-1}).
\end{align*}
Take expectation with respect to $O_0$ on both sides. Notice that only the first term on the RHS depends on $o_0$. We have
\begin{align*}
&\tilde \gamma_{\mu,t|T}(o_{t-1},b_t)\\
\propto~&\sum_{b_{t-1}}\alpha_{\mu,t-1}(o_{t-1},b_{t-1})P(s_t|s_{t-1},a_{t-1})p(s_{t+1:T},a_{t:T},b_t|s_{t},o_{t-1})\\
\propto~&\pi_b(b_t|s_t,o_{t-1})p(s_{t+1:T},a_{t:T}|s_{t},b_t,o_{t-1})\sum_{b_{t-1}}\alpha_{\mu,t-1}(o_{t-1},b_{t-1})\\
=~&\pi_b(b_t|s_t,o_{t-1})\left[\sum_{o_t}p(s_{t+1:T},a_{t:T},o_t|s_{t},b_t,o_{t-1})\right]\sum_{b_{t-1}}\alpha_{\mu,t-1}(o_{t-1},b_{t-1})\\
\propto~&\pi_b(b_t|s_t,o_{t-1})\left[\sum_{o_t}\bar\pi_{hi}(o_{t}|s_{t},o_{t-1},b_{t})\pi_{lo}(a_{t}|s_{t},o_{t})\beta_{t|T}(o_t,b_t)\right]\sum_{b_{t-1}}\alpha_{\mu,t-1}(o_{t-1},b_{t-1}),
\end{align*}
where the multipliers replaced by $\propto$ are independent of $o_{t-1}$ and $b_t$. For the case of $t=2$ the proof is analogous. \qed

\subsection{Discussion on the \texorpdfstring{$Q$}{Q}-function}\label{subsection:q}

In our algorithm, as motivated by Section~\ref{section:alg}, we effectively consider the following joint distribution on the graphical model shown in Figure~\ref{figure:model}: the prior distribution of $(O_0,S_1)$ is $\hat\nu$, and the distribution of the rest of the graphical model is determined by an options with failure policy with parameters $\zeta$ and $\theta$. From the EM literature \citep{balakrishnan2017statistical,jain2017non}, the complete likelihood function is
\begin{equation*}
L(s_{1:T},a_{1:T},o_{0:T},b_{1:T};\theta)=
\hat\nu(o_0,s_1)\mathbb P_{\theta,o_0,s_1}(S_{2:T}=s_{2:T},A_{1:T}=a_{1:T},O_{1:T}=o_{1:T},B_{1:T}=b_{1:T}).
\end{equation*}
The marginal likelihood function is
\begin{equation*}
L^m(s_{1:T},a_{1:T};\theta)=
\sum_{o_{0:T},b_{1:T}}\hat\nu(o_0,s_1)\mathbb P_{\theta,o_0,s_1}(S_{2:T}=s_{2:T},A_{1:T}=a_{1:T},O_{1:T}=o_{1:T},B_{1:T}=b_{1:T}),
\end{equation*}
where the superscript $m$ means \emph{marginal}. From the definition of smoothing distributions, we can verify that $L^m(s_{1:T},a_{1:T};\theta)=(z^\theta_{\gamma,\mu})^{-1}$. 

The standard MLE approach maximizes the logarithm of the marginal likelihood function (marginal log-likelihood) with respect to $\theta$. However, such an optimization objective is hard to evaluate for time series models (e.g., HMMs and our graphical model). As an alternative, the marginal log-likelihood can be lower bounded \citep[Chap.~5.4]{jain2017non} as
\begin{equation*}
\log L^m(s_{1:T},a_{1:T};\theta')\geq \sum_{o_{0:T},b_{1:T}}\frac{L(s_{1:T},a_{1:T},o_{0:T},b_{1:T};\theta)}{L^m(s_{1:T},a_{1:T};\theta)}\log L(s_{1:T},a_{1:T},o_{0:T},b_{1:T};\theta'),
\end{equation*}
where $\theta$ on the RHS is arbitrary. The RHS is usually called the (unnormalized) $Q$-function. For our graphical model, it is denoted as $\tilde Q_{\mu,T}(\theta'|\theta)$.
\begin{multline*}
\tilde Q_{\mu,T}(\theta'|\theta)=\sum_{o_{0:T},b_{1:T}}\hat\nu(o_0,s_1)\mathbb P_{\theta,o_0,s_1}(S_{2:T}=s_{2:T},A_{1:T}=a_{1:T},O_{1:T}=o_{1:T},B_{1:T}=b_{1:T})\\
\times z^\theta_{\gamma,\mu}\log [\hat\nu(o_0,s_1)\mathbb P_{\theta',o_0,s_1}(S_{2:T}=s_{2:T},A_{1:T}=a_{1:T},O_{1:T}=o_{1:T},B_{1:T}=b_{1:T})].
\end{multline*}
The RHS is well-defined from the non-degeneracy assumption. From the classical monotonicity property of EM updates \citep[Chap.~5.7]{jain2017non}, maximizing the (unnormalized) $Q$-function $\tilde Q_{\mu,T}(\theta'|\theta)$ with respect to $\theta'$ guarantees non-negative improvement on the marginal log-likelihood. Therefore, improvements on parameter inference can be achieved via iteratively maximizing the (unnormalized) $Q$-function. 

Using the structure of the hierarchical policy, $\tilde Q_{\mu,T}$ can be rewritten as
\begin{multline*}
\tilde Q_{\mu,T}(\theta'|\theta)=\sum_{t=2}^T\sum_{o_{t-1},b_t}\tilde \gamma^{\theta}_{\mu,t|T}(o_{t-1},b_t)[\log \pi_{b}(b_t|s_t,o_{t-1};\theta'_b)]\\
+\sum_{t=1}^T\sum_{o_t,b_t}\gamma^{\theta}_{\mu,t|T}(o_t,b_t)[\log \pi_{lo}(a_t|s_t,o_t;\theta'_{lo})]+\sum_{t=1}^T\sum_{o_t}\gamma^{\theta}_{\mu,t|T}(o_t,b_t=1)[\log \pi_{hi}(o_t|s_t;\theta'_{hi})]\\
+z^\theta_{\gamma,\mu}\sum_{o_0,b_1}\mu(o_0|s_1)\mathbb P_{\theta,o_0,s_1}(S_{2:T}=s_{2:T},A_{1:T}=a_{1:T},B_1=b_1)[\log \pi_{b}(b_1|s_1,o_0;\theta'_b)]+C,
\end{multline*}
where $C$ contains terms unrelated to $\theta'$. Consider the first term on the last line, which partially captures the effect of assuming $\hat\nu$ on the parameter inference. Since this term is upper bounded by $\max_{b_1,s_1,o_0}|\log \pi_{b}(b_1|s_1,o_0;\theta'_b)|$, when $T$ is large enough this term becomes negligible. The precise argument is similar to the proof of Lemma~\ref{lemma:qdiff}. Therefore, after dropping the last line and normalizing, we arrive at our definition of the (normalized) $Q$-function in (\ref{eq:Q}). 

\section{Details of the performance guarantee}\label{section:guaranteedetailed}

\subsection{Smoothing in an extended graphical model}\label{subsection:moredef}

Before providing the proofs, we first introduce a few definitions. Consider the extended graphical model shown in Figure~\ref{figure:extended} with a parameter $k$; $k\in\mathbb N_+$. 

\begin{figure}[ht]
    \centering
    \includegraphics[width=300pt]{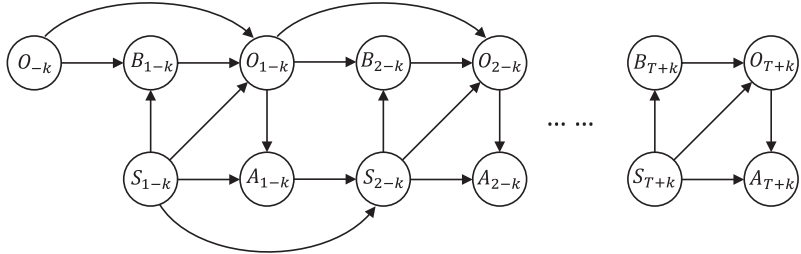}
    \caption{An extended graphical model for hierarchical imitation learning.}\label{figure:extended}
\end{figure}

Let the joint distribution of $(O_{-k},S_{1-k})$ be $\nu^*$. Define the distribution of the rest of the graphical model using an options with failure hierarchical policy with parameters $\zeta$ and $\theta$, analogous to our settings so far. With these two components, the joint distribution on the graphical model is determined. Let $\mathbb P_{\theta,k}$ be such a joint distribution; $\nu^*$ is omitted for conciseness. 

We emphasize the comparison between $\mathbb P_{\theta,k}$ and $\P_{\theta,o_0,s_1}$. The sample space of $\mathbb P_{\theta,k}$ is the domain of $\{S_{1-k:T+k},A_{1-k:T+k},O_{-k:T+k},B_{1-k:T+k}\}$, whereas the sample space of $\P_{\theta,o_0,s_1}$ is the domain of $\{S_{2:T},A_{1:T},O_{1:T},B_{1:T}\}$ since $(O_0,S_1)$ is fixed to $(o_0,s_1)$. 

Consider the infinite length observation sequence $\{s_t,a_t\}_{t\in\mathbb Z}$ corresponding to any $\omega\in\Omega$, where $\Omega$ is defined in (\ref{eq:omega}). Analogous to the non-extended model (Figure~\ref{figure:model}), we can define smoothing distributions for the extended model with any parameter $k$. For all $\theta\in\calTheta$ and $t\in[1:T]$, with any input arguments $o_t$ and $b_t$, the forward message is defined as
\begin{equation*}
\alpha^{\theta}_{k,t}(o_t,b_t)\defeq z_{\alpha,k,t}^{\theta}\mathbb P_{\theta,k}(S_{1-k:t}=s_{1-k:t},A_{1-k:t}=a_{1-k:t},O_t=o_t,B_t=b_t).
\end{equation*}
The backward message is defined as
\begin{multline*}
\beta^{\theta}_{k,t}(o_t,b_t)\defeq\\ z_{\beta,k,t}^{\theta}\mathbb P_{\theta,k}(S_{t+1:T+k}=s_{t+1:T+k},A_{t+1:T+k}=a_{t+1:T+k}|S_t=s_t,A_t=a_t,O_t=o_t,B_t=b_t).
\end{multline*}
The smoothing distribution is defined as
\begin{equation*}
\gamma^{\theta}_{k,t}(o_t,b_t)\defeq z_{\gamma,k}^{\theta}\mathbb P_{\theta,k}(S_{1-k:T+k}=s_{1-k:T+k},A_{1-k:T+k}=a_{1-k:T+k},O_t=o_t,B_t=b_t).
\end{equation*}
The two-step smoothing distribution is defined as
\begin{equation*}
\tilde \gamma^{\theta}_{k,t}(o_{t-1},b_t)\defeq z_{\gamma,k}^{\theta}\mathbb P_{\theta,k}(S_{1-k:T+k}=s_{1-k:T+k},A_{1-k:T+k}=a_{1-k:T+k},O_{t-1}=o_{t-1},B_t=b_t).
\end{equation*}
The quantities $z^\theta_{\alpha,k,t}$, $z^\theta_{\beta,k,t}$ and $z^\theta_{\gamma,k}$ are normalizing constants such that the LHS of the expressions above are probability mass functions. In particular, since $k>0$, we can define $\alpha^{\theta}_{k,t}$ for $t=0$ in the same way as $t\in[1:T]$. The dependency on $T$ in the smoothing distributions is dropped for a cleaner notation. 

Recursive results similar to Theorem~\ref{thm:fb} can be established; the proof is analogous and therefore omitted. As in Theorem~\ref{thm:fb}, we make extensive use of the proportional symbol $\propto$ which stands for, \emph{the LHS equals the RHS multiplied by a normalizing constant}. Moreover, the normalizing constant does not depend on the input arguments of the LHS. 

\begin{corollary}[Forward-backward smoothing for the extended model]\label{corollary:fb}
For all $\theta\in\calTheta$ and $k\in\N_+$, with any input arguments,

\begin{enumerate}[leftmargin=*]
\item (Forward recursion) $\forall t\in[1:T]$, 
\begin{equation}\label{eq:forwarde}
\alpha^\theta_{k,t}(o_t,b_t)\propto\sum_{o_{t-1},b_{t-1}}\pi_b(b_t|s_t,o_{t-1};\theta_b)\bar\pi_{hi}(o_t|s_t,o_{t-1},b_t;\theta_{hi})\pi_{lo}(a_t|s_t,o_t;\theta_{lo}) \alpha^\theta_{k,t-1}(o_{t-1},b_{t-1}).
\end{equation}
\item (Backward recursion) $\forall t\in[1:T-1]$, 
\begin{multline}\label{eq:backwarde}
\beta^\theta_{k,t}(o_t,b_t)\propto\sum_{o_{t+1},b_{t+1}}\pi_b(b_{t+1}|s_{t+1},o_{t};\theta_b)\bar\pi_{hi}(o_{t+1}|s_{t+1},o_{t},b_{t+1};\theta_{hi})\\ \times\pi_{lo}(a_{t+1}|s_{t+1},o_{t+1};\theta_{lo})\beta^\theta_{k,t+1}(o_{t+1},b_{t+1}).
\end{multline}
\item (Smoothing) $\forall t\in[1:T]$, 
\begin{equation}\label{eq:smoothinge}
\gamma^{\theta}_{k,t}(o_t,b_t)\propto\alpha^\theta_{k,t}(o_t,b_t)\beta^\theta_{k,t}(o_t,b_t).
\end{equation}
\item (Two-step smoothing) $\forall t\in[1:T]$, 
\begin{multline}\label{eq:smoothinge2}
\tilde \gamma^{\theta}_{k,t}(o_{t-1},b_t)\propto\pi_b(b_t|s_t,o_{t-1};\theta_{b})\bigg[\sum_{o_t}\bar\pi_{hi}(o_{t}|s_{t},o_{t-1},b_{t};\theta_{hi})\pi_{lo}(a_{t}|s_{t},o_{t};\theta_{lo})\beta^\theta_{k,t}(o_t,b_t)\bigg]\\ \times\bigg[\sum_{b_{t-1}}\alpha^\theta_{k,t-1}(o_{t-1},b_{t-1})\bigg].
\end{multline}
\end{enumerate}
\end{corollary}

The following lemma characterizes the limiting behavior of $\gamma^{\theta}_{k,t}$ and $\tilde \gamma^{\theta}_{k,t}$ as $k\rightarrow\infty$. 

\begin{lemma}[Limits of smoothing distributions]\label{lemma:limitsmoothing} With Assumption~\ref{as:nondeg}, \ref{as:continuity} and \ref{as:initial}, for all $T\geq 2$, $\theta\in\calTheta$, $\omega\in\Omega$ and $t\in[1:T]$, the limits of $\{\gamma^{\theta}_{k,t}\}_{k\in\N_+}$ and $\{\tilde \gamma^{\theta}_{k,t}\}_{k\in\N_+}$ as $k\rightarrow\infty$ exist with respect to the total variation distance. Let $\gamma^{\theta}_{\infty,t}\defeq\lim_{k\rightarrow\infty}\gamma^{\theta}_{k,t}$ and $\tilde\gamma^{\theta}_{\infty,t}\defeq\lim_{k\rightarrow\infty}\tilde\gamma^{\theta}_{k,t}$. They have the following properties: 
\begin{enumerate}[leftmargin=*]
\item $\gamma^{\theta}_{\infty,t}$ and $\tilde\gamma^{\theta}_{\infty,t}$ do not depend on $T$. 
\item $\gamma^{\theta}_{\infty,t}$ and $\tilde\gamma^{\theta}_{\infty,t}$ are entry-wise Lipschitz continuous with respect to $\theta\in\calTheta$. 
\end{enumerate}
\end{lemma}

The proof is given in Appendix~\ref{subsection:limitsmoothing}. The dependency of $\gamma^{\theta}_{\infty,t}$ and $\tilde\gamma^{\theta}_{\infty,t}$ on $\omega$ is omitted for a cleaner notation. 

\subsection{The stochastic convergence of the \texorpdfstring{$Q$}{Q}-function}\label{subsection:asymptoticqfun}

In this subsection, we present the proof of Theorem~\ref{thm:existence}.

First, consider $\gamma^{\theta}_{\infty,t}$ and $\tilde\gamma^{\theta}_{\infty,t}$ defined in Lemma~\ref{lemma:limitsmoothing}. Using the arguments from Section~\ref{section:guarantee}, they can also be analyzed in the \emph{infinitely extended} probability space $(\mathcal{X}^{\mathbb Z}, \mathcal{P}(\mathcal{X}^{\mathbb Z}), \P_{\theta^*,\nu^*})$, where $\mathcal{P}(\cdot)$ denotes the power set. We only define $\gamma^{\theta}_{\infty,t}$ and $\tilde\gamma^{\theta}_{\infty,t}$ for $\omega\in\Omega$; for other sample paths, they are defined arbitrarily. Since $\P_{\theta^*,\nu^*}(\Omega)=1$, such a restriction from $\mathcal{X}^{\mathbb Z}$ to $\Omega$ does not change our probabilistic results. 

For any sample path $\omega$, let $\omega(s_t)$ and $\omega(a_t)$ be the values of $S_t$ and $A_t$ corresponding to $\omega$. With a slight overload of notation, let $\omega(t)=\{\omega(s_t),\omega(a_t),\omega(o_t),\omega(b_t)\}$, which is the set of 
components in $\omega$ corresponding to time $t$.

For all $\theta\in\calTheta$, $\theta'\in\tilde\calTheta$, $\omega\in\Omega$ and $t\in\N_+$, define
\begin{multline*}
f_t(\theta'|\theta;\omega)\defeq\sum_{o_{t-1},b_t}\tilde \gamma^{\theta}_{\infty,t}(o_{t-1},b_t;\omega)\left[\log \pi_{b}(b_t|\omega(s_t),o_{t-1};\theta'_b)\right]+\sum_{o_t,b_t}\gamma^{\theta}_{\infty,t}(o_t,b_t;\omega)\\
\times\left[\log \pi_{lo}(\omega(a_t)|\omega(s_t),o_t;\theta'_{lo})\right]+\sum_{o_t}\gamma^{\theta}_{\infty,t}(o_t,b_t=1;\omega)\left[\log \pi_{hi}(o_t|\omega(s_t);\theta'_{hi})\right],
\end{multline*}
where the dependency of the RHS on $\omega$ is shown explicitly for clarity. $|f_t(\theta'|\theta;\omega)|$ is upper bounded by a constant that does not depend on $\theta$, $\theta'$, $\omega$ and $t$, due to Assumption~\ref{as:nondeg} and \ref{as:continuity}. Moreover, for all $\theta$, $\omega$ and $t$, $f_t(\theta'|\theta;\omega)$ is continuously differentiable with respect to $\theta'\in\tilde\calTheta$; for all $\theta'$, $\omega$ and $t$, $f_t(\theta'|\theta;\omega)$ is Lipschitz continuous with respect to $\theta\in\calTheta$, due to Lemma~\ref{lemma:limitsmoothing}. 

Next, define
\begin{equation}
\bar Q(\theta'|\theta)\defeq\E_{\theta^*,\nu^*}[f_1(\theta'|\theta;\omega)].\label{eq:qbar}
\end{equation}
The subscripts $\theta^*$ and $\nu^*$ in $\E_{\theta^*,\nu^*}$ denote that the expectation is taken with respect to the probability measure $\P_{\theta^*,\nu^*}$. 

With the above definitions, we state the complete version of Theorem~\ref{thm:existence}. The $Q$-function defined in (\ref{eq:Q}) is written as $Q_{\mu,T}(\theta'|\theta;\omega)$, showing its dependency on the sample path. 

\begin{theorem}[The complete version of Theorem~\ref{thm:existence}]\label{thm:stronger} With Assumption~\ref{as:nondeg}, \ref{as:continuity} and \ref{as:initial}, consider $\bar Q(\theta'|\theta)$ defined in (\ref{eq:qbar}), we have
\begin{enumerate}[leftmargin=*]
\item For all $\theta\in\calTheta$, $\bar Q(\theta'|\theta)$ is continuously differentiable with respect to $\theta'\in\tilde\calTheta$, where $\tilde\calTheta$ is defined in Assumption~\ref{as:nondeg}. The gradient is
\begin{equation*}
\nabla\bar Q(\theta'|\theta)=\E_{\theta^*, \nu^*}[\nabla f_1(\theta'|\theta;\omega)].
\end{equation*}
Moreover, as the set of maximizing arguments, $\argmax_{\theta'\in\calTheta}\bar Q(\theta'|\theta)$ is nonempty. 
\item As $T\rightarrow\infty$, 
\begin{equation*}
\sup_{\theta,\theta'\in\calTheta}\sup_{\mu\in\mathcal{M}}\left|Q_{\mu,T}(\theta'|\theta;\omega)-\bar Q(\theta'|\theta)\right|\rightarrow 0,~P_{\theta^*,\nu^*}\text{-a.s.}
\end{equation*}
\end{enumerate}
\end{theorem}

Before proving Theorem~\ref{thm:stronger}, we state the following definition and an auxiliary lemma required for the proof. For all $\theta,\theta'\in\calTheta$, $\omega\in\Omega$ and $T\geq 2$, the sample-path-based population $Q$-function $Q^s_{\infty,T}(\theta'|\theta;\omega)$ is defined as
\begin{equation}
Q^s_{\infty,T}(\theta'|\theta;\omega)\defeq\frac{1}{T}\sum_{t=1}^T f_t(\theta'|\theta;\omega).\label{eq:Qsampleb}
\end{equation}
The superscript \emph{s} in $Q^s_{\infty,T}$ stands for \emph{sample-path-based}. If the sample path $\omega$ is not specified, $Q^s_{\infty,T}(\theta'|\theta)$ is a random variable associated with probability measure $\P_{\theta^*,\nu^*}$. Note that due to stationarity, for any $\theta$, $\theta'$ and $T$, $\bar Q(\theta'|\theta)=\E_{\theta^*,\nu^*}[Q^s_{\infty,T}(\theta'|\theta;\omega)]$.

The difference between $Q^s_{\infty,T}$ and $Q_{\mu,T}$ is bounded in the following lemma.  

\begin{lemma}[Bounding the difference between the $Q$-function and the sample-path-based population $Q$-function]\label{lemma:qdiff} With Assumption~\ref{as:nondeg}, \ref{as:continuity} and \ref{as:initial}, for all $T\geq 2$ and $\omega\in\Omega$, 
\begin{equation*}
\sup_{\theta,\theta'\in\calTheta}\sup_{\mu\in\mathcal{M}}\left|Q^s_{\infty,T}(\theta'|\theta;\omega)-Q_{\mu,T}(\theta'|\theta;\omega)\right|\leq const\cdot T^{-1},
\end{equation*}
where $const$ is a constant independent of $T$ and $\omega$.
\end{lemma}
The proof is provided in Appendix~\ref{section:qdiffproof}. Now we are ready to present the proof of Theorem~\ref{thm:stronger} step-by-step. The structure of this proof is similar to the standard analysis of HMM maximum likelihood estimators \citep[Chap. 12]{cappe2006inference}. 

\begin{proof}[Proof of Theorem~\ref{thm:stronger}] We prove the two parts of the theorem separately.

\vspace{1em}\noindent
1. For all $\theta'\in\tilde\calTheta$, there exists $\delta_{\theta'}>0$ such that the set $\{\tilde\theta;\norms{\tilde\theta-\theta'}_2\leq\delta_{\theta'}\}\subseteq\tilde\calTheta$. For all $\theta\in\calTheta$ and $\omega\in\Omega$, due to the differentiability of $f_1(\theta'|\theta;\omega)$ with respect to $\theta'$, there exists a gradient $\nabla f_1(\theta'|\theta;\omega)$ at any $\theta'\in\tilde\calTheta$ such that
\begin{equation*}
\lim_{\delta\rightarrow 0}\sup_{\tilde\theta\in\tilde\calTheta;\norms{\tilde\theta-\theta'}_2\leq\delta}\frac{|f_1(\tilde\theta|\theta;\omega)-f_1(\theta'|\theta;\omega)-\langle\nabla f_1(\theta'|\theta;\omega),\tilde\theta-\theta'\rangle|}{\norms{\tilde\theta-\theta'}_2}=0.
\end{equation*}

We need to transform the above almost surely (in $\omega$) convergence to the convergence of expectation, using the dominated convergence theorem. As a requirement, the quantity inside the limit on the LHS needs to be upper-bounded. For all $\theta\in\calTheta$, $\theta'\in\tilde\calTheta$, $\omega\in\Omega$ and $0<\delta\leq\delta_{\theta'}$, 
\begin{multline}
\sup_{\tilde\theta\in\tilde\calTheta;\norms{\tilde\theta-\theta'}_2\leq\delta}\frac{|f_1(\tilde\theta|\theta;\omega)-f_1(\theta'|\theta;\omega)-\langle\nabla f_1(\theta'|\theta;\omega),\tilde\theta-\theta'\rangle|}{\norms{\tilde\theta-\theta'}_2}\leq\\
\sup_{\tilde\theta;\norms{\tilde\theta-\theta'}_2\leq\delta_{\theta'}}\frac{|f_1(\tilde\theta|\theta;\omega)-f_1(\theta'|\theta;\omega)|}{\norms{\tilde\theta-\theta'}_2}+\sup_{\tilde\theta;\norms{\tilde\theta-\theta'}_2\leq\delta_{\theta'}}\frac{|\langle\nabla f_1(\theta'|\theta;\omega),\tilde\theta-\theta'\rangle|}{\norms{\tilde\theta-\theta'}_2}.\label{eq:upper}
\end{multline}

Since continuously differentiable functions are Lipschitz continuous on convex and compact subsets, $\pi_{hi}$, $\pi_{lo}$ and $\pi_b$ as functions of $\tilde\theta\in\tilde\calTheta$ are Lipschitz continuous on $\{\tilde\theta;\norms{\tilde\theta-\theta'}_2\leq\delta_{\theta'}\}$, with any other input arguments. From the expression of $f_1$, we can verify that for any fixed $\theta$ and $\omega$, $f_1(\tilde\theta|\theta;\omega)$ as a function of $\tilde\theta$ is Lipschitz continuous on $\{\tilde\theta;\norms{\tilde\theta-\theta'}_2\leq\delta_{\theta'}\}$, and the Lipschitz constant only depends on $\theta'$ and $\delta_{\theta'}$. Consequently, the RHS of (\ref{eq:upper}) can be upper-bounded for all $\omega\in\Omega$. Applying the dominated convergence theorem, we have
\begin{equation}
\lim_{\delta\rightarrow 0}\E_{\theta^*,\nu^*}\left[\sup_{\tilde\theta\in\tilde\calTheta;\norms{\tilde\theta-\theta'}_2\leq\delta}\frac{|f_1(\tilde\theta|\theta;\omega)-f_1(\theta'|\theta;\omega)-\langle\nabla f_1(\theta'|\theta;\omega),\tilde\theta-\theta'\rangle|}{\norms{\tilde\theta-\theta'}_2}\right]=0.\label{eq:importantlim}
\end{equation}

On the other hand, notice that for all $\theta\in\calTheta$, $\theta'\in\tilde\calTheta$ and $\delta>0$, 
\begin{align*}
&\sup_{\tilde\theta\in\tilde\calTheta;\norms{\tilde\theta-\theta'}_2\leq\delta}\frac{|\bar Q(\tilde\theta|\theta)-\bar Q(\theta'|\theta)-\langle \E_{\theta^*,\nu^*}[\nabla f_1(\theta'|\theta;\omega)],\tilde\theta-\theta'\rangle|}{\norms{\tilde\theta-\theta'}_2}\\
=~&\sup_{\tilde\theta\in\tilde\calTheta;\norms{\tilde\theta-\theta'}_2\leq\delta}\frac{|\E_{\theta^*,\nu^*}[f_1(\tilde\theta|\theta;\omega)-f_1(\theta'|\theta;\omega)-\langle\nabla f_1(\theta'|\theta;\omega),\tilde\theta-\theta'\rangle]|}{\norms{\tilde\theta-\theta'}_2}\\
\leq~&\E_{\theta^*,\nu^*}\left[\sup_{\tilde\theta\in\tilde\calTheta;\norms{\tilde\theta-\theta'}_2\leq\delta}\frac{|f_1(\tilde\theta|\theta;\omega)-f_1(\theta'|\theta;\omega)-\langle\nabla f_1(\theta'|\theta;\omega),\tilde\theta-\theta'\rangle|}{\norms{\tilde\theta-\theta'}_2}\right].
\end{align*}
Combining with (\ref{eq:importantlim}) proves the differentiability of $\bar Q(\theta'|\theta)$ with respect to $\theta'\in\tilde\calTheta$ for any fixed $\theta$. The gradient is
\begin{equation*}
\nabla\bar Q(\theta'|\theta)=\E_{\theta^*,\nu^*}[\nabla f_1(\theta'|\theta;\omega)].
\end{equation*}

Analogously, using the dominated convergence theorem we can also show that the gradient $\nabla\bar Q(\theta'|\theta)$ is continuous with respect to $\theta'\in\tilde\calTheta$. Details are omitted due to the similarity with the above procedure. It is worth noting that we let $\theta'\in\tilde\calTheta$ instead of $\calTheta$. In this way, the gradient $\nabla\bar Q(\theta'|\theta)$ can be naturally defined when $\theta'$ is not an interior point of $\calTheta$. 

From differentiability and $\calTheta\subseteq\tilde\calTheta$, $\bar Q(\theta'|\theta)$ is also continuous with respect to $\theta'\in\calTheta$. Since $\calTheta$ is compact, the set of maximizing arguments $\argmax_{\theta'\in\calTheta}\bar Q(\theta'|\theta)$ is nonempty. 

\vspace{1em}\noindent
2. We need to prove the uniform (in $\theta,\theta'\in\calTheta$ and $\mu\in\mathcal{M}$) almost sure convergence of the $Q$-function $Q_{\mu,T}(\theta'|\theta;\omega)$ to the population $Q$-function $\bar Q(\theta'|\theta)$. The proof is separated into three steps. First, we show the almost sure convergence of $Q^s_{\infty,T}(\theta'|\theta;\omega)$ to $\bar Q(\theta'|\theta)$ for all $\theta,\theta'\in\calTheta$ using the ergodic theorem. Second, we extend this pointwise convergence to uniform (in $\theta,\theta'$) convergence using a version of the Arzel\`{a}-Ascoli theorem \citep[Chap.~21]{davidson1994stochastic}. Finally, from Lemma~\ref{lemma:qdiff}, the difference between $Q_{\mu,T}(\theta'|\theta;\omega)$ and $Q^s_{\infty,T}(\theta'|\theta;\omega)$ vanishes uniformly in $\mu$ as $T\rightarrow\infty$. 

Concretely, for the pointwise (in $\theta,\theta'$) almost sure convergence of $Q^s_{\infty,T}(\theta'|\theta;\omega)$ as $T\rightarrow\infty$, we apply Birkhoff's ergodic theorem. Let $\mathcal T:\mathcal{X}^\Z\rightarrow\mathcal{X}^\Z$ be the standard shift operator. That is, for any $t\in\mathbb Z$, $\mathcal T\omega(t)=\omega(t+1)$.
Due to stationarity, $\mathcal T$ is a measure-preserving map, i.e., $\P_{\theta^*,\nu^*}(\mathcal T^{-1}F)=\P_{\theta^*,\nu^*}(F)$ for all $F\in \mathcal{P}(\mathcal{X}^\Z)$. Therefore, the quadruple $\{\mathcal{X}^\Z, \mathcal{P}(\mathcal{X}^\Z), \P_{\theta^*,\nu^*},\mathcal T\}$ defines a dynamical system. 

Here, we need some clarification on some concepts and notations. Consider the Markov chain $\{X_t\}_{t=1}^\infty=\{S_t,A_t,O_t,B_t\}_{t=1}^\infty$ induced by the expert policy, let $\Pi_{X,\theta^*}$ be its set of all stationary distributions. Comparing $\Pi_{X,\theta^*}$ to $\Pi_{\theta^*}$ from Assumption~\ref{as:initial}, they both depend on the true parameter $\theta^*$; the former corresponds to the chain $\{S_t,A_t,O_t,B_t\}_{t=1}^\infty$, while the latter corresponds to the chain $\{O_{t-1},S_t\}_{t=1}^\infty$. From the structure of our graphical model, they are equivalent by some transformation. 

From Section~\ref{section:guarantee}, $\P_{\theta^*,\nu^*}$ is defined from an element of $\Pi_{X,\theta^*}$ that depends on $\nu^*$. Denote this stationary distribution as $\psi$. Since $\nu^*$ is an extreme point of $\Pi_{\theta^*}$ (Assumption~\ref{as:initial}), $\psi$ is also an extreme point of $\Pi_{X,\theta^*}$. Then, we can apply a standard Markov chain ergodicity result. From \citep[Theorem~5.7]{hairer2006ergodic}, the dynamical system $\{\mathcal{X}^\Z, \mathcal{P}(\mathcal{X}^\Z), \P_{\theta^*,\nu^*},\mathcal T\}$ is ergodic. For our case, Birkhoff's ergodic theorem is restated as follows. 
\begin{lemma}[\citep{hairer2006ergodic}, Corollary 5.3 restated]\label{lemma:ergodic} If a dynamical system $\{\mathcal{X}^\Z, \mathcal{P}(\mathcal{X}^\Z), \P_{\theta^*,\nu^*},\mathcal T\}$ is ergodic and $f:\mathcal{X}^\Z\rightarrow \R$ satisfies $\E_{\theta^*,\nu^*}[f(\omega)]<\infty$, then as $T\rightarrow\infty$,
\begin{equation*}
\frac{1}{T}\sum_{t=0}^{T-1}f(\mathcal T^t\omega)\rightarrow\E_{\theta^*,\nu^*}[f(\omega)],~P_{\theta^*,\nu^*}\text{-a.s.}
\end{equation*}
\end{lemma}

For our purpose, observe that for any $\theta,\theta'\in\calTheta$, $f_t(\theta'|\theta;\omega)=f_1(\theta'|\theta;\mathcal T^{t-1}\omega)$. Therefore, applying the ergodic theorem to $Q^s_{\infty,T}(\theta'|\theta)$, as $T\rightarrow\infty$,  
\begin{equation}
Q^s_{\infty,T}(\theta'|\theta;\omega)\rightarrow \bar Q(\theta'|\theta),~P_{\theta^*,\nu^*}\text{-a.s.}\label{eq:pointwise}
\end{equation}

To extend the pointwise convergence in (\ref{eq:pointwise}) to uniform (in $\theta,\theta'$) convergence, the following concept is required. The sequence $\{Q^s_{\infty,T}(\theta'|\theta)\}$ indexed by $T$ as functions of $\theta$ and $\theta'$ is \emph{strongly stochastically equicontinuous} \citep[Equation 21.43]{davidson1994stochastic} if for any $\eps>0$ there exists $\delta>0$ such that
\begin{equation}
\limsup_{T\rightarrow\infty}\sup_{\theta_1,\theta'_1,\theta_2,\theta'_2\in\calTheta;\norms{\theta_1-\theta_2}_2+\norms{\theta'_1-\theta'_2}_2\leq\delta}\left|Q^s_{\infty,T}(\theta'_1|\theta_1;\omega)-Q^s_{\infty,T}(\theta'_2|\theta_2;\omega)\right|<\eps,~P_{\theta^*,\nu^*}\text{-a.s.}\label{eq:sse}
\end{equation}
Indeed this property holds for $\{Q^s_{\infty,T}(\theta'|\theta)\}$, as shown in Appendix~\ref{subsection:sseproof}. The version of the Arzel\`{a}-Ascoli theorem we use is restated as follows, tailored to our need. 

\begin{lemma}[\citep{davidson1994stochastic}, Theorem 21.8 restated]\label{lemma:uniform}
Given (\ref{eq:pointwise}) and (\ref{eq:sse}), as $T\rightarrow \infty$ we have
\begin{equation*}
\sup_{\theta,\theta'\in\calTheta}\left|Q^s_{\infty,T}(\theta'|\theta;\omega)-\bar Q(\theta'|\theta)\right|\rightarrow 0,~P_{\theta^*,\nu^*}\text{-a.s.}
\end{equation*}
\end{lemma}

Combining Lemma~\ref{lemma:qdiff} and Lemma~\ref{lemma:uniform} concludes the proof of the second part. 
\end{proof}

\paragraph{On the concavity of $\bar Q(\cdot|\theta)$.} As discussed after introducing Assumption~\ref{as:concavity}, we expect the following to hold in certain cases of tabular parameterization: for all $\theta\in\calTheta$, the function $\bar Q(\cdot|\theta)$ is strongly concave over $\calTheta$. Details are presented below. 

Consider $\theta'_b$ for example, we need to provide sufficient conditions such that the following function is strongly concave with respect to $\theta'_b\in\calTheta_b$, given any $\theta\in\calTheta$. 
\begin{equation*}
\bar Q_b(\theta'_b|\theta)=\sum_{o_{0},b_1}\E_{\theta^*,\nu^*}\left[\tilde \gamma^{\theta}_{\infty,t}(o_0,b_1;\omega)\log \pi_{b}(b_1|\omega(s_1),o_0;\theta'_b)\right].
\end{equation*}

Let the marginal distribution of $\nu^*$ on $S_1$ be $\nu^*_{S_1}$. If $\nu^*_{S_1}$ is strictly positive on $\cals$, then we rewrite $\bar Q_b(\theta'_b|\theta)$ as
\begin{equation*}
\bar Q_b(\theta'_b|\theta)=\sum_{o_{0},b_1}\sum_{s_1\in\cals}\nu^*_{S_1}(s_1)\E_{\theta^*,\nu^*|S_1=s_1}\left[\tilde \gamma^{\theta}_{\infty,t}(o_0,b_1;\omega)\right]\log \pi_{b}(b_1|s_1,o_0;\theta'_b).
\end{equation*}
In the case of tabular parameterization, $\pi_{b}(b_1|s_1,o_0;\theta'_b)$ is an entry of $\theta'_b$ indexed as $\theta'_b(b_1,s_1,o_0)$; its logarithm is 1-strongly concave on the interval $[0,1]$. $\bar Q_b(\theta'_b|\theta)$ is strongly concave with respect to $\theta'_b$ if $\E_{\theta^*,\nu^*|S_1=s_1}[\tilde \gamma^{\theta}_{\infty,t}(o_0,b_1;\omega)]$ is strictly positive for all $o_0$ and $b_1$. We speculate that this requirement is mild, but a rigorous characterization is quite challenging. 

\subsection{The convergence of the population version algorithm}\label{subsection:proofpopulation}

We first present the complete version of Theorem~\ref{thm:population}, where an upper bound on $\gamma$ is also shown. Notice that we assume all the assumptions, including Assumption~\ref{as:concavity} and \ref{as:local}.

\begin{theorem}[The complete version of Theorem~\ref{thm:population}]\label{thm:populationcomplete} With all the assumptions, 

\begin{enumerate}[leftmargin=*]
\item (First-order stability) There exists $0<\gamma\leq \bar\gamma$ such that for all $\theta\in\calTheta_r$, 
\begin{equation*}
\norm{\nabla\bar Q(\bar M(\theta)|\theta)-\nabla\bar Q(\bar M(\theta)|\theta^*)}_2\leq \gamma\norm{\theta-\theta^*}_2.
\end{equation*}
Specifically, the upper bound $\bar\gamma$ is given by
\begin{multline*}
\bar\gamma=\frac{4|\calo|L_{\theta^*,r}}{\eps^2_b\zeta}\left(\sup_{\theta'\in\calTheta_r}z_{\theta',\theta^*}\right)\bigg(2\max_{o_0,s_1,b_1}\sup_{\theta'_b\in\calTheta_b}\norm{\nabla\log \pi_{b}(b_1|s_1,o_{0};\theta'_b)}_2\\
+\max_{s_1,a_1,o_1}\sup_{\theta'_{lo}\in\calTheta_{lo}}\norm{\nabla\log \pi_{lo}(a_1|s_1,o_1;\theta'_{lo})}_2+\max_{s_1,o_1}\sup_{\theta'_{hi}\in\calTheta_{hi}}\norm{\nabla\log \pi_{hi}(o_1|s_1;\theta'_{hi})}_2\bigg).
\end{multline*}
$\zeta$ is the failure parameter in the options with failure framework; $\eps_b$ is a mixing constant defined in Lemma~\ref{lemma:mixing}; $L_{\theta^*,r}$ is a Lipschitz constant defined in Lemma~\ref{lemma:lipschitz}; $z_{\theta',\theta^*}$ is defined in Lemma~\ref{lemma:fstability}. 
\item (Contraction) Let $\kappa=\gamma/\lambda$. For all $\theta\in\calTheta_r$, 
\begin{equation*}
\norm{\bar M(\theta)-\theta^*}_2\leq \kappa\norm{\theta-\theta^*}_2. 
\end{equation*}
If $\kappa<1$, the population version algorithm converges linearly to the true parameter $\theta^*$. 
\end{enumerate}
\end{theorem}

\begin{proof}[Proof of Theorem~\ref{thm:populationcomplete}] We prove the two parts separately in the following. 

\vspace{1em}\noindent
1. For convenience of notation, let $\nabla\bar Q(\theta'|\theta)=[\nabla_b\bar Q(\theta'|\theta),\nabla_{lo}\bar Q(\theta'|\theta),\nabla_{hi}\bar Q(\theta'|\theta)]$ such that, for example, $\nabla_b\bar Q(\theta'|\theta)$ is the gradient of $\bar Q(\theta'|\theta)$ with respect to $\theta'_b$. Using the expressions of $\nabla\bar Q(\theta'|\theta)$ from Theorem~\ref{thm:stronger}, we have
\begin{multline*}
\norm{\nabla\bar Q(\bar M(\theta)|\theta)-\nabla\bar Q(\bar M(\theta)|\theta^*)}_2
\leq\norm{\nabla_b\bar Q(\bar M(\theta)|\theta)-\nabla_b\bar Q(\bar M(\theta)|\theta^*)}_2\\
+\norm{\nabla_{lo}\bar Q(\bar M(\theta)|\theta)-\nabla_{lo}\bar Q(\bar M(\theta)|\theta^*)}_2+\norm{\nabla_{hi}\bar Q(\bar M(\theta)|\theta)-\nabla_{hi}\bar Q(\bar M(\theta)|\theta^*)}_2.
\end{multline*}

Consider the first term, 
\begin{align*}
&\norm{\nabla_b\bar Q(\bar M(\theta)|\theta)-\nabla_b\bar Q(\bar M(\theta)|\theta^*)}_2\\
=~&\norm{\E_{\theta^*,\nu^*}\bigg\{\sum_{o_{0},b_1}\left[\tilde \gamma^{\theta}_{\infty,1}(o_{0},b_1;\omega)-\tilde \gamma^{\theta^*}_{\infty,1}(o_{0},b_1;\omega)\right]\left[\nabla\log \pi_{b}(b_1|\omega(s_1),o_{0};\bar M(\theta)_b)\right]\bigg\}}_2\\
\leq~&\sum_{o_{0},b_1}\norm{\E_{\theta^*,\nu^*}\bigg\{\left[\tilde \gamma^{\theta}_{\infty,1}(o_{0},b_1;\omega)-\tilde \gamma^{\theta^*}_{\infty,1}(o_{0},b_1;\omega)\right]\left[\nabla\log \pi_{b}(b_1|\omega(s_1),o_{0};\bar M(\theta)_b)\right]\bigg\}}_2\\
\leq~&\sum_{o_{0},b_1}\E_{\theta^*,\nu^*}\bigg\{\left|\tilde \gamma^{\theta}_{\infty,1}(o_{0},b_1;\omega)-\tilde \gamma^{\theta^*}_{\infty,1}(o_{0},b_1;\omega)\right|\norm{\nabla\log \pi_{b}(b_1|\omega(s_1),o_{0};\bar M(\theta)_b)}_2\bigg\}\\
\leq~&\max_{o_0,s_1,b_1}\sup_{\theta'_b\in\calTheta_b}\norm{\nabla\log \pi_{b}(b_1|s_1,o_{0};\theta'_b)}_2\E_{\theta^*,\nu^*}\bigg\{\sum_{o_{0},b_1}\left|\tilde \gamma^{\theta}_{\infty,1}(o_{0},b_1;\omega)-\tilde \gamma^{\theta^*}_{\infty,1}(o_{0},b_1;\omega)\right|\bigg\}\\
\leq~&2\max_{o_0,s_1,b_1}\sup_{\theta'_b\in\calTheta_b}\norm{\nabla\log \pi_{b}(b_1|s_1,o_{0};\theta'_b)}_2\times\sup_{\omega\in\Omega}\norm{\tilde \gamma^{\theta}_{\infty,1}(\omega)-\tilde \gamma^{\theta^*}_{\infty,1}(\omega)}_\tv\\
\leq~&\frac{8|\calo|L_{\theta^*,r}}{\eps^2_b\zeta}\left(\sup_{\theta'\in\calTheta_r}z_{\theta',\theta^*}\right)\left(\max_{o_0,s_1,b_1}\sup_{\theta'_b\in\calTheta_b}\norm{\nabla\log \pi_{b}(b_1|s_1,o_{0};\theta'_b)}_2\right)\norm{\theta-\theta^*}_2.
\end{align*}
We use the triangle inequality and the Jensen's inequality in the third and the fourth line respectively. The fifth line is finite due to $\theta_b$ being compact and the continuity of the gradient (Assumption~\ref{as:continuity}). The last line is due to the limit form of Lemma~\ref{lemma:smoothingdiff2}, similar to the argument in Appendix~\ref{subsection:limitsmoothing}. Notice that the coefficient of $\norms{\theta-\theta^*}_2$ on the last line does not depend on $\theta$. 

Analogously, we have
\begin{multline*}
\norm{\nabla_{lo}\bar Q(\bar M(\theta)|\theta)-\nabla_{lo}\bar Q(\bar M(\theta)|\theta^*)}_2\leq\\ \frac{4|\calo|L_{\theta^*,r}}{\eps^2_b\zeta}\left(\sup_{\theta'\in\calTheta_r}z_{\theta',\theta^*}\right)\left(\max_{s_1,a_1,o_1}\sup_{\theta'_{lo}\in\calTheta_{lo}}\norm{\nabla\log \pi_{lo}(a_1|s_1,o_1;\theta'_{lo})}_2\right)\norm{\theta-\theta^*}_2,
\end{multline*}
\begin{multline*}
\norm{\nabla_{hi}\bar Q(\bar M(\theta)|\theta)-\nabla_{hi}\bar Q(\bar M(\theta)|\theta^*)}_2\leq\\ \frac{4|\calo|L_{\theta^*,r}}{\eps^2_b\zeta}\left(\sup_{\theta'\in\calTheta_r}z_{\theta',\theta^*}\right)\left(\max_{s_1,o_1}\sup_{\theta'_{hi}\in\calTheta_{hi}}\norm{\nabla\log \pi_{hi}(o_1|s_1;\theta'_{hi})}_2\right)\norm{\theta-\theta^*}_2.
\end{multline*}
Combining everything, we have the upper bound on $\gamma$. 

\vspace{1em}\noindent
2. The proof of the second part mirrors the proof of \citep[Theorem~4]{balakrishnan2017statistical}. The main difference is the construction of the following self-consistency (\emph{a.k.a.} fixed-point) condition. 
\begin{lemma}[Self-consistency]\label{lemma:selfconsistency} With all the assumptions, $\theta^*=\bar M(\theta^*)$. 
\end{lemma}

The proof of this lemma is presented in Appendix~\ref{subsection:selfconsistency}. Such a condition is used without proof in \citep{balakrishnan2017statistical} since it only considers i.i.d. samples, and the self-consistency condition for EM with i.i.d. samples is a well-known result. However, for the case of dependent samples like our graphical model, such a condition results from the stochastic convergence of the $Q$-function which is not immediate. 

For the rest of the proof, we present a brief sketch here for completeness. Due to concavity, we have the first order optimality conditions: for all $\theta,\theta'\in\calTheta$, $\langle\nabla \bar Q(\bar M(\theta^*)|\theta^*),\theta-\bar M(\theta^*)\rangle\leq 0$ and $\langle\nabla \bar Q(\bar M(\theta)|\theta),\theta'-\bar M(\theta)\rangle\leq 0$. Using $\theta^*=\bar M(\theta^*)$, we can combine the two optimality conditions together and obtain the following. For all $\theta\in\calTheta$, 
\begin{equation*}
\langle\nabla \bar Q(\bar M(\theta)|\theta^*)-\nabla \bar Q(\theta^*|\theta^*),\theta^*-\bar M(\theta)\rangle\leq \langle\nabla \bar Q(\bar M(\theta)|\theta^*)-\nabla \bar Q(\bar M(\theta)|\theta),\theta^*-\bar M(\theta)\rangle.
\end{equation*}
From Assumption~\ref{as:concavity}, $\lhs\geq\lambda\norms{\theta^*-\bar M(\theta)}_2^2$. From Cauchy-Schwarz and the first part of this theorem, $\rhs\leq \gamma\norms{\theta^*-\bar M(\theta)}_2\norms{\theta-\theta^*}_2$. Canceling $\norms{\theta^*-\bar M(\theta)}_2$ on both sides completes the proof. 
\end{proof}

\subsection{Proof of Theorem~\ref{thm:perturbed}}\label{subsection:proofperturbed}

1. We first show the strong consistency of $M_{\mu,T}(\theta;\omega)$, the parameter update of Algorithm~\ref{algorithm}, as an estimator of $\bar M(\theta)$. This follows from standard techniques in the analysis of M-estimators. In particular, consider the set of sample paths $\omega$ such that $\omega\in\Omega$ and $\argmax_{\theta'\in\calTheta} Q_{\mu,T}(\theta'|\theta;\omega)$ has a unique element $M_{\mu,T}(\theta;\omega)$. Such a set of sample paths has probability measure 1. 

For all $\theta\in\calTheta$, $T\geq 2$ and $\mu\in\mathcal{M}$, with one of the above sample path $\omega$, 
\begin{align*}
0&\leq \bar Q(\bar M(\theta)|\theta)-\bar Q(M_{\mu,T}(\theta;\omega)|\theta)\\
&\leq \bar Q(\bar M(\theta)|\theta)-Q_{\mu,T}(\bar M(\theta)|\theta;\omega)+Q_{\mu,T}(\bar M(\theta)|\theta;\omega)-Q_{\mu,T}(M_{\mu,T}(\theta;\omega)|\theta;\omega)\\
&\hspace{20em}+Q_{\mu,T}(M_{\mu,T}(\theta;\omega)|\theta;\omega)-\bar Q(M_T(\theta;\omega)|\theta)\\
&\leq 2\sup_{\theta'\in\calTheta}\left|\bar Q(\theta'|\theta)-Q_{\mu,T}(\theta'|\theta;\omega)\right|.
\end{align*}

From Theorem~\ref{thm:stronger}, $\P_{\theta^*,\nu^*}$-almost surely, $\sup_{\theta,\theta'\in\calTheta}\sup_{\mu\in\mathcal{M}}|\bar Q(\theta'|\theta)-Q_{\mu,T}(\theta'|\theta;\omega)|\rightarrow 0$ as $T\rightarrow\infty$. Therefore,
\begin{equation*}
\sup_{\theta\in\calTheta_r}\sup_{\mu\in\mathcal{M}}\left[\bar Q(\bar M(\theta)|\theta)-\bar Q(M_{\mu,T}(\theta;\omega)|\theta)\right]\rightarrow 0,~P_{\theta^*,\nu^*}\text{-a.s.}
\end{equation*}
An equivalent argument is the following. $\P_{\theta^*,\nu^*}$-almost surely, for any $\delta>0$ there exists $T_\omega\in\N_+$ such that for all $T\geq T_\omega$, $\sup_{\theta\in\calTheta_r}\sup_{\mu\in\mathcal{M}}[\bar Q(\bar M(\theta)|\theta)-\bar Q(M_{\mu,T}(\theta;\omega)|\theta)]\leq \delta$. In particular, for any $\eps>0$, let
\begin{equation*}
\delta=\frac{1}{2}\inf_{\theta\in\calTheta_r}\bigg[\bar Q(\bar M(\theta)|\theta)-\sup_{\theta'\in\calTheta;\norms{\theta'-\bar M(\theta)}_2\geq\eps}\bar Q(\theta'|\theta)\bigg].
\end{equation*}

From the identifiability assumption (Assumption~\ref{as:local}), the RHS is positive. Therefore, such an assignment of $\delta$ is valid. Consequently, for all $T\geq T_\omega$, $\theta\in\calTheta_r$ and $\mu\in\mathcal{M}$,
\begin{equation*}
\bar Q(\bar M(\theta)|\theta)-\bar Q(M_{\mu,T}(\theta;\omega)|\theta)<\bar Q(\bar M(\theta)|\theta)-\sup_{\theta'\in\calTheta;\norms{\theta'-\bar M(\theta)}_2\geq\eps}\bar Q(\theta'|\theta), 
\end{equation*}
which means that $\norms{M_{\mu,T}(\theta;\omega)-\bar M(\theta)}_2<\eps$. Taking supremum over $\theta\in\calTheta_r$ and $\mu\in\mathcal{M}$, we summarize the argument as the following. $\P_{\theta^*,\nu^*}$-almost surely, for any $\eps>0$ there exists $T_\omega\in\N_+$ such that for all $T\geq T_\omega$, 
\begin{equation*}
\sup_{\theta\in\calTheta_r}\sup_{\mu\in\mathcal{M}}\norm{M_{\mu,T}(\theta;\omega)-\bar M(\theta)}_2<\eps.
\end{equation*}
Such a result is equivalent to the uniform (in $\theta$ and $\mu$) strong consistency of $M_{\mu,T}(\theta;\omega)$ as an estimator of $\bar M(\theta)$. As $T\rightarrow\infty$,
\begin{equation*}
\sup_{\theta\in\calTheta_r}\sup_{\mu\in\mathcal{M}}\norm{M_{\mu,T}(\theta;\omega)-\bar M(\theta)}_2\rightarrow 0,~P_{\theta^*,\nu^*}\text{-a.s.}
\end{equation*}

This result is insufficient for Part 1, since $T_\omega$ is sample path dependent. To get rid of this sample path dependency, we use the dominated convergence theorem. Notice that $\P_{\theta^*,\nu^*}$-almost surely, for all $T\geq 2$, $\sup_{\theta\in\calTheta_r}\sup_{\mu\in\mathcal{M}}\norms{M_{\mu,T}(\theta;\omega)-\bar M(\theta)}_2$ is bounded due to the compactness of $\calTheta$. Therefore we have
\begin{equation*}
\lim_{T\rightarrow\infty}\E_{\theta^*,\nu^*}\left[\sup_{\theta\in\calTheta_r}\sup_{\mu\in\mathcal{M}}\norm{M_{\mu,T}(\theta;\omega)-\bar M(\theta)}_2\right]=0. 
\end{equation*}

For any $q>0$, there exists $\underline T(q)\in\N_+$ such that for all $T\geq \underline T(q)$,
\begin{equation*}
\E_{\theta^*,\nu^*}\left[\sup_{\theta\in\calTheta_r}\sup_{\mu\in\mathcal{M}}\norm{M_{\mu,T}(\theta;\omega)-\bar M(\theta)}_2\right]\leq q. 
\end{equation*}
Applying Markov's inequality, for any $\Delta>0$, 
\begin{equation*}
\P_{\theta^*,\nu^*}\left(\sup_{\theta\in\calTheta_r}\sup_{\mu\in\mathcal{M}}\norm{M_{\mu,T}(\theta;\omega)-\bar M(\theta)}_2\geq\Delta\right)\leq\frac{1}{\Delta}\E_{\theta^*,\nu^*}\left[\sup_{\theta\in\calTheta_r}\sup_{\mu\in\mathcal{M}}\norm{M_{\mu,T}(\theta;\omega)-\bar M(\theta)}_2\right]\leq \frac{q}{\Delta}. 
\end{equation*}
Scaling $q$ yields the desirable result. 

\vspace{1em}\noindent
2. The proof of Part 2 is the same as \citep[Theorem~5]{balakrishnan2017statistical}. We present a sketch for completeness. For all $T\geq \underline T(\Delta,q)$, condition the following proof on the high probability event that $\sup_{\theta\in\calTheta_r}\sup_{\mu\in\mathcal{M}}\norm{M_{\mu,T}(\theta;\omega)-\bar M(\theta)}_2\leq\Delta$. 

Assume $\norms{\theta^{(n-1)}-\theta^*}_2\leq r$, which holds for $n=1$. Then, using the triangle inequality, the result from Theorem~\ref{thm:population}, the above concentration and $\Delta\leq (1-\kappa)r$, we have the following for any $\mu$. 
\begin{align}
\norm{\theta^{(n)}-\theta^*}_2&\leq \norm{\bar M(\theta^{(n-1)})-\theta^*}_2+\norm{M_{\mu,T}(\theta^{(n-1)})-\bar M(\theta^{(n-1)})}_2\nonumber\\
&\leq \kappa\norms{\theta^{(n-1)}-\theta^*}_2+\Delta,\label{eq:onestep} 
\end{align}
and $\norms{\theta^{(n)}-\theta^*}_2\leq \kappa r+(1-\kappa)r=r$. From induction, the one step relation (\ref{eq:onestep}) holds for all $n\in\N_+$. Unrolling (\ref{eq:onestep}) and regrouping the terms completes the proof. \qed

\section{Proofs of auxiliary lemmas}\label{section:technical}

This section presents proofs omitted in earlier sections. Assumptions~\ref{as:nondeg}, \ref{as:continuity} and \ref{as:initial} are assumed. 

In particular, the first three subsections develop a few essential lemmas required for the proofs in later subsections. In Appendix~\ref{subsection:mixing}, we show an important mixing property of the options with failure framework. In Appendix~\ref{subsection:fs}, such a mixing property is used to prove a general contraction result of our forward-backward smoothing procedure (Theorem~\ref{thm:fb} and Corollary~\ref{corollary:fb}), similar to the concept of \emph{filtering stability} in the HMM literature. At a high level, considering the forward-backward recursion in the extended graphical model (Corollary~\ref{corollary:fb}), this result characterizes the effect of changing $\theta$ and the boundary conditions $\alpha^\theta_{k,0}$ and $\beta^\theta_{k,T}$ on the smoothing distribution $\gamma^\theta_{k,t}$, given any observation sequence $\{s_t,a_t\}_{t\in\mathbb Z}$. Due to this high level reasoning, we name this result as the \emph{smoothing stability} lemma. Appendix~\ref{subsection:approximation} provides concrete applications of this lemma to quantities defined in earlier sections. 

\subsection{Mixing}\label{subsection:mixing}

Recall that $\zeta$ is the auxiliary parameter in the options with failure framework. 

\begin{lemma}[Mixing]\label{lemma:mixing} There exists a constant $\eps_b>0$ and a conditional distribution $\bar\pi_{o,b}(o_t,b_t|s_t;\theta)$ parameterized by $\theta$ such that for all $\theta\in\calTheta$, with any input arguments $b_t$, $s_t$, $o_{t-1}$ and $o_t$, 
\begin{equation*}
0<\eps_b\zeta\bar\pi_{o,b}(o_t,b_t|s_t;\theta)\leq\pi_b(b_{t}|s_{t},o_{t-1};\theta_b)\bar\pi_{hi}(o_{t}|s_{t},o_{t-1},b_{t};\theta_{hi})\leq \eps_b^{-1}|\calo|\bar\pi_{o,b}(o_t,b_t|s_t;\theta).
\end{equation*}
\end{lemma}

\begin{proof}[Proof of Lemma~\ref{lemma:mixing}]The proof is separated into two parts. 

\vspace{1em}\noindent
1. We first show an intermediate result: there exists a constant $\eps_b>0$ and a conditional distribution $\bar \pi_b(b_t|s_t; \theta_b)$ parameterized by $\theta_b$ such that for all $\theta_b\in\calTheta_b$, with any input arguments $b_t$, $s_t$ and $o_{t-1}$, 
\begin{equation*}
0<\eps_b\bar \pi_b(b_t|s_t; \theta_b)\leq\pi_b(b_{t}|s_{t},o_{t-1};\theta_b)\leq \eps_b^{-1}\bar \pi_b(b_t|s_t; \theta_b). 
\end{equation*}

This can be proved as follows. Let $c_b=\inf_{\theta_b\in\calTheta_b}\min_{b_t,s_t,o_{t-1}}\pi_b(b_{t}|s_{t},o_{t-1};\theta_b)$. Similar to the procedure in Appendix~\ref{section:proof1}, from the non-degeneracy assumption, the differentiabiilty assumption and $\calTheta$ being compact, we have $c_b>0$. For any $\theta_b\in\Theta_b$, with any input arguments $b_t$ and $s_t$, let $f(b_t,s_t;\theta_b)=\min_{o_{t-1}\in \calo}\pi_b(b_{t}|s_{t},o_{t-1};\theta_b)$. Observe that $c_b\leq f(b_t,s_t;\theta_b)\leq 1$. Let $\eps_b=c_b/2$ and
\begin{equation*}
\bar \pi_b(b_t|s_t; \theta_b)=\frac{f(b_t,s_t;\theta_b)}{\sum_{b'_t\in\{0,1\}}f(b'_t,s_t;\theta_b)}. 
\end{equation*}
Clearly $\eps_b\bar \pi_b(b_t|s_t; \theta_b)>0$. Moreover, for any $o_{t-1}$, $\eps_b\bar \pi_b(b_t|s_t; \theta_b)< 2c_b\bar \pi_b(b_t|s_t; \theta_b)\leq f(b_t,s_t;\theta_b)\leq \pi_b(b_{t}|s_{t},o_{t-1};\theta_b)$. 

On the other hand, with any input arguments, 
\begin{equation*}
\eps_b^{-1}\bar \pi_b(b_t|s_t; \theta_b)\geq \eps_b^{-1}c_b/2=1\geq \pi_b(b_{t}|s_{t},o_{t-1};\theta_b), 
\end{equation*}
which completes the proof of the first part. 

\vspace{1em}\noindent
2. Define $\bar\pi_{o,b}(o_t,b_t|s_t;\theta)$ as follows. With any input arguments, let
\begin{align*}
&\bar\pi_{o,b}(o_t,b_t=0|s_t;\theta)\defeq\bar \pi_b(b_t=0|s_t; \theta_b)/|\calo|,\\
&\bar\pi_{o,b}(o_t,b_t=1|s_t;\theta)\defeq\bar \pi_b(b_t=1|s_t; \theta_b)\pi_{hi}(o_t|s_t;\theta_{hi}).
\end{align*}

Clearly $\eps_b\zeta\bar\pi_{o,b}(o_t,b_t|s_t;\theta)>0$. Omit the dependency on $\theta$ for a cleaner notation since every term is parameterized by $\theta$. When $b_t=1$, with any other input arguments, 
\begin{equation*}
\eps_b\bar \pi_b(b_t=1|s_t)\pi_{hi}(o_t|s_t)\leq\pi_b(b_{t}=1|s_{t},o_{t-1})\bar\pi_{hi}(o_{t}|s_{t},o_{t-1},b_{t}=1)\leq \eps_b^{-1}\bar \pi_b(b_t=1|s_t)\pi_{hi}(o_t|s_t).
\end{equation*}
Similarly, when $b_t=0$ and $o_t=o_{t-1}$, 
\begin{align*}
\eps_b\bar \pi_b(b_t=0|s_t)\zeta/|\calo|&\leq\eps_b\bar \pi_b(b_t=0|s_t)\bigg(1-\frac{|\calo|-1}{|\calo|}\zeta\bigg)\\
&\leq\pi_b(b_{t}=0|s_{t},o_{t-1})\bar\pi_{hi}(o_{t}=o_{t-1}|s_{t},o_{t-1},b_{t}=0)\\
&\leq \eps_b^{-1}\bar \pi_b(b_t=0|s_t).
\end{align*}

Finally, when $b_t=0$ and $o_t\neq o_{t-1}$,
\begin{equation*}
\eps_b\bar \pi_b(b_t=0|s_t)\zeta/|\calo|\leq \pi_b(b_{t}=0|s_{t},o_{t-1})\bar\pi_{hi}(o_{t}|s_{t},o_{t-1},b_{t}=0)\leq \eps_b^{-1}\bar \pi_b(b_t=0|s_t)\zeta/|\calo|.
\end{equation*}
Combining the above cases and the definition of $\bar\pi_{o,b}(o_t,b_t|s_t;\theta)$ completes the proof.
\end{proof}

\subsection{Smoothing stability}\label{subsection:fs}

Before stating the smoothing stability lemma, we introduce a few definitions. The quantities defined in this subsection depend on an observation sequence $\{s_t,a_t\}_{t\in\mathbb Z}$, but such a dependency is usually omitted to simplify the notation, unless specified otherwise. Consistent with our notations so far, in the following we make extensive use of the proportional symbol $\propto$. 

\subsubsection{Forward and backward recursion operators}\label{subsubsection:fb}

With any given observation sequence $\{s_t,a_t\}_{t\in\mathbb Z}$ and any $\theta\in\calTheta$, define the filtering operator $F^\theta_t$ as the following. For any probability measure $\vphi$ over $\mathcal O\times\{0,1\}$, $F^\theta_t\vphi$ is also a probability measure such that with any input arguments $o_t$ and $b_t$, 
\begin{equation}
F^\theta_t\vphi(o_{t},b_{t})\propto\sum_{o_{t-1},b_{t-1}}\pi_b(b_t|s_t,o_{t-1};\theta_b)\bar\pi_{hi}(o_t|s_t,o_{t-1},b_t;\theta_{hi})\pi_{lo}(a_t|s_t,o_t;\theta_{lo})\vphi(o_{t-1},b_{t-1}). \label{eq:def1}
\end{equation}
The RHS has exactly the form of the forward recursion, therefore the recursion on both $\alpha^{\theta}_{k,t}$ in (\ref{eq:forward}) and $\alpha^\theta_{\mu,t}$ in (\ref{eq:forwarde}) can be expressed using $F^\theta_t$. For generality, let $\{\vphi^\theta_t\}_{t\in\mathbb Z}$ and $\{\hat\vphi^{\hat\theta}_t\}_{t\in\mathbb Z}$ be any two indexed sets of probability measures such that $F^\theta_t\vphi^\theta_{t-1}=\vphi^\theta_t$ and $F^{\hat\theta}_t\hat\vphi^{\hat\theta}_{t-1}=\hat\vphi^{\hat\theta}_t$. We restrict $\{\vphi^\theta_t\}_{t\in\mathbb Z}$ and $\{\hat\vphi^{\hat\theta}_t\}_{t\in\mathbb Z}$ to be strictly positive. Due to Assumption~\ref{as:nondeg}, such a restriction is valid. Notice that $\theta$ and $\hat\theta$ here can be equal. We use the seemingly more complicated notation $\{\hat\vphi^{\hat\theta}_t\}_{t\in\mathbb Z}$ because even if $\theta=\hat\theta$, $\{\vphi^\theta_t\}_{t\in\mathbb Z}$ and $\{\hat\vphi^{\hat\theta}_t\}_{t\in\mathbb Z}$ are still different; in this case they are just two different sets of probability measures satisfying the same recursion $F^\theta_t$. 

Similarly, we define the backward recursion operator $B^\theta_t$ as follows. For any probability measure $\rho$ over $\mathcal O\times\{0,1\}$, $B^\theta_t\rho$ is also a probability measure such that with any input arguments $o_t$ and $b_t$, 
\begin{multline}
B^\theta_t\rho(o_t,b_t)\propto\sum_{o_{t+1},b_{t+1}}\pi_b(b_{t+1}|s_{t+1},o_{t};\theta_b)\bar\pi_{hi}(o_{t+1}|s_{t+1},o_{t},b_{t+1};\theta_{hi})\\
\times\pi_{lo}(a_{t+1}|s_{t+1},o_{t+1};\theta_{lo})\rho(o_{t+1},b_{t+1}).\label{eq:def2}
\end{multline}
The recursion on both $\beta^\theta_{t|T}$ in (\ref{eq:backward}) and $\beta^\theta_{k,t}$ in (\ref{eq:backwarde}) can be expressed using $B^\theta_t$. Let $\{\rho^\theta_t\}_{t\in\mathbb Z}$ and $\{\hat\rho^{\hat\theta}_t\}_{t\in\mathbb Z}$ be any two indexed sets of probability measures such that $B^\theta_t\rho^\theta_{t+1}=\rho^\theta_t$ and $B^{\hat\theta}_t\hat\rho^{\hat\theta}_{t+1}=\hat\rho^{\hat\theta}_t$. We restrict $\{\rho^\theta_t\}_{t\in\mathbb Z}$ and $\{\hat\rho^{\hat\theta}_t\}_{t\in\mathbb Z}$ to be strictly positive. 

The operation $\otimes$ is defined as follows: $\{(\vphi^{\theta}\otimes\hat\rho^{\hat\theta})_t\}_{t\in\mathbb Z}$ is an indexed set of probability measures such that for any input arguments $o_t$ and $b_t$,  
\begin{equation}
(\vphi^{\theta}\otimes\hat\rho^{\hat\theta})_t(o_t,b_t)\propto \vphi^{\theta}_t(o_t,b_t)\hat\rho^{\hat\theta}_t(o_t,b_t).\label{eq:def3}
\end{equation}

Finally, we clarify the use of $\propto$ in the above definitions. In (\ref{eq:def1}), (\ref{eq:def2}) and (\ref{eq:def3}), the normalizing constants replaced by $\propto$ are independent of the input arguments $(o_t,b_t)$. 

\subsubsection{Forward and backward smoothing operators}

For any $\theta,\hat\theta\in\calTheta$ and any $t$, with any observation sequence $\{s_t,a_t\}_{t\in\mathbb Z}$ and any input arguments $o_t$ and $b_t$, observe that
\begin{multline*}
(\hat\vphi^{\hat\theta}\otimes\rho^{\theta})_t(o_t,b_t)\propto\sum_{o_{t-1},b_{t-1}}\pi_b(b_t|s_t,o_{t-1};\hat\theta_b)\bar\pi_{hi}(o_t|s_t,o_{t-1},b_t;\hat\theta_{hi})\pi_{lo}(a_t|s_t,o_t;\hat\theta_{lo})\\
\times\rho^{\theta}_t(o_t,b_t)\frac{(\hat\vphi^{\hat\theta}\otimes\rho^{\theta})_{t-1}(o_{t-1},b_{t-1})}{\rho^{\theta}_{t-1}(o_{t-1},b_{t-1})},
\end{multline*}
and
\begin{equation*}
\rho^{\theta}_{t-1}(o_{t-1},b_{t-1})\propto \sum_{o'_{t},b'_{t}}\pi_b(b'_{t}|s_{t},o_{t-1};\theta_b)\bar\pi_{hi}(o'_{t}|s_{t},o_{t-1},b'_{t};\theta_{hi})\pi_{lo}(a_{t}|s_{t},o'_{t};\theta_{lo})\rho^{\theta}_t(o'_{t},b'_{t}). 
\end{equation*}
To simplify notation, let
\begin{equation}
h(\theta;o_{t-1},s_t,a_t,o_t,b_t)=\pi_b(b_t|s_t,o_{t-1};\theta_b)\bar\pi_{hi}(o_t|s_t,o_{t-1},b_t;\theta_{hi})\pi_{lo}(a_t|s_t,o_t;\theta_{lo}). \label{eq:definitionofh}
\end{equation}
Then, 
\begin{equation}
(\hat\vphi^{\hat\theta}\otimes\rho^{\theta})_t(o_t,b_t)=C^{\hat\theta,\theta}_F\sum_{o_{t-1},b_{t-1}}\frac{h(\hat\theta;o_{t-1},s_t,a_t,o_t,b_t)\rho^{\theta}_t(o_{t},b_{t})(\hat\vphi^{\hat\theta}\otimes\rho^{\theta})_{t-1}(o_{t-1},b_{t-1})}{\sum_{o'_t,b'_t}h(\theta;o_{t-1},s_t,a_t,o'_t,b'_t)\rho^{\theta}_t(o'_{t},b'_{t})},\label{eq:forwardop}
\end{equation}
where $C^{\hat\theta,\theta}_F$ is a normalizing constant such that
\begin{equation*}
\left(C^{\hat\theta,\theta}_F\right)^{-1}=\sum_{o_{t-1},b_{t-1}}\frac{\sum_{o_t,b_t}h(\hat\theta;o_{t-1},s_t,a_t,o_t,b_t)\rho^{\theta}_t(o_{t},b_{t})}{\sum_{o'_t,b'_t}h(\theta;o_{t-1},s_t,a_t,o'_t,b'_t)\rho^{\theta}_t(o'_{t},b'_{t})}(\hat\vphi^{\hat\theta}\otimes\rho^{\theta})_{t-1}(o_{t-1},b_{t-1}).
\end{equation*}

From (\ref{eq:forwardop}), we define the forward smoothing operator $K^{\hat\theta,\theta}_{F,t}$ on the probability measure $(\hat\vphi^{\hat\theta}\otimes\rho^{\theta})_{t-1}$ such that as probability measures, 
\begin{equation*}
(\hat\vphi^{\hat\theta}\otimes\rho^{\theta})_{t-1}K^{\hat\theta,\theta}_{F,t}=(\hat\vphi^{\hat\theta}\otimes\rho^{\theta})_{t}.
\end{equation*}
The subscript $F$ in $K^{\hat\theta,\theta}_{F,t}$ stands for \emph{forward}. $K^{\hat\theta,\theta}_{F,t}$ depends on the the parameters $\theta$ and $\hat\theta$, the observation $\{s_t,a_t\}_{t\in\mathbb Z}$ and the specific choice of $\{\rho^\theta_t\}_{t\in\mathbb Z}$. In the general case of $\theta\neq\hat\theta$, $K^{\hat\theta,\theta}_{F,t}$ is a nonlinear operator which requires rather sophisticated analysis. However, when $\theta=\hat\theta$, it is straightforward to verify that the normalizing constant $C^{\theta,\theta}_F=1$, and $K^{\theta,\theta}_{F,t}$ becomes a linear operator. 

In fact, the linear operator $K^{\theta,\theta}_{F,t}$ can be regarded as the standard operation of a Markov transition kernel on probability measures. With a slight overload of notation, define such a Markov transition kernel on $\calo\times\{0,1\}$, entry-wise, as the following. For any $(o_t,b_t)$ and $(o_{t-1},b_{t-1})$ in $\calo\times\{0,1\}$, 
\begin{equation}
K^{\theta,\theta}_{F,t}(o_{t},b_{t}|o_{t-1},b_{t-1})\defeq\frac{h(\theta;o_{t-1},s_t,a_t,o_t,b_t)\rho^\theta_t(o_t,b_t)}{\sum_{o'_{t},b'_{t}}h(\theta;o_{t-1},s_t,a_t,o'_t,b'_t)\rho^\theta_t(o'_{t},b'_{t})}.\label{eq:forwardkernel}
\end{equation}
We name this Markov transition kernel as the forward smoothing kernel. Such a definition is analogous to \emph{Markovian decomposition} in the HMM literature \citep{cappe2006inference}. The only caveat here is that we also allow perturbations on the parameter. The resulting operator $K^{\hat\theta,\theta}_{F,t}$ is nonlinear and no longer corresponds to a Markov transition kernel. 

To proceed, we characterize the difference between operators $K^{\hat\theta,\theta}_{F,t}$ and $K^{\theta,\theta}_{F,t}$ when $\hat\theta$ and $\theta$ are close. First, we show a version of Lipschitz continuity for the options with failure framework. 

\begin{lemma}[Lipschitz continuity]\label{lemma:lipschitz} For all $\theta\in\calTheta$ and $\delta>0$, there exists a real number $L_{\theta,\delta}$ such that with any input arguments $o_{t-1}$, $s_t$, $a_t$, $o_t$ and $b_t$, the function $h(\tilde\theta;o_{t-1},s_t,a_t,o_t,b_t)$ defined in (\ref{eq:definitionofh}) is $L_{\theta,\delta}$-Lipschitz with respect to $\tilde\theta$ on the set $\{\tilde\theta;\tilde\theta\in\calTheta, \norms{\tilde\theta-\theta}_2\leq \delta\}$. 
Moreover, $L_{\theta,\delta}$ is upper bounded by a constant that does not depend on $\theta$ and $\delta$. 
\end{lemma}

\begin{proof}[Proof of Lemma~\ref{lemma:lipschitz}]
Due to Assumption~\ref{as:continuity}, with any input arguments $o_{t-1}$, $s_t$, $a_t$, $o_t$ and $b_t$, $h(\tilde\theta;o_{t-1},s_t,a_t,o_t,b_t)$ is continuously differentiable with respect to $\tilde\theta\in\tilde\calTheta$. As continuously differentiable functions are Lipschitz continuous on convex and compact subsets, $h(\tilde\theta;o_{t-1},s_t,a_t,o_t,b_t)$ is Lipschitz continuous on $\calTheta$, hence also on $\{\tilde\theta;\tilde\theta\in\calTheta, \norms{\tilde\theta-\theta}_2\leq \delta\}$. The Lipschitz constants depend on the choice of input arguments $o_{t-1}$, $s_t$, $a_t$, $o_t$ and $b_t$. 

We can let $L_{\theta,\delta}$ be the smallest Lipschitz constant on $\{\tilde\theta;\tilde\theta\in\calTheta, \norms{\tilde\theta-\theta}_2\leq \delta\}$ that holds for all input arguments $o_{t-1}$, $s_t$, $a_t$, $o_t$ and $b_t$. Clearly $L_{\theta,\delta}$ is upper bounded by any Lipschitz constant on $\calTheta$ that holds for all input arguments, which does not depend on $\theta$ and $\delta$. 
\end{proof}

Next, we bound the difference between operators $K^{\hat\theta,\theta}_{F,t}$ and $K^{\theta,\theta}_{F,t}$. 

\begin{lemma}[Perturbation on the forward smoothing kernel]\label{lemma:perturbsmooth} Let $\vphi$ be any probability measure on $\mathcal O\times\{0,1\}$. Let $K^{\hat\theta,\theta}_{F,t}$ and $K^{\theta,\theta}_{F,t}$ be defined with the same observation sequence $\{s_t,a_t\}_{t\in\mathbb Z}$ and the same choice of $\{\rho^\theta_t\}_{t\in\mathbb Z}$. Their difference is only in the first entry of the superscript ($\hat\theta$ in $K^{\hat\theta,\theta}_{F,t}$; $\theta$ in $K^{\theta,\theta}_{F,t}$). Then, for all $t$, $\vphi$, $\theta$, $\hat\theta$, $\{s_t,a_t\}_{t\in\mathbb Z}$ and $\{\rho^\theta_t\}_{t\in\mathbb Z}$, 
\begin{equation*}
\norm{\vphi K^{\hat\theta,\theta}_{F,t}-\vphi K^{\theta,\theta}_{F,t}}_\tv\leq\frac{\max_{o_{t-1},o_t,b_t}h(\theta;o_{t-1},s_t,a_t,o_t,b_t)}{\min_{o_{t-1},o_t,b_t}h(\theta;o_{t-1},s_t,a_t,o_t,b_t)}\frac{L_{\theta,\norms{\hat\theta-\theta}_2}\norms{\hat\theta-\theta}_2}{\min_{o_{t-1},o_t,b_t}h(\hat\theta;o_{t-1},s_t,a_t,o_t,b_t)}.
\end{equation*}
\end{lemma}

\begin{proof}[Proof of Lemma~\ref{lemma:perturbsmooth}]
From the definitions, for any $t$, $\vphi$, $\theta$, $\hat\theta$, $\{s_t,a_t\}_{t\in\mathbb Z}$ and $\{\rho^\theta_t\}_{t\in\mathbb Z}$, 
\begin{align*}
&\norm{\vphi K^{\hat\theta,\theta}_{F,t}-\vphi K^{\theta,\theta}_{F,t}}_\tv\\
=~&\frac{1}{2}\sum_{o_t,b_t}\left|\sum_{o_{t-1},b_{t-1}}\frac{\left[C^{\hat\theta,\theta}_F h(\hat\theta;o_{t-1},s_t,a_t,o_t,b_t)-h(\theta;o_{t-1},s_t,a_t,o_t,b_t)\right]}{\sum_{o'_{t},b'_{t}}h(\theta;o_{t-1},s_t,a_t,o'_t,b'_t)\rho^\theta_t(o'_{t},b'_{t})}\rho^\theta_t(o_t,b_t)\vphi(o_{t-1},b_{t-1})\right|\\
\leq~&\frac{1}{2}\sum_{o_{t-1},b_{t-1}}\frac{\sum_{o_t,b_t}\left|C^{\hat\theta,\theta}_F h(\hat\theta;o_{t-1},s_t,a_t,o_t,b_t)-h(\theta;o_{t-1},s_t,a_t,o_t,b_t)\right|\rho^\theta_t(o_t,b_t)}{\sum_{o'_{t},b'_{t}}h(\theta;o_{t-1},s_t,a_t,o'_t,b'_t)\rho^\theta_t(o'_{t},b'_{t})}\vphi(o_{t-1},b_{t-1}).
\end{align*}

From the definition of the normalizing constant $C^{\hat\theta,\theta}_F$, we have
\begin{equation*}
\left(C^{\hat\theta,\theta}_F\right)^{-1}=\sum_{o_{t-1},b_{t-1}}\frac{\sum_{o_t,b_t}h(\hat\theta;o_{t-1},s_t,a_t,o_t,b_t)\rho^{\theta}_t(o_{t},b_{t})}{\sum_{o'_t,b'_t}h(\theta;o_{t-1},s_t,a_t,o'_t,b'_t)\rho^{\theta}_t(o'_{t},b'_{t})}\vphi(o_{t-1},b_{t-1}).
\end{equation*}
Therefore, 
\begin{equation*}
C^{\hat\theta,\theta}_F\leq\max_{o_{t-1}}\frac{\sum_{o_t,b_t}h(\theta;o_{t-1},s_t,a_t,o_t,b_t)\rho^{\theta}_t(o_{t},b_{t})}{\sum_{o_t,b_t}h(\hat\theta;o_{t-1},s_t,a_t,o_t,b_t)\rho^{\theta}_t(o_{t},b_{t})},
\end{equation*}
and
\begin{align*}
&\left|C^{\hat\theta,\theta}_F-1\right|\\
=~&\left|\sum_{o_{t-1},b_{t-1}}\frac{\sum_{o_t,b_t}[h(\hat\theta;o_{t-1},s_t,a_t,o_t,b_t)-h(\theta;o_{t-1},s_t,a_t,o_t,b_t)]\rho^{\theta}_t(o_{t},b_{t})}{\sum_{o_t,b_t}h(\theta;o_{t-1},s_t,a_t,o_t,b_t)\rho^{\theta}_t(o_{t},b_{t})}\vphi(o_{t-1},b_{t-1})\right|C^{\hat\theta,\theta}_F\\
\leq~&\frac{L_{\theta,\norms{\hat\theta-\theta}_2}\norms{\hat\theta-\theta}_2 C^{\hat\theta,\theta}_F}{\min_{o_{t-1}}\sum_{o_t,b_t}h(\theta;o_{t-1},s_t,a_t,o_t,b_t)\rho^{\theta}_t(o_{t},b_{t})}.
\end{align*}

As a result, for any given $o_{t-1}$, $o_t$ and $b_t$, 
\begin{align*}
&\left|C^{\hat\theta,\theta}_Fh(\hat\theta;o_{t-1},s_t,a_t,o_t,b_t)-h(\theta;o_{t-1},s_t,a_t,o_t,b_t)\right|\\
\leq~&C^{\hat\theta,\theta}_F\left|h(\hat\theta;o_{t-1},s_t,a_t,o_t,b_t)-h(\theta;o_{t-1},s_t,a_t,o_t,b_t)\right|+\left|C^{\hat\theta,\theta}_F-1\right|h(\theta;o_{t-1},s_t,a_t,o_t,b_t)\\
\leq~& \left[1+\frac{h(\theta;o_{t-1},s_t,a_t,o_t,b_t)}{\min_{o'_{t-1}}\sum_{o'_t,b'_t}h(\theta;o'_{t-1},s_t,a_t,o'_t,b'_t)\rho^{\theta}_t(o'_{t},b'_{t})}\right]L_{\theta,\norms{\hat\theta-\theta}_2}\norm{\hat\theta-\theta}_2 C^{\hat\theta,\theta}_F. 
\end{align*}

Combining everything together, 
\begin{align*}
&\norm{\vphi K^{\hat\theta,\theta}_{F,t}-\vphi K^{\theta,\theta}_{F,t}}_\tv\\
\leq~&L_{\theta,\norms{\hat\theta-\theta}_2}\norm{\hat\theta-\theta}_2 C^{\hat\theta,\theta}_F\times\max_{o_{t-1}}\frac{1+\frac{\sum_{o_t,b_t}h(\theta;o_{t-1},s_t,a_t,o_t,b_t)\rho^\theta_t(o_t,b_t)}{\min_{o'_{t-1}}\sum_{o'_t,b'_t}h(\theta;o'_{t-1},s_t,a_t,o'_t,b'_t)\rho^{\theta}_t(o'_{t},b'_{t})}}{2\sum_{o'_{t},b'_{t}}h(\theta;o_{t-1},s_t,a_t,o'_t,b'_t)\rho^\theta_t(o'_{t},b'_{t})}\\
=~& \frac{L_{\theta,\norms{\hat\theta-\theta}_2}\norms{\hat\theta-\theta}_2 C^{\hat\theta,\theta}_F}{\min_{o'_{t-1}}\sum_{o'_t,b'_t}h(\theta;o'_{t-1},s_t,a_t,o'_t,b'_t)\rho^{\theta}_t(o'_{t},b'_{t})}\\
\leq~&\frac{\max_{o_{t-1},o_t,b_t}h(\theta;o_{t-1},s_t,a_t,o_t,b_t)}{\min_{o_{t-1},o_t,b_t}h(\theta;o_{t-1},s_t,a_t,o_t,b_t)}\frac{L_{\theta,\norms{\hat\theta-\theta}_2}\norms{\hat\theta-\theta}_2}{\min_{o_{t-1},o_t,b_t}h(\hat\theta;o_{t-1},s_t,a_t,o_t,b_t)}.\qedhere
\end{align*}
\end{proof}

On the other hand, we can formulate a backward smoothing recursion as
\begin{equation}
(\vphi^{\theta}\otimes\hat\rho^{\hat\theta})_t(o_t,b_t)=C^{\theta,\hat\theta}_B\sum_{o_{t+1},b_{t+1}}\frac{h(\hat\theta;o_{t},s_{t+1},a_{t+1},o_{t+1},b_{t+1})\vphi^{\theta}_t(o_{t},b_{t})(\vphi^{\theta}\otimes\hat\rho^{\hat\theta})_{t+1}(o_{t+1},b_{t+1})}{\sum_{o'_t,b'_t}h(\theta;o'_{t},s_{t+1},a_{t+1},o_{t+1},b_{t+1})\vphi^{\theta}_t(o'_{t},b'_{t})},\label{eq:backwardop}
\end{equation}
where $C^{\theta,\hat\theta}_B$ is a normalizing constant such that
\begin{equation*}
\left(C^{\theta,\hat\theta}_B\right)^{-1}=\sum_{o_{t+1},b_{t+1}}\frac{\sum_{o_t,b_t}h(\hat\theta;o_{t},s_{t+1},a_{t+1},o_{t+1},b_{t+1})\vphi^{\theta}_t(o_{t},b_{t})}{\sum_{o'_t,b'_t}h(\theta;o'_{t},s_{t+1},a_{t+1},o_{t+1},b_{t+1})\vphi^{\theta}_t(o'_{t},b'_{t})}(\vphi^{\theta}\otimes\hat\rho^{\hat\theta})_{t+1}(o_{t+1},b_{t+1}).
\end{equation*}

The subscript $B$ in $K^{\theta,\hat\theta}_{B,t}$ stands for \emph{backward}. Similar to the forward smoothing operator $K^{\hat\theta,\theta}_{F,t}$, we can define the backward smoothing operator $K^{\theta,\hat\theta}_{B,t}$ from (\ref{eq:backwardop}) such that as probability measures, 
\begin{equation*}
(\vphi^{\theta}\otimes\hat\rho^{\hat\theta})_{t+1}K^{\theta,\hat\theta}_{B,t}=(\vphi^{\theta}\otimes\hat\rho^{\hat\theta})_{t}.
\end{equation*}
Analogous to $K^{\hat\theta,\theta}_{F,t}$, in the general case of $\theta\neq\hat\theta$, $K^{\theta,\hat\theta}_{B,t}$ is a nonlinear operator. However, if $\theta=\hat\theta$, $K^{\theta,\hat\theta}_{B,t}$ becomes a linear operator and induces a Markov transition kernel. The following lemma is similar to Lemma~\ref{lemma:perturbsmooth}. We state it without proof. 

\begin{lemma}[Perturbation on the backward smoothing kernel]\label{lemma:perturbsmooth2} Let $\rho$ be any probability measure on $\mathcal O\times\{0,1\}$. Let $K^{\theta,\hat\theta}_{B,t}$ and $K^{\theta,\theta}_{B,t}$ be defined with the same observation sequence $\{s_t,a_t\}_{t\in\mathbb Z}$ and the same choice of $\{\vphi^{\theta}_t\}_{t\in\mathbb Z}$. Then, for any $t$, $\rho$, $\theta$, $\hat\theta$, $\{s_t,a_t\}_{t\in\mathbb Z}$ and $\{\vphi^{\theta}_t\}_{t\in\mathbb Z}$, 
\begin{multline*}
\norm{\rho K^{\theta,\hat\theta}_{B,t}-\rho K^{\theta,\theta}_{B,t}}_\tv\leq\frac{\max_{o_t,o_{t+1},b_{t+1}}h(\theta;o_{t},s_{t+1},a_{t+1},o_{t+1},b_{t+1})}{\min_{o_t,o_{t+1},b_{t+1}}h(\theta;o_{t},s_{t+1},a_{t+1},o_{t+1},b_{t+1})}\\ \times\frac{L_{\hat\theta,\norms{\hat\theta-\theta}_2}\norms{\hat\theta-\theta}_2}{\min_{o_t,o_{t+1},b_{t+1}}h(\hat\theta;o_{t},s_{t+1},a_{t+1},o_{t+1},b_{t+1})}.
\end{multline*}
\end{lemma}

Notice that the bounds in both Lemma~\ref{lemma:perturbsmooth} and Lemma~\ref{lemma:perturbsmooth2} depend on the observation sequence $\{s_t,a_t\}_{t\in\mathbb Z}$. 

\subsubsection{A perturbed contraction result for smoothing stability}

For any $t_1,t_2\in\mathbb Z$ with $t_1\leq t_2$, let $\mathbb I=[t_1:t_2]$. 
Remember the following definition from Appendix~\ref{subsubsection:fb}, with the index set restricted to $\mathbb I$: for any $\theta,\hat\theta\in\calTheta$, $\{\vphi^{\theta}_t\}_{t\in\mathbb I}$ and $\{\hat\vphi^{\hat\theta}_t\}_{t\in\mathbb I}$ are two indexed sets of probability measures defined on $\calo\times\{0,1\}$ such that, for all $t\in\mathbb I$, (1) if $t\neq t_1$, $F^{\theta}_t\vphi^{\theta}_{t-1}=\vphi^{\theta}_t$ and $F^{\hat\theta}_t\hat\vphi^{\hat\theta}_{t-1}=\hat\vphi^{\hat\theta}_t$; (2) $\vphi^{\theta}_t$ and $\hat\vphi^{\hat\theta}_t$ are strictly positive on their domains. $\{\rho^{\theta}_t\}_{t\in\mathbb I}$ and $\{\hat\rho^{\hat\theta}_t\}_{t\in\mathbb I}$ are two indexed sets of probability measures defined on $\calo\times\{0,1\}$ such that for all $t\in\mathbb I$, (1) if $t\neq t_2$, $B^{\theta}_t\rho^{\theta}_{t+1}=\rho^{\theta}_t$ and $B^{\hat\theta}_t\hat\rho^{\hat\theta}_{t+1}=\hat\rho^{\hat\theta}_t$; (2) $\rho^{\theta}_t$ and $\hat\rho^{\hat\theta}_t$ are strictly positive on their domains. $\theta$ and $\hat\theta$ are allowed to be equal. 

The smoothing stability lemma is stated as follows. 

\begin{lemma}[Smoothing stability] \label{lemma:fstability} With $\{\vphi^{\theta}_t\}_{t\in\mathbb I}$, $\{\hat\vphi^{\hat\theta}_t\}_{t\in\mathbb I}$, $\{\rho^{\theta}_t\}_{t\in\mathbb I}$ and $\{\hat\rho^{\hat\theta}_t\}_{t\in\mathbb I}$ defined above, 
\begin{equation*}
\norm{(\vphi^{\theta}\otimes\rho^{\theta})_{t_2}-(\hat\vphi^{\hat\theta}\otimes\rho^{\theta})_{t_2}}_\tv\leq \bigg(1-\frac{\eps^2_b\zeta}{|\calo|}\bigg)^{t_2-t_1}+\frac{|\calo|z_{\theta,\hat\theta}L_{\theta,\norms{\hat\theta-\theta}_2}}{\eps^2_b\zeta}\norm{\hat\theta-\theta}_2,
\end{equation*}
\begin{equation*}
\norm{(\hat\vphi^{\hat\theta}\otimes\rho^{\theta})_{t_1}-(\hat\vphi^{\hat\theta}\otimes\hat\rho^{\hat\theta})_{t_1}}_\tv\leq \bigg(1-\frac{\eps^2_b\zeta}{|\calo|}\bigg)^{t_2-t_1}+\frac{|\calo|z_{\theta,\hat\theta}L_{\theta,\norms{\hat\theta-\theta}_2}}{\eps^2_b\zeta}\norm{\hat\theta-\theta}_2, 
\end{equation*}
where $z_{\theta,\theta'}$ is a positive real number dependent only on $\theta$ and $\hat\theta$. Specifically, 
\begin{equation*}
z_{\theta,\theta'}=\max_{s'_t,a'_t}\frac{[\max_{o_{t-1},o_t,b_t}h(\theta;o_{t-1},s'_t,a'_t,o_t,b_t)]\vee[\max_{o_{t-1},o_t,b_t}h(\hat\theta;o_{t-1},s'_t,a'_t,o_t,b_t)]}{[\min_{o_{t-1},o_t,b_t}h(\theta;o_{t-1},s'_t,a'_t,o_t,b_t)][\min_{o_{t-1},o_t,b_t}h(\hat\theta;o_{t-1},s'_t,a'_t,o_t,b_t)]}.
\end{equation*}
\end{lemma}

Intuitively, if $\hat\theta=\theta$, Lemma~\ref{lemma:fstability} has the form of an exact contraction, which is similar to the standard filtering stability result for HMMs. Indeed, our proof uses the classical techniques of uniform forgetting from the HMM literature \citep{cappe2006inference}. If $\hat\theta$ is different from $\theta$, such a contraction is perturbed. For HMMs, similar results are provided in \citep[Proposition~2.2, Theorem~2.3]{de2017consistent}. 

\begin{proof}[Proof of Lemma~\ref{lemma:fstability}] Consider the first bound. It holds trivially when $t_2=t_1$. Now consider only $t_2>t_1$. Using the forward smoothing operators, for any $t_1<t\leq t_2$, 
\begin{equation*}
(\vphi^{\theta}\otimes\rho^{\theta})_{t-1}K^{\theta,\theta}_{F,t}-(\hat\vphi^{\hat\theta}\otimes\rho^{\theta})_{t-1}K^{\hat\theta,\theta}_{F,t}=(\vphi^{\theta}\otimes\rho^{\theta})_{t}-(\hat\vphi^{\hat\theta}\otimes\rho^{\theta})_{t}.
\end{equation*}
Therefore, 
\begin{multline*}
\norm{(\vphi^{\theta}\otimes\rho^{\theta})_{t}-(\hat\vphi^{\hat\theta}\otimes\rho^{\theta})_{t}}_\tv\leq \norm{\left[(\vphi^{\theta}\otimes\rho^{\theta})_{t-1}-(\hat\vphi^{\hat\theta}\otimes\rho^{\theta})_{t-1}\right]K^{\theta,\theta}_{F,t}}_\tv\\+\norm{(\hat\vphi^{\hat\theta}\otimes\rho^{\theta})_{t-1}K^{\theta,\theta}_{F,t}-(\hat\vphi^{\hat\theta}\otimes\rho^{\theta})_{t-1}K^{\hat\theta,\theta}_{F,t}}_\tv, 
\end{multline*}
where the first term is due to $K^{\theta,\theta}_{F,t}$ being a linear operator. 

From Lemma~\ref{lemma:perturbsmooth}, the second term on the RHS is upper bounded by $z_{\theta,\hat\theta}L_{\theta,\norms{\hat\theta-\theta}_2}\norms{\hat\theta-\theta}_2$. As for the first term, we can construct the classical Doeblin-type minorization condition \citep[Chap.~4.3]{cappe2006inference}. Applying Lemma~\ref{lemma:mixing} in the definition of the Markov transition kernel $K^{\theta,\theta}_{F,t}$ (\ref{eq:forwardkernel}), we have
\begin{equation}
K^{\theta,\theta}_{F,t}(o_{t},b_{t}|o_{t-1},b_{t-1})\geq\frac{\eps^2_b\zeta}{|\calo|}\frac{\bar\pi_{o,b}(o_t,b_t|s_t;\theta)\pi_{lo}(a_t|s_t,o_t;\theta_{lo})\rho^\theta_t(o_t,b_t)}{\sum_{o'_{t},b'_{t}}\bar\pi_{o,b}(o'_t,b'_t|s_t;\theta)\pi_{lo}(a_{t}|s_{t},o'_{t};\theta_{lo})\rho^\theta_t(o'_{t},b'_{t})}
\eqdef\frac{\eps^2_b\zeta}{|\calo|}\bar \pi^{\theta}_{F,t}(o_t,b_t).\label{eq:onestepdoeblin}
\end{equation}
Observe that $\bar \pi^{\theta}_{F,t}$ just defined is a probability measure. Further define $\bar K^{\theta,\theta}_{F,t}$ entry-wise as
\begin{equation*}
\bar K^{\theta,\theta}_{F,t}(o_{t},b_{t}|o_{t-1},b_{t-1})\defeq\bigg(1-\frac{\eps^2_b\zeta}{|\calo|}\bigg)^{-1}\bigg(K^{\theta,\theta}_{F,t}(o_{t},b_{t}|o_{t-1},b_{t-1})-\frac{\eps^2_b\zeta}{|\calo|}\bar \pi^{\theta}_{F,t}(o_t,b_t)\bigg).
\end{equation*}
We can verify that $\bar K^{\theta,\theta}_{F,t}$ is also a Markov transition kernel. Moreover, 
\begin{equation*}
\left[(\vphi^{\theta}\otimes\rho^{\theta})_{t-1}-(\hat\vphi^{\hat\theta}\otimes\rho^{\theta})_{t-1}\right]K^{\theta,\theta}_{F,t}=\bigg(1-\frac{\eps^2_b\zeta}{|\calo|}\bigg)\left[(\vphi^{\theta}\otimes\rho^{\theta})_{t-1}-(\hat\vphi^{\hat\theta}\otimes\rho^{\theta})_{t-1}\right]\bar K^{\theta,\theta}_{F,t}.
\end{equation*}

To proceed, the standard approach is to use the fact that the Dobrushin coefficient of $\bar K^{\theta,\theta}_{F,t}$ is upper bounded by one. For clarity, we avoid such definitions and take a more direct approach here, which requires the extension of the total variation distance for two probability measures to the total variation norm for a finite signed measure. For a finite signed measure $\nu$ over a finite set $\Omega$, let the total variation norm of $\nu$ be
\begin{equation*}
\norm{\nu}_\tv\defeq \frac{1}{2}\sum_{\omega\in\Omega}\left|\nu(\omega)\right|.
\end{equation*}
When $\nu$ is the difference between two probability measures $\nu_1-\nu_2$, the total variation norm of $\nu$ coincides with the total variation distance between $\nu_1$ and $\nu_2$. Therefore, the same notation $\norms{\cdot}_\tv$ is adopted here. 

Let $\mathcal{\bar M}(\calo\times\{0,1\})$ be the set of finite signed measures over the finite set $\calo\times\{0,1\}$. From \citep[Chap.~4.3.1]{cappe2006inference}, $\mathcal{\bar M}(\calo\times\{0,1\})$ is a Banach space. Define an operator norm $\norms{\cdot}_\op$ for $\bar K^{\theta,\theta}_{F,t}$ as
\begin{equation*}
\norm{\bar K^{\theta,\theta}_{F,t}}_\op\defeq\sup\left\{\norm{\nu\bar K^{\theta,\theta}_{F,t}}_\tv;\norm{\nu}_\tv=1,\nu\in\mathcal{\bar M}(\calo\times\{0,1\})\right\}. 
\end{equation*}
Since $\bar K^{\theta,\theta}_{F,t}$ is a Markov transition kernel, $\norms{\bar K^{\theta,\theta}_{F,t}}_\op=1$ \citep[Lemma 4.3.6]{cappe2006inference}. Therefore, 
\begin{align*}
&\norm{(\vphi^{\theta}\otimes\rho^{\theta})_{t_2}-(\hat\vphi^{\hat\theta}\otimes\rho^{\theta})_{t_2}}_\tv\\
\leq~&\norm{\left[(\vphi^{\theta}\otimes\rho^{\theta})_{t_2-1}-(\hat\vphi^{\hat\theta}\otimes\rho^{\theta})_{t_2-1}\right]K^{\theta,\theta}_{F,t_2}}_\tv+\norm{(\hat\vphi^{\hat\theta}\otimes\rho^{\theta})_{t_2-1}\left(K^{\theta,\theta}_{F,t_2}-K^{\hat\theta,\theta}_{F,t_2}\right)}_\tv\\
=~&\bigg(1-\frac{\eps^2_b\zeta}{|\calo|}\bigg)\norm{\left[(\vphi^{\theta}\otimes\rho^{\theta})_{t_2-1}-(\hat\vphi^{\hat\theta}\otimes\rho^{\theta})_{t_2-1}\right]\bar K^{\theta,\theta}_{F,t_2}}_\tv+z_{\theta,\hat\theta}L_{\theta,\norms{\hat\theta-\theta}_2}\norms{\hat\theta-\theta}_2\\
\leq~&\bigg(1-\frac{\eps^2_b\zeta}{|\calo|}\bigg)\norm{(\vphi^{\theta}\otimes\rho^{\theta})_{t_2-1}-(\hat\vphi^{\hat\theta}\otimes\rho^{\theta})_{t_2-1}}_\tv\norm{\bar K^{\theta,\theta}_{F,t_2}}_\op+z_{\theta,\hat\theta}L_{\theta,\norms{\hat\theta-\theta}_2}\norms{\hat\theta-\theta}_2\\
=~&\bigg(1-\frac{\eps^2_b\zeta}{|\calo|}\bigg)\norm{(\vphi^{\theta}\otimes\rho^{\theta})_{t_2-1}-(\hat\vphi^{\hat\theta}\otimes\rho^{\theta})_{t_2-1}}_\tv+z_{\theta,\hat\theta}L_{\theta,\norms{\hat\theta-\theta}_2}\norms{\hat\theta-\theta}_2.
\end{align*}
The second inequality is due to the sub-multiplicativity of the operator norm. Finally, the desirable result follows from unrolling the summation and identifying the geometric series. 

The proof of the second bound is analogous, using the backward smoothing operators instead of the forward smoothing operators. Details are omitted.
\end{proof} 

Note that Lemma~\ref{lemma:fstability} only holds when considering the options with failure framework. For the standard options framework, the one-step Doeblin-type minorization condition (\ref{eq:onestepdoeblin}) we construct in the proof does not hold anymore, due to the failure of Lemma~\ref{lemma:mixing}. Instead, one could target the two-step minorization condition: define a two step smoothing kernel similar to $K^{\theta,\theta}_{F,t}$ and lower bound it similar to (\ref{eq:onestepdoeblin}). Notations are much more complicated. For simplicity, this extension is not considered in this paper. 

\subsection{The approximation lemmas}\label{subsection:approximation}

This subsection applies Lemma~\ref{lemma:fstability} to quantities defined in earlier sections. 

First, we bound the difference of smoothing distributions in the non-extended graphical model (as in Theorem~\ref{thm:fb}) and the extended one with parameter $k$ (as in Corollary~\ref{corollary:fb}). The parameter $\theta$ in the two models can be different. The bounds use quantities defined in Appendix~\ref{subsection:mixing} and Appendix~\ref{subsection:fs}. Recall the definition of $\Omega$ from \ref{eq:omega}. 

\begin{lemma}[Bounding the difference of smoothing distributions, Part I]\label{lemma:smoothingdiff1} For all $\theta,\hat\theta\in\calTheta$, $k\in\mathbb N_+$ and $\mu\in\mathcal{M}$, with the observation sequence $\{s_t,a_t\}_{t\in\Z}$ corresponding to any $\omega\in\Omega$, we have

\begin{enumerate}[leftmargin=*]
\item $\forall t\in[1:T]$,
\begin{equation*}
\norm{\gamma^{\theta}_{\mu,t|T}-\gamma^{\hat\theta}_{k,t}}_\tv\leq\bigg(1-\frac{\eps^2_b\zeta}{|\calo|}\bigg)^{t-1}+\bigg(1-\frac{\eps^2_b\zeta}{|\calo|}\bigg)^{T-t}+\frac{2|\calo|z_{\theta,\hat\theta}L_{\theta,\norms{\hat\theta-\theta}_2}}{\eps^2_b\zeta}\norm{\hat\theta-\theta}_2.
\end{equation*}
\item $\forall t\in[2:T]$,
\begin{equation*}
\norm{\tilde\gamma^{\theta}_{\mu,t|T}-\tilde\gamma^{\hat\theta}_{k,t}}_\tv\leq 2\bigg(1-\frac{\eps^2_b\zeta}{|\calo|}\bigg)^{t-2}+\bigg(1-\frac{\eps^2_b\zeta}{|\calo|}\bigg)^{T-t}+\frac{4|\calo|z_{\theta,\hat\theta}L_{\theta,\norms{\hat\theta-\theta}_2}}{\eps^2_b\zeta}\norm{\hat\theta-\theta}_2.
\end{equation*}
\end{enumerate}
\end{lemma}

Similarly, we can bound the difference of smoothing distributions in two extended graphical models with different $k$ and different parameter $\theta$. 

\begin{lemma}[Bounding the difference of smoothing distributions, Part II]\label{lemma:smoothingdiff2} For all $\theta,\hat\theta\in\calTheta$ and $t\in[1:T]$, with any two integers $k_2>k_1>0$ and the observation sequence $\{s_t,a_t\}_{t\in\Z}$ corresponding to any $\omega\in\Omega$, we have
\begin{equation*}
\norm{\gamma^{\theta}_{k_1,t}-\gamma^{\hat\theta}_{k_2,t}}_\tv\leq\bigg(1-\frac{\eps^2_b\zeta}{|\calo|}\bigg)^{t+k_1-1}+\bigg(1-\frac{\eps^2_b\zeta}{|\calo|}\bigg)^{T+k_1-t}+\frac{2|\calo|z_{\theta,\hat\theta}L_{\theta,\norms{\hat\theta-\theta}_2}}{\eps^2_b\zeta}\norm{\hat\theta-\theta}_2,
\end{equation*}
\begin{equation*}
\norm{\tilde \gamma^{\theta}_{k_1,t}-\tilde \gamma^{\hat\theta}_{k_2,t}}_\tv\leq 2\bigg(1-\frac{\eps^2_b\zeta}{|\calo|}\bigg)^{t+k_1-2}+\bigg(1-\frac{\eps^2_b\zeta}{|\calo|}\bigg)^{T+k_1-t}+\frac{4|\calo|z_{\theta,\hat\theta}L_{\theta,\norms{\hat\theta-\theta}_2}}{\eps^2_b\zeta}\norm{\hat\theta-\theta}_2.
\end{equation*}
\end{lemma}

It can be easily verified that in Lemma~\ref{lemma:smoothingdiff1} and Lemma~\ref{lemma:smoothingdiff2}, the bounds still hold if $\theta$ and $\hat\theta$ on the LHS are interchanged. We only present the proof of Lemma~\ref{lemma:smoothingdiff1}. As for Lemma~\ref{lemma:smoothingdiff2}, the proof is analogous therefore omitted. Our proof essentially relies on the smoothing stability lemma (Lemma~\ref{lemma:fstability}). 

\begin{proof}[Proof of Lemma~\ref{lemma:smoothingdiff1}]
Consider the first bound. For a cleaner notation, let
\begin{equation*}
\Delta_{\theta,\hat\theta}=\frac{|\calo|z_{\theta,\hat\theta}L_{\theta,\norms{\hat\theta-\theta}_2}}{\eps^2_b\zeta}\norm{\hat\theta-\theta}_2.
\end{equation*}

Apply Lemma~\ref{lemma:fstability} as follows: $\forall t\in[1:T]$, let $\vphi^\theta_t=\alpha^\theta_{\mu,t}$ and $\hat\vphi^{\hat\theta}_t=\alpha^{\hat\theta}_{k,t}$; let $\rho^\theta_t=\beta^\theta_{t|T}$ and $\hat\rho^{\hat\theta}_t=\beta^{\hat\theta}_{k,t}$. Due to Assumption~\ref{as:nondeg}, the strictly positive requirement is satisfied. Then, we have
\begin{equation*}
\norm{\frac{\alpha^\theta_{\mu,t}\cdot\beta^\theta_{t|T}}{\langle\alpha^\theta_{\mu,t},\beta^\theta_{t|T}\rangle}-\frac{\alpha^{\hat\theta}_{k,t}\cdot\beta^\theta_{t|T}}{\langle\alpha^{\hat\theta}_{k,t},\beta^\theta_{t|T}\rangle}}_\tv\leq \bigg(1-\frac{\eps^2_b\zeta}{|\calo|}\bigg)^{t-1}+\Delta_{\theta,\hat\theta},
\end{equation*}
\begin{equation*}
\norm{\frac{\alpha^{\hat\theta}_{k,t}\cdot\beta^\theta_{t|T}}{\langle\alpha^{\hat\theta}_{k,t},\beta^\theta_{t|T}\rangle}-\frac{\alpha^{\hat\theta}_{k,t}\cdot\beta^{\hat\theta}_{k,t}}{\langle\alpha^{\hat\theta}_{k,t},\beta^{\hat\theta}_{k,t}\rangle}}_\tv\leq \bigg(1-\frac{\eps^2_b\zeta}{|\calo|}\bigg)^{T-t}+\Delta_{\theta,\hat\theta},
\end{equation*}
where $\cdot$ denotes element-wise product and $\langle\cdot,\cdot\rangle$ denotes Euclidean inner product. Therefore, 
\begin{align*}
\norm{\gamma^{\theta}_{\mu,t|T}-\gamma^{\hat\theta}_{k,t}}_\tv&=\norm{\frac{\alpha^\theta_{\mu,t}\cdot\beta^\theta_{t|T}}{\langle\alpha^\theta_{\mu,t},\beta^\theta_{t|T}\rangle}-\frac{\alpha^{\hat\theta}_{k,t}\cdot\beta^{\hat\theta}_{k,t}}{\langle\alpha^{\hat\theta}_{k,t},\beta^{\hat\theta}_{k,t}\rangle}}_\tv\\
&\leq \norm{\frac{\alpha^\theta_{\mu,t}\cdot\beta^\theta_{t|T}}{\langle\alpha^\theta_{\mu,t},\beta^\theta_{t|T}\rangle}-\frac{\alpha^{\hat\theta}_{k,t}\cdot\beta^\theta_{t|T}}{\langle\alpha^{\hat\theta}_{k,t},\beta^\theta_{t|T}\rangle}}_\tv+\norm{\frac{\alpha^{\hat\theta}_{k,t}\cdot\beta^\theta_{t|T}}{\langle\alpha^{\hat\theta}_{k,t},\beta^\theta_{t|T}\rangle}-\frac{\alpha^{\hat\theta}_{k,t}\cdot\beta^{\hat\theta}_{k,t}}{\langle\alpha^{\hat\theta}_{k,t},\beta^{\hat\theta}_{k,t}\rangle}}_\tv\\
&\leq \bigg(1-\frac{\eps^2_b\zeta}{|\calo|}\bigg)^{t-1}+\bigg(1-\frac{\eps^2_b\zeta}{|\calo|}\bigg)^{T-t}+2\Delta_{\theta,\hat\theta}.
\end{align*}

Next, we bound the difference of two-step smoothing distributions $\norms{\tilde\gamma^{\theta}_{\mu,t|T}-\tilde\gamma^{\hat\theta}_{k,t}}_\tv$. Although the idea is straightforward, the details are tedious. For any $t\in[2:T]$, from (\ref{eq:smoothing2}) we have
\begin{align*}
&\tilde \gamma^{\theta}_{\mu,t|T}(o_{t-1},b_t)\\
\propto~&\pi_b(b_t|s_t,o_{t-1};\theta_{b})\left[\sum_{o_t}\bar\pi_{hi}(o_{t}|s_{t},o_{t-1},b_{t};\theta_{hi})\pi_{lo}(a_{t}|s_{t},o_{t};\theta_{lo})\frac{\gamma^\theta_{\mu,t|T}(o_t,b_t)}{\alpha^\theta_{\mu,t}(o_t,b_t)}\right]\left[\sum_{b_{t-1}}\alpha^\theta_{\mu,t-1}(o_{t-1},b_{t-1})\right]\\
\propto~&\sum_{o_t}\frac{\bar\pi_{hi}(o_{t}|s_{t},o_{t-1},b_{t};\theta_{hi})\pi_{lo}(a_{t}|s_{t},o_{t};\theta_{lo})\gamma^\theta_{\mu,t|T}(o_t,b_t)[\sum_{b_{t-1}}\alpha^\theta_{\mu,t-1}(o_{t-1},b_{t-1})]\pi_b(b_t|s_t,o_{t-1};\theta_{b})}{\sum_{o'_{t-1},b_{t-1}}\pi_b(b_t|s_t,o'_{t-1};\theta_b)\bar\pi_{hi}(o_t|s_t,o'_{t-1},b_t;\theta_{hi})\pi_{lo}(a_t|s_t,o_t;\theta_{lo})\alpha^\theta_{\mu,t-1}(o'_{t-1},b_{t-1})}\\
=~&\sum_{o_t}\frac{\pi_b(b_t|s_t,o_{t-1};\theta_{b})\bar\pi_{hi}(o_{t}|s_{t},o_{t-1},b_{t};\theta_{hi})[\sum_{b_{t-1}}\alpha^\theta_{\mu,t-1}(o_{t-1},b_{t-1})]}{\sum_{o'_{t-1}}\pi_b(b_t|s_t,o'_{t-1};\theta_b)\bar\pi_{hi}(o_t|s_t,o'_{t-1},b_t;\theta_{hi})[\sum_{b_{t-1}}\alpha^\theta_{\mu,t-1}(o'_{t-1},b_{t-1})]}\gamma^\theta_{\mu,t|T}(o_t,b_t).
\end{align*}
The denominators are all positive due to the non-degeneracy assumption. It can be easily verified that the normalizing constants involved in the second and the third line cancel each other. As abbreviations, define
\begin{align*}
&g^\theta(o_{t-1},s_t,o_t,b_t)\defeq\pi_b(b_t|s_t,o_{t-1};\theta_{b})\bar\pi_{hi}(o_{t}|s_{t},o_{t-1},b_{t};\theta_{hi}),\\
&g^{\hat\theta}(o_{t-1},s_t,o_t,b_t)\defeq\pi_b(b_t|s_t,o_{t-1};\hat\theta_{b})\bar\pi_{hi}(o_{t}|s_{t},o_{t-1},b_{t};\hat\theta_{hi}),\\
&f^\theta_{\mu,t}(o_{t-1},s_t,o_t,b_t)\defeq\frac{g^\theta(o_{t-1},s_t,o_t,b_t)[\sum_{b_{t-1}}\alpha^\theta_{\mu,t-1}(o_{t-1},b_{t-1})]}{\sum_{o'_{t-1}}g^{\theta}(o'_{t-1},s_t,o_t,b_t)[\sum_{b_{t-1}}\alpha^\theta_{\mu,t-1}(o'_{t-1},b_{t-1})]},\\
&f^{\hat\theta}_{k,t}(o_{t-1},s_t,o_t,b_t)\defeq\frac{g^{\hat\theta}(o_{t-1},s_t,o_t,b_t)[\sum_{b_{t-1}}\alpha^{\hat\theta}_{k,t-1}(o_{t-1},b_{t-1})]}{\sum_{o'_{t-1}}g^{\hat\theta}(o'_{t-1},s_t,o_t,b_t)[\sum_{b_{t-1}}\alpha^{\hat\theta}_{k,t-1}(o'_{t-1},b_{t-1})]}.
\end{align*}
Then, 
\begin{align}
\norm{\tilde\gamma^{\theta}_{\mu,t|T}-\tilde\gamma^{\hat\theta}_{k,t}}_\tv
=~&\frac{1}{2}\sum_{o_{t-1},b_t}\bigg|\sum_{o_t}[f^\theta_{\mu,t}(o_{t-1},s_t,o_t,b_t)\gamma^\theta_{\mu,t|T}(o_t,b_t)-f^{\hat\theta}_{k,t}(o_{t-1},s_t,o_t,b_t)\gamma^{\hat\theta}_{k,t|T}(o_t,b_t)]\bigg|\nonumber\\
\leq~&\frac{1}{2}\sum_{o_{t-1},b_t,o_t}\left|f^\theta_{\mu,t}(o_{t-1},s_t,o_t,b_t)-f^{\hat\theta}_{k,t}(o_{t-1},s_t,o_t,b_t)\right|\gamma^\theta_{\mu,t|T}(o_t,b_t)\nonumber\\
&\hspace{5em}+\frac{1}{2}\sum_{o_{t-1},b_t,o_t}f^{\hat\theta}_{k,t}(o_{t-1},s_t,o_t,b_t)\left|\gamma^\theta_{\mu,t|T}(o_t,b_t)-\gamma^{\hat\theta}_{k,t|T}(o_t,b_t)\right|.\label{eq:twoterms}
\end{align}

Now, we bound the two terms on the RHS separately. Consider the first term in (\ref{eq:twoterms}), 
\begin{align}
&\frac{1}{2}\sum_{o_{t-1},o_t,b_t}\left|f^\theta_{\mu,t}(o_{t-1},s_t,o_t,b_t)-f^{\hat\theta}_{k,t}(o_{t-1},s_t,o_t,b_t)\right|\gamma^\theta_{\mu,t|T}(o_t,b_t)\nonumber\\
\leq~&\frac{1}{2}\max_{o_t,b_t}\sum_{o_{t-1},b_{t-1}} \bigg|\frac{g^\theta(o_{t-1},s_t,o_t,b_t)\alpha^\theta_{\mu,t-1}(o_{t-1},b_{t-1})}{\sum_{o'_{t-1},b'_{t-1}}g^{\theta}(o'_{t-1},s_t,o_t,b_t)\alpha^\theta_{\mu,t-1}(o'_{t-1},b'_{t-1})}\nonumber\\
&\hspace{8em}-\frac{g^\theta(o_{t-1},s_t,o_t,b_t)\alpha^{\hat\theta}_{k,t-1}(o_{t-1},b_{t-1})}{\sum_{o'_{t-1},b'_{t-1}}g^{\theta}(o'_{t-1},s_t,o_t,b_t)\alpha^{\hat\theta}_{k,t-1}(o'_{t-1},b'_{t-1})}\bigg|\nonumber\\
&\hspace{2em}+\frac{1}{2}\max_{o_t,b_t}\sum_{o_{t-1},b_{t-1}}\alpha^{\hat\theta}_{k,t-1}(o_{t-1},b_{t-1})\bigg|\frac{g^\theta(o_{t-1},s_t,o_t,b_t)}{\sum_{o'_{t-1},b'_{t-1}}g^{\theta}(o'_{t-1},s_t,o_t,b_t)\alpha^{\hat\theta}_{k,t-1}(o'_{t-1},b'_{t-1})}\nonumber\\
&\hspace{10em}-\frac{g^{\hat\theta}(o_{t-1},s_t,o_t,b_t)}{\sum_{o'_{t-1},b'_{t-1}}g^{\hat\theta}(o'_{t-1},s_t,o_t,b_t)\alpha^{\hat\theta}_{k,t-1}(o'_{t-1},b'_{t-1})}\bigg|.\label{eq:aux1}
\end{align}

Denote the two terms on the RHS of (\ref{eq:aux1}) as $\Delta_1$ and $\Delta_2$ respectively. To bound $\Delta_1$, we can apply Lemma~\ref{lemma:fstability} on the index set $[1:t-1]$ as follows, assuming $t> 2$. For any $t'\in[1:t-1]$, let $\vphi^\theta_{t'}=\alpha^\theta_{\mu,t'}$ and $\hat\vphi^{\hat\theta}_{t'}=\alpha^{\hat\theta}_{k,t'}$. For any $(o_t,b_t)$, let $\rho^\theta_{t-1}(o_{t-1},b_{t-1})=z^{-1}_\theta g^\theta(o_{t-1},s_t,o_t,b_t)$, where $z_\theta$ is a normalizing constant. For $1\leq t'< t-1$, let $\rho^\theta_{t'}=B^\theta_{t'}\rho^\theta_{t'+1}$. Then, 
\begin{equation*}
\Delta_1\leq \bigg(1-\frac{\eps^2_b\zeta}{|\calo|}\bigg)^{t-2}+\Delta_{\theta,\hat\theta}.
\end{equation*}
Such a bound holds trivially if $t\leq 2$. 

Next, we bound $\Delta_2$ as follows. Straightforward computation yields the following result. 
\begin{align*}
\Delta_2&=\frac{1}{2}\max_{o_t,b_t}\sum_{o_{t-1},b_{t-1}}\alpha^{\hat\theta}_{k,t-1}(o_{t-1},b_{t-1})\bigg|\frac{h(\theta;o_{t-1},s_t,a_t,o_t,b_t)}{\sum_{o'_{t-1},b'_{t-1}}h(\theta;o'_{t-1},s_t,a_t,o_t,b_t)\alpha^{\hat\theta}_{k,t-1}(o'_{t-1},b'_{t-1})}\\
&\hspace{7em}-\frac{h(\hat\theta;o_{t-1},s_t,a_t,o_t,b_t)}{\sum_{o'_{t-1},b'_{t-1}}h(\hat\theta;o'_{t-1},s_t,a_t,o_t,b_t)\alpha^{\hat\theta}_{k,t-1}(o'_{t-1},b'_{t-1})}\bigg|\\
&\leq\max_{o_t,b_t}\frac{\sum_{o_{t-1},b_{t-1}}\left|h(\theta;o_{t-1},s_t,a_t,o_t,b_t)-h(\hat\theta;o_{t-1},s_t,a_t,o_t,b_t)\right|\alpha^{\hat\theta}_{k,t-1}(o_{t-1},b_{t-1})}{\sum_{o'_{t-1},b'_{t-1}}h(\theta;o'_{t-1},s_t,a_t,o_t,b_t)\alpha^{\hat\theta}_{k,t-1}(o'_{t-1},b'_{t-1})}\\
&\leq \frac{\max_{o_{t-1},o_t,b_t}\left|h(\theta;o_{t-1},s_t,a_t,o_t,b_t)-h(\hat\theta;o_{t-1},s_t,a_t,o_t,b_t)\right|}{\min_{o_{t-1},o_t,b_t}h(\theta;o_{t-1},s_t,a_t,o_t,b_t)}\leq \Delta_{\theta,\hat\theta},
\end{align*}
where we use the definition of $h(\theta;o_{t-1},s_t,a_t,o_t,b_t)$ in (\ref{eq:definitionofh}).

As for the second term in (\ref{eq:twoterms}), 
\begin{align*}
&\frac{1}{2}\sum_{o_{t-1},b_t,o_t}f^\theta_{k,t}(o_{t-1},s_t,o_t,b_t)\left|\gamma^\theta_{\mu,t|T}(o_t,b_t)-\gamma^\theta_{k,t|T}(o_t,b_t)\right|\\
=~&\norm{\gamma^{\theta}_{\mu,t|T}-\gamma^{\theta}_{k,t}}_\tv\leq\bigg(1-\frac{\eps^2_b\zeta}{|\calo|}\bigg)^{t-1}+\bigg(1-\frac{\eps^2_b\zeta}{|\calo|}\bigg)^{T-t}+2\Delta_{\theta,\hat\theta}.
\end{align*}
Combining the above gives the desirable result. 
\end{proof}

\subsection{Proof of Lemma~\ref{lemma:limitsmoothing}}\label{subsection:limitsmoothing}

Based on Lemma~\ref{lemma:smoothingdiff2}, for all $T\geq 2$, $\theta\in\calTheta$ and $t\in[1:T]$, with any observation sequence, both the sequences $\{\gamma^{\theta}_{k,t}\}_{k\in\N_+}$ and $\{\tilde \gamma^{\theta}_{k,t}\}_{k\in\N_+}$ are Cauchy sequences associated with the total variation distance. Moreover, the set of probability measures over the finite sample space $\calo\times\{0,1\}$ is complete. Therefore, the limits of both $\{\gamma^{\theta}_{k,t}\}_{k\in\N_+}$ and $\{\tilde \gamma^{\theta}_{k,t}\}_{k\in\N_+}$ as $k\rightarrow\infty$ exist with respect to the total variation distance. From the definitions of $\{\gamma^{\theta}_{k,t}\}_{k\in\N_+}$ and $\{\tilde \gamma^{\theta}_{k,t}\}_{k\in\N_+}$ in Appendix~\ref{subsection:moredef}, it is clear that their limits as $k\rightarrow\infty$ do not depend on $T$. 

The Lipschitz continuity of $\gamma^{\theta}_{\infty,t}$ also follows from Lemma~\ref{lemma:smoothingdiff2}. Specifically, for all $\theta,\hat\theta\in\calTheta$ and $t\in[1:T]$, with any observation sequence, 
\begin{equation*}
\norm{\gamma^{\theta}_{\infty,t}-\gamma^{\hat\theta}_{\infty,t}}_\tv\leq\frac{2|\calo|z_{\theta,\hat\theta}L_{\theta,\norms{\hat\theta-\theta}_2}}{\eps^2_b\zeta}\norm{\hat\theta-\theta}_2.
\end{equation*}
The coefficient of $\norms{\hat\theta-\theta}_2$ on the RHS can be upper bounded by a constant that does not depend on $\theta$ and $\hat\theta$. The same argument holds for $\tilde\gamma^{\theta}_{\infty,t}$. \qed

\subsection{Proof of Lemma~\ref{lemma:qdiff}}\label{section:qdiffproof}

For a cleaner notation, we omit the dependency on $\omega$ in the following analysis. From the definitions, for all $\theta,\theta'\in\calTheta$ and $\mu\in\mathcal{M}$, 
\begin{align*}
&Q^s_{\infty,T}(\theta'|\theta)-Q_{\mu,T}(\theta'|\theta)\\
=~&\frac{1}{T}\bigg\{\sum_{t=2}^T\sum_{o_{t-1},b_t}\left[\tilde \gamma^{\theta}_{\infty,t}(o_{t-1},b_t)-\tilde \gamma^{\theta}_{\mu,t|T}(o_{t-1},b_t)\right][\log \pi_{b}(b_t|s_t,o_{t-1};\theta'_b)]\\
&+\sum_{t=1}^T\sum_{o_t,b_t}\left[\gamma^{\theta}_{\infty,t}(o_t,b_t)-\gamma^{\theta}_{\mu,t|T}(o_t,b_t)\right][\log \pi_{lo}(a_t|s_t,o_t;\theta'_{lo})]\\
&+\sum_{t=1}^T\sum_{o_t}\left[\gamma^{\theta}_{\infty,t}(o_t,b_t=1)-\gamma^{\theta}_{\mu,t|T}(o_t,b_t=1)\right][\log \pi_{hi}(o_t|s_t;\theta'_{hi})]+err\bigg\},
\end{align*}
where the last term is a small error term associated with $t=1$ such that, 
\begin{equation*}
\left|err\right|=\bigg|\sum_{o_{0},b_1}\tilde \gamma^{\theta}_{\infty,1}(o_{0},b_1)\left[\log \pi_{b}(b_1|s_1,o_{0};\theta'_b)\right]\bigg|\leq \max_{b_1,s_1,o_0}|\log \pi_{b}(b_1|s_1,o_{0};\theta'_b)|.
\end{equation*}
The maximum on the RHS is finite due to the non-degeneracy assumption. Furthermore, 
\begin{align*}
&\left|Q^s_{\infty,T}(\theta'|\theta)-Q_{\mu,T}(\theta'|\theta)\right|\\
\leq~&\frac{1}{T}\bigg\{\sum_{t=2}^T\max_{b_t,s_t,o_{t-1}}\left|\log \pi_{b}(b_t|s_t,o_{t-1};\theta'_b)\right|\sum_{o_{t-1},b_t}\left|\tilde \gamma^{\theta}_{\infty,t}(o_{t-1},b_t)-\tilde \gamma^{\theta}_{\mu,t|T}(o_{t-1},b_t)\right|\\
&+\sum_{t=1}^T\max_{a_t,s_t,o_t}\left|\log \pi_{lo}(a_t|s_t,o_t;\theta'_{lo})\right|\sum_{o_t,b_t}\left|\gamma^{\theta}_{\infty,t}(o_t,b_t)-\gamma^{\theta}_{\mu,t|T}(o_t,b_t)\right|\\
&+\sum_{t=1}^T\max_{s_t,o_t}\left|\log \pi_{hi}(o_t|s_t;\theta'_{hi})\right|\sum_{o_t}\left|\gamma^{\theta}_{\infty,t}(o_t,b_t=1)-\gamma^{\theta}_{\mu,t|T}(o_t,b_t=1)\right|+|err|\bigg\}. 
\end{align*}
Since the bounds in Lemma~\ref{lemma:smoothingdiff1} hold for any $k>0$, they also hold in the limit as $k\rightarrow\infty$. Therefore, for any $\theta$, $\mu$ and any $t\in[1:T]$,
\begin{equation*}
\norm{\gamma^{\theta}_{\mu,t|T}-\gamma^{\theta}_{\infty,t}}_\tv\leq\bigg(1-\frac{\eps^2_b\zeta}{|\calo|}\bigg)^{t-1}+\bigg(1-\frac{\eps^2_b\zeta}{|\calo|}\bigg)^{T-t}.
\end{equation*}
For any $\theta$, $\mu$ and any $t\in[2:T]$,
\begin{equation*}
\norm{\tilde\gamma^{\theta}_{\mu,t|T}-\tilde\gamma^{\theta}_{\infty,t}}_\tv\leq 2\bigg(1-\frac{\eps^2_b\zeta}{|\calo|}\bigg)^{t-2}+\bigg(1-\frac{\eps^2_b\zeta}{|\calo|}\bigg)^{T-t}.
\end{equation*}
Combining everything above, 
\begin{align*}
&\left|Q^s_{\infty,T}(\theta'|\theta)-Q_{\mu,T}(\theta'|\theta)\right|\\
\leq~&\frac{1}{T}\bigg\{\max_{b_t,s_t,o_{t-1}}\left|\log \pi_{b}(b_t|s_t,o_{t-1};\theta'_b)\right|\bigg[1+2\sum_{t=2}^T\norm{\tilde\gamma^{\theta}_{\mu,t|T}-\tilde\gamma^{\theta}_{\infty,t}}_\tv\bigg]\\
&+2\left[\max_{a_t,s_t,o_t}\left|\log \pi_{lo}(a_t|s_t,o_t;\theta'_{lo})\right|+\max_{s_t,o_t}\left|\log \pi_{hi}(o_t|s_t;\theta'_{hi})\right|\right]\sum_{t=1}^T\norm{\gamma^{\theta}_{\mu,t|T}-\gamma^{\theta}_{\infty,t}}_\tv\bigg\}\\
\leq~&\frac{1}{T}\bigg\{\bigg(1+\frac{6|O|}{\eps_b^2\zeta}\bigg)\max_{b_t,s_t,o_{t-1}}\left|\log \pi_{b}(b_t|s_t,o_{t-1};\theta'_b)\right|\\
&\hspace{5em}+\frac{4|O|}{\eps_b^2\zeta}\left[\max_{a_t,s_t,o_t}\left|\log \pi_{lo}(a_t|s_t,o_t;\theta'_{lo})\right|+\max_{s_t,o_t}\left|\log \pi_{hi}(o_t|s_t;\theta'_{hi})\right|\right]\bigg\}=\frac{C(\theta')}{T},
\end{align*}
where $C(\theta')$ is a positive real number that only depends on $\theta'$ and the structural constants $|\calo|$, $\zeta$ and $\eps_b$. Due to Assumption~\ref{as:continuity}, $C(\theta')$ is continuous with respect to $\theta'$. Since $\calTheta$ is compact, $\sup_{\theta'\in\calTheta}C(\theta')<\infty$. Therefore, 
\begin{equation*}
\left|Q^s_{\infty,T}(\theta'|\theta)-Q_{\mu,T}(\theta'|\theta)\right|\leq \frac{1}{T}\sup_{\theta'\in\calTheta}C(\theta').
\end{equation*}
Taking supremum with respect to $\theta$, $\theta'$ and $\mu$ completes the proof. \qed

\subsection{Proof of the strong stochastic equicontinuity condition (\ref{eq:sse})}\label{subsection:sseproof}

First, for all $\delta>0$ and $\omega\in\Omega$, 
\begin{align*}
&\limsup_{T\rightarrow\infty}\sup_{\theta_1,\theta'_1,\theta_2,\theta'_2\in\calTheta;\norms{\theta_1-\theta_2}_2+\norms{\theta'_1-\theta'_2}_2\leq\delta}\left|Q^s_{\infty,T}(\theta'_1|\theta_1;\omega)-Q^s_{\infty,T}(\theta'_2|\theta_2;\omega)\right|\\
\leq~&\limsup_{T\rightarrow\infty}\frac{1}{T}\sup_{\theta_1,\theta'_1,\theta_2,\theta'_2\in\calTheta;\norms{\theta_1-\theta_2}_2+\norms{\theta'_1-\theta'_2}_2\leq\delta}\left|f_t(\theta'_1|\theta_1;\omega)-f_t(\theta'_2|\theta_2;\omega)\right|.
\end{align*}
Due to the boundedness of $f_t(\theta'|\theta;\omega)$ from Appendix~\ref{subsection:asymptoticqfun}, we can apply the ergodic theorem (Lemma~\ref{lemma:ergodic}). $\P_{\theta^*,\nu^*}$ almost surely, 
\begin{align*}
&\limsup_{T\rightarrow\infty}\frac{1}{T}\sum_{t=1}^T\sup_{\theta_1,\theta'_1,\theta_2,\theta'_2\in\calTheta;\norms{\theta_1-\theta_2}_2+\norms{\theta'_1-\theta'_2}_2\leq\delta}\left|f_t(\theta'_1|\theta_1;\omega)-f_t(\theta'_2|\theta_2;\omega)\right|\\
=~&\E_{\theta^*,\nu^*}\bigg[\sup_{\theta_1,\theta'_1,\theta_2,\theta'_2\in\calTheta;\norms{\theta_1-\theta_2}_2+\norms{\theta'_1-\theta'_2}_2\leq\delta}\left|f_1(\theta'_1|\theta_1;\omega)-f_1(\theta'_2|\theta_2;\omega)\right|\bigg]\\
\leq~&\E_{\theta^*,\nu^*}\bigg[\sup_{\theta_1,\theta'_1,\theta'_2\in\calTheta;\norms{\theta'_1-\theta'_2}_2\leq\delta}\left|f_1(\theta'_1|\theta_1;\omega)-f_1(\theta'_2|\theta_1;\omega)\right|\bigg]\\
&\hspace{5em}+\E_{\theta^*,\nu^*}\bigg[\sup_{\theta_1,\theta_2,\theta'_2\in\calTheta;\norms{\theta_1-\theta_2}_2\leq\delta}\left|f_1(\theta'_2|\theta_1;\omega)-f_1(\theta'_2|\theta_2;\omega)\right|\bigg].
\end{align*}

Notice that for all $\theta_1$, $\theta'_1$, $\theta'_2$ and $\omega$, 
\begin{multline*}
\left|f_1(\theta'_1|\theta_1;\omega)-f_1(\theta'_2|\theta_1;\omega)\right|\leq\max_{o_t}\left|\log \pi_{hi}(o_t|\omega(s_t);\theta'_{1,hi})-\log \pi_{hi}(o_t|\omega(s_t);\theta'_{2,hi})\right|\\
+\max_{o_t}\left|\log \pi_{lo}(\omega(a_t)|\omega(s_t),o_t;\theta'_{1,lo})-\log \pi_{lo}(\omega(a_t)|\omega(s_t),o_t;\theta'_{2,lo})\right|\\
+\max_{o_{t-1},b_t}\left|\log \pi_{b}(b_t|\omega(s_t),o_{t-1};\theta'_{1,b})-\log \pi_{b}(b_t|\omega(s_t),o_{t-1};\theta'_{2,b})\right|.
\end{multline*}

The RHS does not depend on $\theta_1$. Due to Assumption~\ref{as:continuity}, $\pi_{hi}$, $\pi_{lo}$ and $\pi_b$ as functions of the parameter $\theta$ are uniformly continuous on $\calTheta$, with any other input arguments. Therefore it is straightforward to verify that, for any $\omega\in\Omega$, 
\begin{equation*}
\lim_{\delta\rightarrow 0}\sup_{\theta_1,\theta'_1,\theta'_2\in\calTheta;\norms{\theta'_1-\theta'_2}_2\leq\delta}\left|f_1(\theta'_1|\theta_1;\omega)-f_1(\theta'_2|\theta_1;\omega)\right|=0.
\end{equation*}
Applying the dominated convergence theorem, 
\begin{equation*}
\lim_{\delta\rightarrow 0}\E_{\theta^*,\nu^*}\bigg[\sup_{\theta_1,\theta'_1,\theta'_2\in\calTheta;\norms{\theta'_1-\theta'_2}_2\leq\delta}\left|f_1(\theta'_1|\theta_1;\omega)-f_1(\theta'_2|\theta_1;\omega)\right|\bigg]=0.
\end{equation*}

Similarly, using Lemma~\ref{lemma:limitsmoothing} we can show that for any $\omega\in\Omega$, 
\begin{equation*}
\lim_{\delta\rightarrow 0}\sup_{\theta_1,\theta_2,\theta'_2\in\calTheta;\norms{\theta_1-\theta_2}_2\leq\delta}\left|f_1(\theta'_2|\theta_1;\omega)-f_1(\theta'_2|\theta_2;\omega)\right|=0.
\end{equation*}
Using the dominated convergence theorem gives the convergence of the expectation as well. Combining the above gives the strong stochastic equicontinuity condition (\ref{eq:sse}). \qed

\subsection{Proof of Lemma~\ref{lemma:selfconsistency}}\label{subsection:selfconsistency}

Consider the following joint distribution on the graphical model shown in Figure~\ref{figure:model}: the prior distribution of $(O_0,S_1)$ is $\nu^*$, and the joint distribution of the rest of the graphical model is determined by an options with failure policy with parameters $\zeta$ and $\theta$. Notice that this is the \emph{correct} graphical model for the inference of the true parameter $\theta^*$, since the assumed prior distribution of $(O_0,S_1)$ coincides with the correct one. 

For clarity, we use the same notations as in Appendix~\ref{subsection:q} for the complete likelihood function, the marginal likelihood function and the (unnormalized) $Q$-function. Specifically, such quantities used in this proof have the same symbols as those defined in Appendix~\ref{subsection:q}, but mathematically they are not the same. 

Parallel to Appendix~\ref{subsection:q}, the complete likelihood function is
\begin{equation*}
L(s_{1:T},a_{1:T},o_{0:T},b_{1:T};\theta)=
\nu^*(o_0,s_1)\mathbb P_{\theta,o_0,s_1}(S_{2:T}=s_{1:T},A_{1:T}=a_{1:T},O_{1:T}=o_{1:T},B_{1:T}=b_{1:T}).
\end{equation*}
The marginal likelihood function is
\begin{equation*}
L^m(s_{1:T},a_{1:T};\theta)=\sum_{o_0}\nu^*(o_0,s_1)\mathbb P_{\theta,o_0,s_1}(S_{2:T}=s_{1:T},A_{1:T}=a_{1:T}).
\end{equation*}
Let $\mu^*$ be the conditional distribution of $O_0$ given $s_1$. For any $o_0\in\calo$, 
\begin{equation*}
\mu^*(o_0|s_1)=\frac{\nu^*(o_0,s_1)}{\sum_{o'_0\in\calo}\nu^*(o'_0,s_1)}. 
\end{equation*}
Therefore, for the inference of $\theta^*$ considered in this proof, the (unnormalized) $Q$-function can be expressed as
\begin{align*}
&\tilde Q_{\mu^*,T}(\theta'|\theta)\\
=~&\sum_{o_{0:T},b_{1:T}}\frac{L(s_{1:T},a_{1:T},o_{0:T},b_{1:T};\theta)}{L^m(s_{1:T},a_{1:T};\theta)}\log L(s_{1:T},a_{1:T},o_{0:T},b_{1:T};\theta')\\
=~&\sum_{o_{0:T},b_{1:T}}\mu^*(o_0|s_1)\mathbb P_{\theta,o_0,s_1}(S_{2:T}=s_{2:T},A_{1:T}=a_{1:T},O_{1:T}=o_{1:T},B_{1:T}=b_{1:T})\\
&\hspace{2em}
\times z^\theta_{\gamma,\mu^*}\log [\nu^*(o_0,s_1)\mathbb P_{\theta',o'_0,s_1}(S_{2:T}=s_{1:T},A_{1:T}=a_{1:T},O_{1:T}=o_{1:T},B_{1:T}=b_{1:T})]. 
\end{align*}
We can rewrite $\tilde Q_{\mu^*,T}(\theta'|\theta)$ using the structure of the options with failure framework, drop the terms irrelevant to $\theta'$ and normalize using $T$. The result is the following definition of the (normalized) $Q$-function:
\begin{align*}
Q^*_T(\theta'|\theta)\defeq &\frac{\sum_{o_0,b_1}\nu^*(o_0|s_1)\mathbb P_{\theta,o_0,s_1}(S_{2:T}=s_{2:T},A_{1:T}=a_{1:T},B_1=b_1)[\log \pi_{b}(b_1|s_1,o_0;\theta'_b)]}{T\sum_{o_0}\nu^*(o_0,s_1)\mathbb P_{\theta,o_0,s_1}(S_{2:T}=s_{1:T},A_{1:T}=a_{1:T})}\\
&\hspace{5em}+\frac{1}{T}\sum_{t=1}^T\sum_{o_t,b_t}\gamma^{\theta}_{\mu^*,t|T}(o_t,b_t)[\log \pi_{lo}(a_t|s_t,o_t;\theta'_{lo})]\\
&\hspace{5em}+\frac{1}{T}\sum_{t=1}^T\sum_{o_t}\gamma^{\theta}_{\mu^*,t|T}(o_t,b_t=1)[\log \pi_{hi}(o_t|s_t;\theta'_{hi})]\\
&\hspace{10em}+\frac{1}{T}\sum_{t=2}^T\sum_{o_{t-1},b_t}\tilde \gamma^{\theta}_{\mu^*,t|T}(o_{t-1},b_t)[\log \pi_{b}(b_t|s_t,o_{t-1};\theta'_b)]. 
\end{align*}

We draw a comparison between $Q^*_T(\theta'|\theta)$ and $Q_{\mu^*,T}(\theta'|\theta)$ defined in (\ref{eq:Q}): their difference is in the first term of $Q^*_T(\theta'|\theta)$. Maximizing $Q^*_T(\theta'|\theta)$ with respect to $\theta'$ is equivalent to maximizing the (unnormalized) $Q$-function $\tilde Q_{\mu^*,T}(\theta'|\theta)$. In Algorithm~\ref{algorithm}, since $Q^*_T(\theta'|\theta)$ is unavailable, we use $Q_{\mu^*,T}(\theta'|\theta)$ as its approximation. 

$Q^*_T(\theta'|\theta)$ depends on the observation sequence, therefore it is a function of a sample path $\omega\in\Omega$. In the following we explicitly show this dependency by writing $Q^*_T(\theta'|\theta;\omega)$. Clearly, for all $\theta,\theta'\in\calTheta$, $\omega\in\Omega$ and $T\geq 2$, 
\begin{equation*}
\left|Q^*_T(\theta'|\theta;\omega)-Q_{\mu^*,T}(\theta'|\theta;\omega)\right|\leq \frac{1}{T}\sup_{\theta'\in\calTheta}\max_{b_1,s_1,o_0}\left|\log \pi_{b}(b_1|s_1,o_0;\theta'_b)\right|.
\end{equation*}
Combining this with the stochastic convergence of $Q_{\mu^*,T}$ as shown in Theorem~\ref{thm:existence}, we have, that for any $\theta\in\calTheta$, as $T\rightarrow\infty$, 
\begin{equation*}
\left|Q^*_T(\theta|\theta^*;\omega)-\bar Q(\theta|\theta^*)\right|\rightarrow 0,~P_{\theta^*,\nu^*}\text{-a.s.}
\end{equation*}
Using the dominated convergence theorem, such a convergence holds in expectation as well. For any $\theta\in\calTheta$, 
\begin{equation*}
\lim_{T\rightarrow\infty}\E_{\theta^*,\nu^*}\left[Q^*_T(\theta|\theta^*;\omega)\right]=\bar Q(\theta|\theta^*). 
\end{equation*}

Since maximizing $Q^*_T(\theta|\theta^*)$ with respect to $\theta$ is equivalent to maximizing the (unnormalized) $Q$-function $\tilde Q_{\mu^*,T}(\theta|\theta^*)$, the standard monotonicity property of the EM update holds as well. For all $\theta\in\calTheta$, $\omega\in\Omega$ and $T\geq 2$, 
\begin{equation*}
\log L^m[\omega(s_{1:T}),\omega(a_{1:T});\theta]-\log L^m[\omega(s_{1:T}),\omega(a_{1:T});\theta^*]\geq T\left[Q^*_T(\theta|\theta^*;\omega)-Q^*_T(\theta^*|\theta^*;\omega)\right].
\end{equation*}
Taking expectation on both sides, we have
\begin{equation*}
\E_{\theta^*,\nu^*}[\lhs]=\sum_{s_{1:T},a_{1:T}}L^m(s_{1:T},a_{1:T};\theta^*)\log\frac{L^m(s_{1:T},a_{1:T};\theta)}{L^m(s_{1:T},a_{1:T};\theta^*)}\leq 0,
\end{equation*}
due to the non-negativity of the Kullback-Leibler divergence. Therefore, $\E_{\theta^*,\nu^*}[Q^*_T(\theta|\theta^*;\omega)]\leq \E_{\theta^*,\nu^*}[Q^*_T(\theta^*|\theta^*;\omega)]$, and in the limit we have $\bar Q(\theta|\theta^*)\leq \bar Q(\theta^*|\theta^*)$ for all $\theta\in\calTheta$. Applying the identifiability assumption for the uniqueness of $\bar M(\theta^*)$ completes the proof. \qed

\section{Additional experiments and details omitted in Section~\ref{section:experiment}}\label{section:additional}

\subsection{Generation of the observation sequences}

We first introduce the method to sample observation sequences from the stationary Markov chain induced by the expert policy. Using the expert policy and a fixed $(o_0,s_1)$ pair, we generate 50 sample paths of length 20,000. Then, the first 10,000 time steps in each sample path are discarded, and the rest state-action pairs are saved as the observation sequences used in the algorithm. For different $T$, we just take the first $T$ time steps in each observation sequence. 

Such a procedure is motivated by Proposition~\ref{lemma:ergo}: it can be easily verified that Assumption~\ref{as:nondeg} and \ref{as:continuity} hold in our numerical example. Therefore, from Proposition~\ref{lemma:ergo}, the distribution of $X_t$ approaches the unique stationary distribution regardless of the initial $(o_0,s_1)$ pair. In this way, Assumption~\ref{as:initial} is approximately satisfied. 

\subsection{Analytical expression of the parameter update}

For our numerical example, the parameter update of Algorithm~\ref{algorithm} has a unique analytical solution. For all $\theta\in\calTheta$, $\omega\in\Omega$, $T\geq 2$ and $\mu\in\mathcal{M}$, we first derive the analytical expression of $M_{\mu,T}(\theta;\omega)_{hi}$ which is the updated parameter for $\pi_{hi}$ based on the previous parameter $\theta$. Such a notation for parameter updates is borrowed from Assumption~\ref{as:local}. Using the expression of the $Q$-function (\ref{eq:Q}), we have
\begin{equation*}
M_{\mu,T}(\theta;\omega)_{hi}\in\argmax_{\theta'_{hi}\in\calTheta_{hi}}\sum_{t=1}^T\sum_{o_t}\gamma^{\theta}_{\mu,t|T}(o_t,b_t=1)[\log \pi_{hi}(o_t|s_t;\theta'_{hi})],
\end{equation*}
where $s_t$ on the RHS is the state value $\omega(s_t)$ from the sample path $\omega$. We omit $\omega$ on the RHS for a cleaner notation. Let $f(\theta'_{hi})$ denote the sum inside the argmax. Then, 
\begin{align*}
f(\theta'_{hi})&=\sum_{t=1}^T\bigg\{\gamma^{\theta}_{\mu,t|T}(o_t=\textrm{LEFTEND},b_t=1)[\log \pi_{hi}(o_t=\textrm{LEFTEND}|s_t;\theta'_{hi})]\\
&\hspace{5em}+\gamma^{\theta}_{\mu,t|T}(o_t=\textrm{RIGHTEND},b_t=1)[\log \pi_{hi}(o_t=\textrm{RIGHTEND}|s_t;\theta'_{hi})]\bigg\}\\
&=\sum_{t=1}^T\bigg\{\gamma^{\theta}_{\mu,t|T}(o_t=\textrm{LEFTEND},b_t=1)\Big[\mathbbm{1}[s_t=1,2]\log\theta'_{hi}+\mathbbm{1}[s_t=3,4]\log(1-\theta'_{hi})\Big]\\
&\hspace{1em}+\gamma^{\theta}_{\mu,t|T}(o_t=\textrm{RIGHTEND},b_t=1)\Big[\mathbbm{1}[s_t=3,4]\log\theta'_{hi}+\mathbbm{1}[s_t=1,2]\log(1-\theta'_{hi})\Big]\bigg\}.
\end{align*}
Taking the derivative of $f(\theta'_{hi})$, we can verify that $f(\theta'_{hi})$ is strongly concave. Therefore, the parameter update for $\pi_{hi}$ is unique. 
\begin{equation*}
  M_{\mu,T}(\theta;\omega)_{hi}=
  \begin{cases}
    0.1, & \text{if $\tilde M_{\mu,T}(\theta;\omega)_{hi}<0.1$}, \\
    \tilde M_{\mu,T}(\theta;\omega)_{hi}, & \text{if $0.1\leq\tilde M_{\mu,T}(\theta;\omega)_{hi}\leq 0.9$},\\
    0.9,& \text{if $\tilde M_{\mu,T}(\theta;\omega)_{hi}>0.9$,}
  \end{cases}
\end{equation*}
where $\tilde M_{\mu,T}(\theta;\omega)_{hi}$ is the unconstrained parameter update given as
\begin{multline*}
\tilde M_{\mu,T}(\theta;\omega)_{hi}=\frac{\sum_{t=1}^T\gamma^{\theta}_{\mu,t|T}(o_t=\textrm{LEFTEND},b_t=1)\mathbbm{1}[s_t=1,2]}{\sum_{t=1}^T\sum_{o_t}\gamma^{\theta}_{\mu,t|T}(o_t,b_t=1)}\\
+\frac{\sum_{t=1}^T\gamma^{\theta}_{\mu,t|T}(o_t=\textrm{RIGHTEND},b_t=1)\mathbbm{1}[s_t=3,4]}{\sum_{t=1}^T\sum_{o_t}\gamma^{\theta}_{\mu,t|T}(o_t,b_t=1)}.
\end{multline*}
Similarly, the unconstrained parameter updates for $\pi_{lo}$ and $\pi_b$ are the following:
\begin{multline*}
\tilde M_{\mu,T}(\theta;\omega)_{lo}=\frac{1}{T}\sum_{t=1}^T\sum_{b_t}\bigg\{\gamma^{\theta}_{\mu,t|T}(o_t=\textrm{LEFTEND},b_t)\mathbbm{1}[a_t=\textrm{LEFT}]\\+\gamma^{\theta}_{\mu,t|T}(o_t=\textrm{RIGHTEND},b_t)\mathbbm{1}[a_t=\textrm{RIGHT}]\bigg\}.
\end{multline*}
\begin{equation*}
\tilde M_{\mu,T}(\theta;\omega)_{b}=\frac{1}{T-1}\sum_{t=2}^T\sum_{o_{t-1}}\bigg\{\tilde\gamma^{\theta}_{\mu,t|T}(o_{t-1},b_t=1)\mathbbm{1}[\textrm{event}]+\tilde\gamma^{\theta}_{\mu,t|T}(o_{t-1},b_t=0)\mathbbm{1}[\neg \textrm{event}]\bigg\},
\end{equation*}
where the $\textrm{event}=\{(s_t=1,o_{t-1}=\textrm{LEFTEND})\vee(s_t=4,o_{t-1}=\textrm{RIGHTEND})\}$. The parameter updates $M_{\mu,T}(\theta;\omega)_{lo}$ and $M_{\mu,T}(\theta;\omega)_{b}$ are the projections of $\tilde M_{\mu,T}(\theta;\omega)_{lo}$ and $\tilde M_{\mu,T}(\theta;\omega)_{b}$ onto $[0.1,0.9]$, respectively.

\subsection{Supplementary results to Figure~\ref{figure:exp_1}}

In this subsection we present supplementary results to Figure~\ref{figure:exp_1}. In Figure~\ref{figure:exp_1}, $err(n,T)$ is defined as the average of $\norms{\theta^{(n)}-\theta^*}_2$ over all the 50 sample paths. Here, we divide the set of sample paths into smaller sets and evaluate the average of $\norms{\theta^{(n)}-\theta^*}_2$ over these smaller sets separately. The settings for the computation of parameter estimates are the same as in Section~\ref{section:experiment}. The following procedure serves as the post-processing step of the obtained parameter estimates. 

Concretely, as defined in Section~\ref{section:experiment}, we obtain a sequence $\{\norms{\theta^{(n)}-\theta^*}_2;\omega,T\}_{n\in[0:N]}$ after running Algorithm~\ref{algorithm} with any sample path $\omega$ and any $T$. After fixing $T$ and letting $n=N$, $\norms{\theta^{(N)}-\theta^*}_2$ is a function of $\omega$ only. With a given threshold interval $I=[I_1,I_2]$, we define a smaller set of sample paths as the set of $\omega$ with $\norms{\theta^{(N)}-\theta^*}_2$ greater than the $I_1$-th percentile and less than the $I_2$-th percentile. Let $err(n,T,I)$ be the average of $\norms{\theta^{(n)}-\theta^*}_2$ over this smaller set of sample path specified by interval $I$. For $T=8000$, the values of $err(n,T,I)$ with specific choices of $I$ are plotted below. If $I=[0,100]$, $err(n,T,I)$ is equivalent to $err(n,T)$ investigated in Section~\ref{section:experiment}. 

\begin{figure}[ht]
    \centering
    \includegraphics[width=450pt]{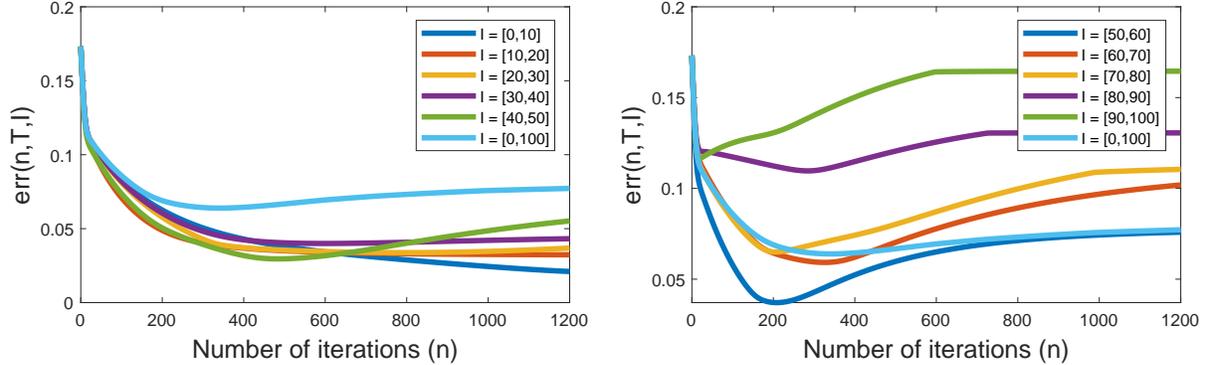}
    \caption[]{Plots of $err(n,T,I)$ with varying $n$ and $I$; $T$ is fixed as 8000.}\label{figure:exp_1_supp}
\end{figure}

Figure~\ref{figure:exp_1_supp} suggests that with probability around 0.6, our algorithm with the particular choice of $T$ and $\theta^{(0)}$ achieves decent performance, decreasing the original estimation error by at least a half. A worth-noting observation is that, for all the choices of $I$ (including $I=[90,100]$ representing the \emph{failed} sample paths), $err(n,T,I)$ roughly follows the same exponential decay in the early stage of the algorithm (roughly the first 10 iterations). The same behavior can be observed for $T=5000$ and $T=10000$ as well. It is not clear whether this behavior is general or specific to our numerical example. Detailed investigation is required in future work. 

\subsection[]{Varying $\mu$}

\begin{wrapfigure}{R}{0.45\textwidth}
\centering
\includegraphics[width=200pt]{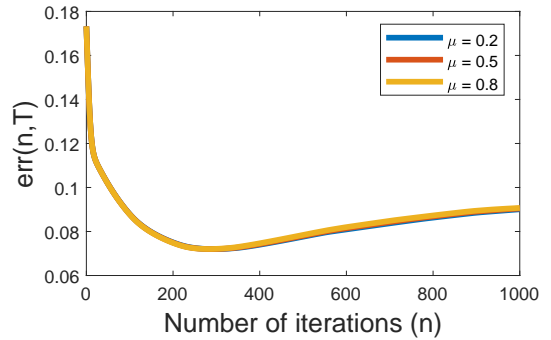}
\caption{Plots of $err(n,T)$ with varying $n$ and $\mu$; $T$ is fixed to 5000.}\label{figure:exp_3}
\end{wrapfigure}
In this subsection we investigate the effect of $\mu$ on the performance of Algorithm~\ref{algorithm}. Intuitively, from the uniform forgetting analysis throughout this paper, it is reasonable to expect that at each iteration, the effect of $\mu$ on the parameter update is negligible if $T$ is large. However, such a negligible error could accumulate if $N$ is large. The effect of $\mu$ on the final parameter estimate is not clear without experiments. 

We use the same observation sequences as in Section~\ref{section:experiment}. $T$ is fixed as 5000. $\theta^{(0)}=(0.5,0.6,0.7)$, and the parameter space for all the three parameters remains the same as $[0.1,0.9]$. For all $s_1$, $\mu(o_0=\textrm{RIGHTEND}|s_1)\in\{0.2,0.5,0.8\}$. The performance of the algorithm is evaluated by $err(n,T)$ defined in Section~\ref{section:experiment}. The result is presented in Figure~\ref{figure:exp_3}, which shows that indeed, the effect of $\mu$ on the final performance of the algorithm is negligible. For $n=1000$, $\max_{\mu}err(n,T)$ is $0.7\%$ higher than $\min_{\mu}err(n,T)$. 

\subsection{Random initialization}

Up to this point, all the empirical results use the same initial parameter estimate $\theta^{(0)}=(0.5,0.6,0.7)$ on all the 50 sample paths. In this subsection, we evaluate the effect of the initial estimation error $\{\theta^{(0)}-\theta^*\}_2$ on the performance of the algorithm, by applying random $\theta^{(0)}$. Such a randomization is not considered in Section~\ref{section:experiment} since more explanations are required. 

In this experiment, we use the same observation sequences as in Section~\ref{section:experiment}. $T$ is fixed to 8000. For all $s_1$, $\mu(o_0=\textrm{RIGHTEND}|s_1)=1$. The parameter space for all the three parameters remains the same as $[0.1,0.9]$. For each observation sequence, we first generate three independent samples $x_{hi}$, $x_{lo}$ and $x_b$ uniformly from the interval $[0,1]$. Then, $\theta^{(0)}$ is generated as follows: with a scale factor $w\in\{0.1,0.2,0.3\}$, let $\theta^{(0)}_{hi}=\theta^*_{hi}-w x_{hi}$, $\theta^{(0)}_{lo}=\theta^*_{lo}-w x_{lo}$ and $\theta^{(0)}_{b}=\theta^*_{b}-w x_{b}$. As a result, $\theta^{(0)}$ dependent on $w$ is different for different observation sequences. The choices of $\theta^{(0)}$ are not symmetrical with respect to $\theta^*$ due to the restriction of the bounded parameter space. For the parameter estimates obtained from the computation, $err(n,T)$ is defined as in Section~\ref{section:experiment}. The result is shown in Figure~\ref{figure:exp_2}. 

\begin{figure}[ht]
    \centering
    \includegraphics[width=450pt]{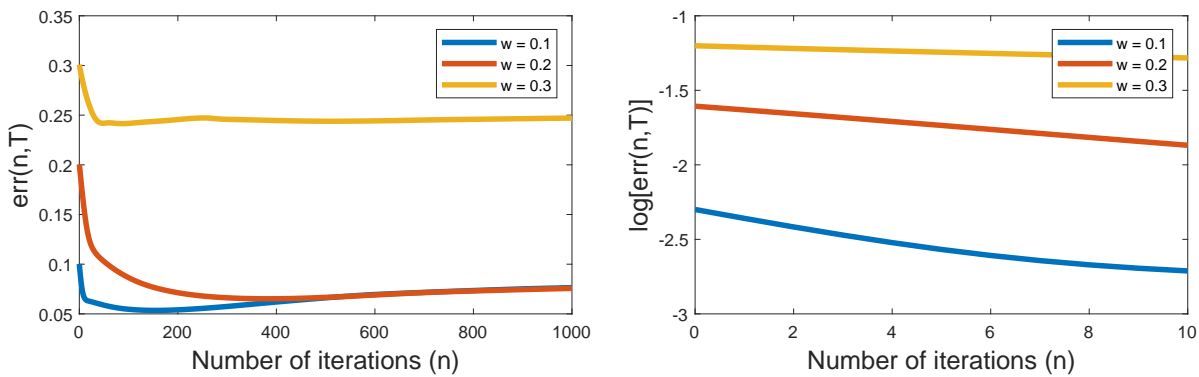}
    \caption[]{Plots of $err(n,T)$ with varying $n$ and $\theta^{(0)}$; $T$ is fixed to 8000.}\label{figure:exp_2}
\end{figure}

From Figure~\ref{figure:exp_2}, the curves corresponding to $w=0.1$ and $w=0.2$ qualitatively match the performance guarantee in Theorem~\ref{thm:perturbed}. The algorithm achieves decent performance when $\{\theta^{(0)}-\theta^*\}_2$ is intermediate (the case of $w=0.2$), where the average estimation error $err(n,T)$ is reduced by at least a half. If $\{\theta^{(0)}-\theta^*\}_2$ is small (the case of $w=0.1$), the parameter estimates cannot improve much from $\theta^{(0)}$. If $\{\theta^{(0)}-\theta^*\}_2$ is large (the case of $w=0.3$), the algorithm cannot converge to the vicinity of the true parameter, which is consistent with our local convergence analysis. 

\end{document}